\documentclass{article}

\usepackage{graphicx,fullpage}
\usepackage{hyperref}
\usepackage{listings}

 \usepackage{orcidlink}

\usepackage{amsmath}
\usepackage{amssymb}
\usepackage{mathtools}
\usepackage{amsthm}

\theoremstyle{plain}
\newtheorem{theorem}{Theorem}[section]
\newtheorem{proposition}[theorem]{Proposition}

\theoremstyle{definition}
\newtheorem{definition}[theorem]{Definition}

\theoremstyle{remark}

\usepackage[textsize=tiny]{todonotes}

\graphicspath{{./Fig/}}

\def\id{\mathrm{id}}

\def\barcalN{\overline{\mathcal{N}}}

\def\mmax{{\mathrm{max}}}
\def\mmin{{\mathrm{min}}}

\def\inv{\mathrm{inv}}

\def\Hilbert{\mathrm{Hilbert}}
\def\Sym{\mathrm{Sym}}
\def\Log{{\mathrm{Log}}}

\def\eqdef{{:=}}
\def\equaldef{{:=}}
\def\det{\mathrm{det}}

\def\sp{\mathrm{sp}}
\def\Sinh{\mathrm{Sinh}}

\def\LERP{\mathrm{LERP}}
\def\Cosh{\mathrm{Cosh}}
\def\sinh{\mathrm{sinh}}
\def\cosh{\mathrm{cosh}}

\def\arccosh{\mathrm{arccosh}}

\def\FR{\mathrm{FR}}

\def\calM{\mathcal{M}}

\def\Sym{\mathrm{Sym}}

\def\dP{\mathrm{d}P}

\def\dtheta{\mathrm{d}\theta}

\def\trace{\mathrm{trace}}
\def\proj{\mathrm{proj}}
\def\barN{{\overline{\mathcal{N}}}}
\def\dP{\mathrm{d}P}
\def\CO{\mathrm{CO}}

\def\GL{\mathrm{GL}}
\def\KL{\mathrm{KL}}

\def\diag{\mathrm{diag}}

\def\SPD{\mathrm{SPD}}
\def\Aff{\mathrm{Aff}}
\def\bbR{\mathbb{R}}
\def\Fisher{\mathrm{Fisher}}
\def\vectortwo#1#2{{\left[\begin{array}{l}#1 \cr #2\end{array}\right]}}
\def\mattwotwo#1#2#3#4{{\left[\begin{array}{ll}#1 & #2\cr #3 & #4\end{array}\right]}}

\def\matthreethree#1#2#3#4#5#6#7#8#9{{\left[\begin{array}{lll}#1 & #2 & #3
\cr #4 & #5 & #6\cr 
#7 & #8 & #9\end{array}\right]}}

\def\calE{\mathcal{E}}
\def\Ball{\mathrm{Ball}}
\def\tr{\mathrm{tr}}

\def\inner#1#2{{\langle #1, #2\rangle}}
\def\calN{\mathcal{N}}

\def\dx{\mathrm{d}x}

\def\SPD{\mathrm{SPD}}
\def\calP{\mathcal{P}}

\def\Length{\mathrm{Len}}
\def\dt{\mathrm{d}t}

\def\dim{\mathrm{dim}}
\def\Aff{\mathrm{Aff}}

\def\std{\mathrm{std}}

\def\barP{{\bar P}}
\def\dSigma{\mathrm{d}\Sigma}
\def\dmu{\mathrm{d}\mu}
\def\ds{\mathrm{d}s}
\def\st{\ :\ }
\def\bbR{\mathbb{R}}

\newtheorem{property}{Property}
\newtheorem{example}{Example}

\begin{document}

\title{Fisher-Rao distance and pullback SPD cone distances between multivariate normal distributions\footnote{A preliminary version of this work appeared as ``Fisher-Rao and pullback Hilbert cone distances on the multivariate Gaussian manifold with applications to simplification and quantization of mixtures,'' at the 2nd Annual Topology, Algebra, and Geometry in Machine Learning Workshop (TAG-ML) of ICML'23.}}
 
\author{Frank Nielsen\orcidlink{0000-0001-5728-0726}\\
\ \\
Sony Computer Science Laboratories Inc.\\ Tokyo, Japan}

\date{}
 
\maketitle

\begin{abstract}
Data sets consisting of multivariate normal distributions abound in many scientific areas like diffusion tensor imaging, structure tensor computer vision, radar signal processing, machine learning, etc.
In order to process those normal data sets for downstream tasks like filtering, classification, or clustering, one needs to define proper notions of dissimilarities between normals and smooth paths linking them. 
The Fisher-Rao distance defined as the Riemannian geodesic distance induced by the Fisher information metric is such a principled metric distance which however is not known in closed-form excepts for a few particular cases including the univariate and same-mean/same-covariance cases.
In this work, we first report a fast and robust method to approximate the Fisher-Rao distance between multivariate normal distributions with a guaranteed $1+\epsilon$ factor for any $\epsilon>0$.
Second, we introduce a class of distances based on diffeomorphic embeddings of the normal manifold into a submanifold of the higher-dimensional symmetric positive-definite cone corresponding to the manifold of centered normal distributions. 
We show that the projective Hilbert distance on the cone yields a metric distance on the embedded normal submanifold,
 and we pullback that cone distance with its associated straight line Hilbert cone geodesics to obtain a distance and smooth paths between normal distributions.
Compared to the Fisher-Rao distance approximation, the pullback Hilbert cone distance is computationally light since it requires to compute only the extreme minimal and maximal eigenvalues of matrices.
Finally, we show how to use those distances and paths in clustering tasks.
\end{abstract}

%%%%
\section{Introduction}
%%%%

Data sets of multivariate normal distributions (MVNs) are increasing frequent  in many scientific areas like medical imaging (diffusion tensor imaging~\cite{MVNGeodesicShooting-2014}), computer vision (image segmentation~\cite{carson2002blobworld} or structure tensor imaging~\cite{CovarianceTracking-2006}), signal processing (covariance matrices~\cite{barbaresco2012information} in radar or brain computer interfaces~\cite{barachant2011multiclass}), and machine learning (Gaussian mixtures or kernel density estimators~\cite{davis2006differential}). These data sets can be viewed as (weighted) point sets on a Gaussian manifold and Riemannian and information-geometric structures~\cite{Skovgaard-1984,IG-MVN-1999} on that manifold allows one to define geodesics and distances or divergences which allows on to build algorithms like filtering, classification, clustering or optimization techniques~\cite{tuzel2008pedestrian,absil2008optimization,Sra-2015}.
For example, we may simplify a Gaussian mixture model~\cite{davis2006differential,goldberger2008simplifying,zhang2010simplifying} (GMM) with $n$ components by viewing the mixture as a weighted point set and simplify the mixture by clustering the point set into $k$ clusters using $k$-means or $k$-medioids~\cite{davis2006differential} (as known as discrete $k$-means). We may also consider $n$ GMMs with potentially different components and build a codebook of all mixture components to quantize and compress the representation of these GMMs.  
In this work, we consider two kinds of metric distances and metric geodesics: The Fisher-Rao distance~\cite{strapasson2016clustering} and a new distance obtained by pulling back the Hilbert cone projective distance on an embedding of the Gaussian manifold into the higher-dimensional symmetric positive-definite matrix cone~\cite{SDPMVN-1990}.

The paper is organized as follows: In Section~\ref{sec:FRMVN}, we recall the Fisher-Rao geodesic distance and mention its lack of general closed-form formula on the Gaussian manifold. We then build on a recent breakthrough which studied the MVN Fisher-Rao geodesics~\cite{kobayashi2023geodesics} (\S\ref{sec:FRgeo}) to design an approximation method  which upper bounds the true Fisher-Rao distance with guaranteed $(1+\epsilon)$ precision for any $\epsilon>0$ (Theorem~\ref{thm:guar} in \S\ref{sec:FRapprox}). We present applications to simplification and quantization of GMMs in Section~\ref{sec:clustering} using fast clustering methods relying on nearest neighbor query data structures~\cite{yianilos1993data} and smallest enclosing balls in metric spaces.
In Section~\ref{sec:Hilbert}, we introduce the novel pullback Hilbert cone distance which is fast to compute and enjoys simple  expression of geodesics. 

%%%
\section{Fisher-Rao geodesics and distances}\label{sec:FRMVN}
%%%

A normal distribution $N(\mu,\Sigma)$ has probability density function (pdf) $p_{\mu,\Sigma}(x)$ defined on the full support $\bbR^d$  given by:
\begin{eqnarray*}
p_{\mu,\Sigma}(x)=
\frac{(2\pi)^{-\frac{d}{2}}}{\sqrt{\det(\Sigma)}} \exp\left(-\frac{(x-\mu)^\top\Sigma^{-1}(x-\mu)}{2}\right).
\end{eqnarray*}

Consider  the statistical model consisting of all $d$-variate normal distributions:
$$
\calN(d)=
\left\{ N(\lambda) \st \lambda=(\mu,\Sigma)\in\Lambda(d)=\bbR^d\times \Sym_+(d,\bbR)
\right\},
$$  
where $\Sym_+(d,\bbR)$ denote the set of $d\times d$ positive-definite matrices.
The dimension of $\calN(d)$ is $m=\dim(\Lambda(d))=d+\frac{d(d+1)}{2}=\frac{d(d+3)}{2}$.
When the dimension is clear from context, we omit to specify the dimension and write $\calN$ for short.

Let $l_\lambda(x)=\log p_\lambda(x)$ denote the log-likelihood function.
The MVN model is both identifiable (bijection between $\lambda$ and $p_\lambda$) and regular~\cite{calin2014geometric,IG-2016}, i.e.,
 the  Fisher information matrix $I(\theta)=-E_\theta[\nabla^2 l_\theta(x)]$ is positive-definite and defines a Riemannian metric 
$g_\Fisher$ called the Fisher information metric.
The Riemannian manifold $(\calN,g_\Fisher)$ is called the Fisher-Rao Gaussian manifold with squared infinitesimal length element~\cite{Skovgaard-1984}  $\ds_\Fisher^2$ at $(\mu,\Sigma)$ given by:
$$
\ds_\Fisher^2= \dmu^\top \Sigma^{-1} \dmu + \frac{1}{2}\tr\left(\left(\Sigma^{-1}\dSigma\right)^2\right),
$$ 
where $\dmu\in\bbR^d$ and $\dSigma\in\Sym(d,\bbR)$ (vector space of symmetric $d\times d$ matrices).

The Fisher-Rao length of a smooth curve $c(t)$ with $t\in [a,b]$ is defined by integrating the infinitesimal Fisher-Rao length element along the curve:
$\Length(c)=\int_a^b \ds_\Fisher(c(t))\, \dt$,
and the Fisher-Rao distance~\cite{Hotelling-1930,Rao-1945} between $N_0=N(\mu_0,\Sigma_0)$ and $N_1=N(\mu_1,\Sigma_1)$ is the geodesic distance on $(\calN,g_\Fisher)$, i.e., the length of the Riemannian geodesic $\gamma_\FR^\calN(N_0,N_1;t)$:
\begin{equation}
\rho_\FR(N_0,N_1)=\int_0^1 \ds_\Fisher(\gamma_\FR^\calN(N_0,N_1;t))\,\dt.
\end{equation}
In Riemannian geometry~\cite{godinho2012introduction}, geodesics with boundary conditions $N_0=\gamma_\FR^\calN(N_0,N_1;0)$ and 
$N_1=\gamma_\FR(N_0,N_1;1)$ are length minimizing curves among all curves $c(t)$ satisfying $c(0)=N(\mu_0,\Sigma_0)$ and $c(1)=N(\mu_1,\Sigma_1)$:
$$
\rho_\FR(N_0,N_1)=\inf_{\substack{c(t)\\ c(0)=p_{\mu_0,\Sigma_0}\\ c(1)=p_{\mu_1,\Sigma_1}}}  \, \left\{\Length(c)\right\}.
$$
Riemannian geodesics are parameterized by (normalized) arc length $t$
which ensures that
\begin{equation}
\rho(\gamma_\rho(P_0,P_1;s),\gamma_\rho(P_0,P_1;t))=|s-t|\, \rho(P_0,P_1).
\end{equation}

More generally, geodesics in differential geometry~\cite{calin2014geometric} are auto-parallel curves with respect to an affine connection $\nabla$: $\nabla_{\dot\gamma} \dot\gamma=0$,
where $\dot\ =\frac{d}{\dt}$ and $\nabla_X Y$ is the covariant derivative induced by the connection.
In Riemannian geometry, the default connection is the unique {\em Levi-Civita metric connection}~\cite{godinho2012introduction} $\nabla^g$ induced by the metric $g$.

In general, the Fisher-Rao distance between MVNs is not known in closed-form~\cite{FRMVNReview-2020,nielsen2023simple}.
However, there are two main cases where closed-form formula are known:  
\begin{itemize} 
\item The case $d=1$:
The Fisher-Rao distance between univariate normal distributions~\cite{Yoshizawa-1972} $N_0=N(\mu_0,\sigma_0^2)$ 
and $N_1=N(\mu_1,\sigma_1^2)$ is
\begin{equation}\label{eq:FR1D}
\rho_{\calN}(N_0,N_1)=\sqrt{2}\log \left(\frac{1+\Delta(\mu_0,\sigma_0;\mu_1,\sigma_1)}{1-\Delta(\mu_0,\sigma_0;\mu_1,\sigma_1)}\right)
,
\end{equation}
where for $(a,b,c,d)\in\bbR^4\backslash\{0\}$, 
\begin{equation}
\Delta(a,b;c,d)=\sqrt{\frac{(c-a)^2+2(d-b)^2}{(c-a)^2+2(d+b)^2}}  
\end{equation} 
is a M\"obius distance~\cite{burbea1982entropy}.

\item The case where MVNs $N_0$ and $N_1$ share the same mean~\cite{DoubleCone-James-1973,Skovgaard-1984}, i.e., they belong to some submanifold $\calN_\mu=\{N(\mu,\Sigma)\st\Sigma\in\Sym_+(d,\bbR)\}$.
When $\mu=0$, we let $\calP(d)=\calN_0(d)$.
We have:
\begin{eqnarray}\label{eq:FRsamemu}
\rho_{\calN_\mu}(N_0,N_1)&=& \sqrt{\frac{1}{2} \sum_{i=1}^d \log^2 \lambda_i(\Sigma_0^{-\frac{1}{2}}\Sigma_1\Sigma_0^{-\frac{1}{2}})},\label{eq:RaoCMVN}
\end{eqnarray}
where $\lambda_i(M)$ denotes the $i$-th largest eigenvalue of matrix $M$.
Observe that matrix $\Sigma_0^{-1}\Sigma_1$ may not be symmetric but $\Sigma_0^{-\frac{1}{2}}\Sigma_1\Sigma_0^{-\frac{1}{2}}$ is always SPD and $\lambda_i(\Sigma_0^{-1}\Sigma_1)=\lambda_i(\Sigma_0^{-\frac{1}{2}}\Sigma_1\Sigma_0^{-\frac{1}{2}})$.
The submanifolds $\calN_\mu$ are totally geodesic in $\calN$, and the Fisher-Rao geodesics are known in closed form:
$\gamma_\FR^\calN(N_0,N_1;t)=N(\mu,\Sigma_t)$
with
\begin{eqnarray}
\Sigma_t &=& \Sigma_0^{\frac{1}{2}}\, (\Sigma_0^{-\frac{1}{2}}\Sigma_1\Sigma_0^{-\frac{1}{2}})^t\, \Sigma_0^{\frac{1}{2}}.
% faux dans BMP paper &=& \Sigma_0\, (\Sigma_0^{-1}\Sigma_1)^t.
\end{eqnarray}
The geodesic $\gamma(\Sigma_0,\Sigma_1;t)=\Sigma_t$ (with $t\in [0,1]$) midpoint is thus 
$$
\gamma(\Sigma_0,\Sigma_1;\frac{1}{2})=\Sigma_{\frac{1}{2}}=\Sigma_0^{\frac{1}{2}}\, (\Sigma_0^{-\frac{1}{2}}\Sigma_1\Sigma_0^{-\frac{1}{2}})^{\frac{1}{2}}\, \Sigma_0^{\frac{1}{2}}.
$$  
See also~Appendix~\ref{sec:ahm}.
\end{itemize}

Notice that all submanifolds $\calN_\mu$ are non-positive curvature manifolds (NPC)~\cite{bridson2013metric,cheng2016recursive}, i.e. sectional curvatures are non-positive.
However, $\calN(d)$ is not a NPC manifold when $d>1$ since some sectional curvatures can be positive~\cite{Skovgaard-1984}. 
NPC property is important for designing optimization algorithms on manifolds with guaranteed convergence~\cite{cheng2016recursive}.
In a NPC manifold $(M,g)$, we can write the Riemannian distance using the Riemannian logarithm map $\Log_p: M\rightarrow T_pM$:
$\rho_g(p_1,p_2)=\|\Log_{p_1}(p_2)\|_{p_1}$,
where $\|v\|_p=\sqrt{g_p(v,v)}$.
On the NPC SPD cone $\Sym_+(d,\bbR)$ equipped with the trace metric~\cite{moakher2005differential,EllipticIsometrySPD-2021}
$$
g_P^\trace(P_1,P_2):=\tr(P^{-1}P_1P^{-1}P_2),
$$
the length element at $P$ is $\ds_P=\|P^{-\frac{1}{2}} \dP P^{-\frac{1}{2}}\|_F$, and the Riemannian logarithm map is expressed using the matrix logarithm $\Log$, and we have
$$
\rho_{g_\trace}(P_1,P_2)=\|\Log_{P_1}(P_2)\|_{P_1}=\|\Log(P_1^{-1}P_2)\|_F,
$$
where $\|X\|_F=\sqrt{\inner{X}{X}_F}=\sqrt{\sum_{i,j} X_{ij}^2}$ is the  Frobenius norm induced by the Frobenius inner product: $\inner{A}{B}_F=\tr(A^\top B)$ (Hilbert-Schmidt inner product). The Fr\"obenius norm can be calculated in quadratic time as 
$\|M\|_F=\sqrt{\sum_{i,j=1}^d m_{ij}^2}$ for a $d\times d$ matrix $M=[m_{ij}]$, and is equivalent to the $\ell_2$ norm on the vectorization of $M$: $\|M\|_F=\|\mathrm{vec}(M)\|_2$, where $\mathrm{vec}(M)=[m_{11},\ldots,m_{1d},\ldots,m_{d1},\ldots,m_{dd}]^\top \in\bbR^{d^2}$.
The SPD cone is also a Bruhat-Tits space~\cite{lang1999bruhat}.

Historically, the SPD Riemannian trace metric distance was studied   by Siegel~\cite{Siegel-1964} in the wider context of the complex manifold of symmetric complex square matrices with positive-definite imaginary part:
The so-called Siegel upper half space~\cite{friedland2004revisiting} which generalizes the Poincar\'e upper plane.
It was shown recently  that the Siegel upper half space is NPC~\cite{cabanes2021classification}.
Another popular distance in machine learning in the Wasserstein distance for which the underlying geometry on the Gaussian space was studied in~\cite{Takatsu-2011}.

%%%
\subsection{Invariance under action of the positive affine group}
%%%

The length element $\ds_\Fisher$ is invariant under the action of the positive affine group~\cite{Eriksen-MVNGeodesic-1986,eriksen1987geodesics}
$$
\Aff_+(d,\bbR)\equaldef \left\{(a,A) \st a\in\bbR^d, A\in\GL_+(d,\bbR)\right\},
$$ 
where $\GL_+(d,\bbR)$ denotes the group of $d\times d$ matrices with positive determinant.
The group identity element of $\Aff_+(d,\bbR)$ is $e=(0,I)$ and the group operation is $(a_1,A_1).(a_2,A_2)=(a_1+A_1a_2,A_1A_2)$ with inverse operation $(a,A)^{-1}=(-A^{-1}a,A^{-1})$). The positive affine group may be handled as a matrix group by mapping elements $(a,A)$ to $(d+1)\times (d+1)$ matrices 
$M_{(a,A)}\equaldef \mattwotwo{A}{a}{0}{1}$.
Then the matrix group operation is the matrix multiplication and inverse operation is given by the matrix inverse.
Let us consider the following group action (denoted by the dot $.$) of the positive affine group on the Gaussian manifold $\calN$:
$$
(a,A).N(\mu,\Sigma)=N(a+A\mu,A\Sigma A^\top).
$$
This action corresponds to the affine transformation of random variables: $Y=a+AX\sim N(a+A\mu,A\Sigma A^\top)$ where $X\sim N(\mu,\Sigma)$.
The statistical model $\calN$ can thus be interpreted as a group with identity element the standard MVN $N_\std=N(0,I)$:
$\calN(d)=\left\{ (a,A).N_\std\st (a,A)\in\Aff_+(d)\right\}$.
We get a Lie group differential structure on $\calN$~\cite{kwon2009visual} which moreover extends to a statistical Lie group structure in information geometry~\cite{furuhata2021characterization}.

It can be checked that the Fisher-Rao length element is invariant under the action of $\Aff_+(d,\bbR)$
and therefore the Fisher-Rao distance is also invariant: 
$\rho_\FR((a,A).N_0:(a,A).N_1)=\rho_\FR(N_0,N_1)$.
It follows that the Fisher-Rao geodesics in $\calN$ are equivariant~\cite{Eriksen-MVNGeodesic-1986,kobayashi2023geodesics} 
$\gamma_\FR^\calN(B.N_0,B.N_1;t)=B.\gamma_\FR^\calN(N_0,N_1;t)$ for any $B\in\Aff_+(d,\bbR)$ and we can therefore consider without loss of generality that $N_0$ is the standard normal distribution and
$N_1\rightarrow N_1'=N\left(\Sigma_0^{-\frac{1}{2}}(\mu_1,-\mu_0),\Sigma_0^{-\frac{1}{2}}\Sigma_1\Sigma_0^{-\frac{1}{2}} \right)$.

%%%
\subsection{Fisher-Rao geodesics with boundary conditions}\label{sec:FRgeo}
%%%

The Fisher-Rao geodesic Ordinary Differential Equation (ODE) for MVNs was first studied by Skovgaard~\cite{Skovgaard-1984}:
\begin{equation}\label{eq:geodesicODE}
\left\{ \begin{array}{lcl}
\ddot\mu-\dot\Sigma\Sigma^{-1}\dot\mu &=& 0,\\
\ddot\Sigma+\dot\mu\dot\mu^\top-\dot\Sigma\Sigma^{-1}\dot\Sigma &=& 0.
\end{array}
\right.
\end{equation}

\begin{figure}
\centering

\begin{tabular}{cc}
\fbox{\includegraphics[width=0.45\columnwidth]{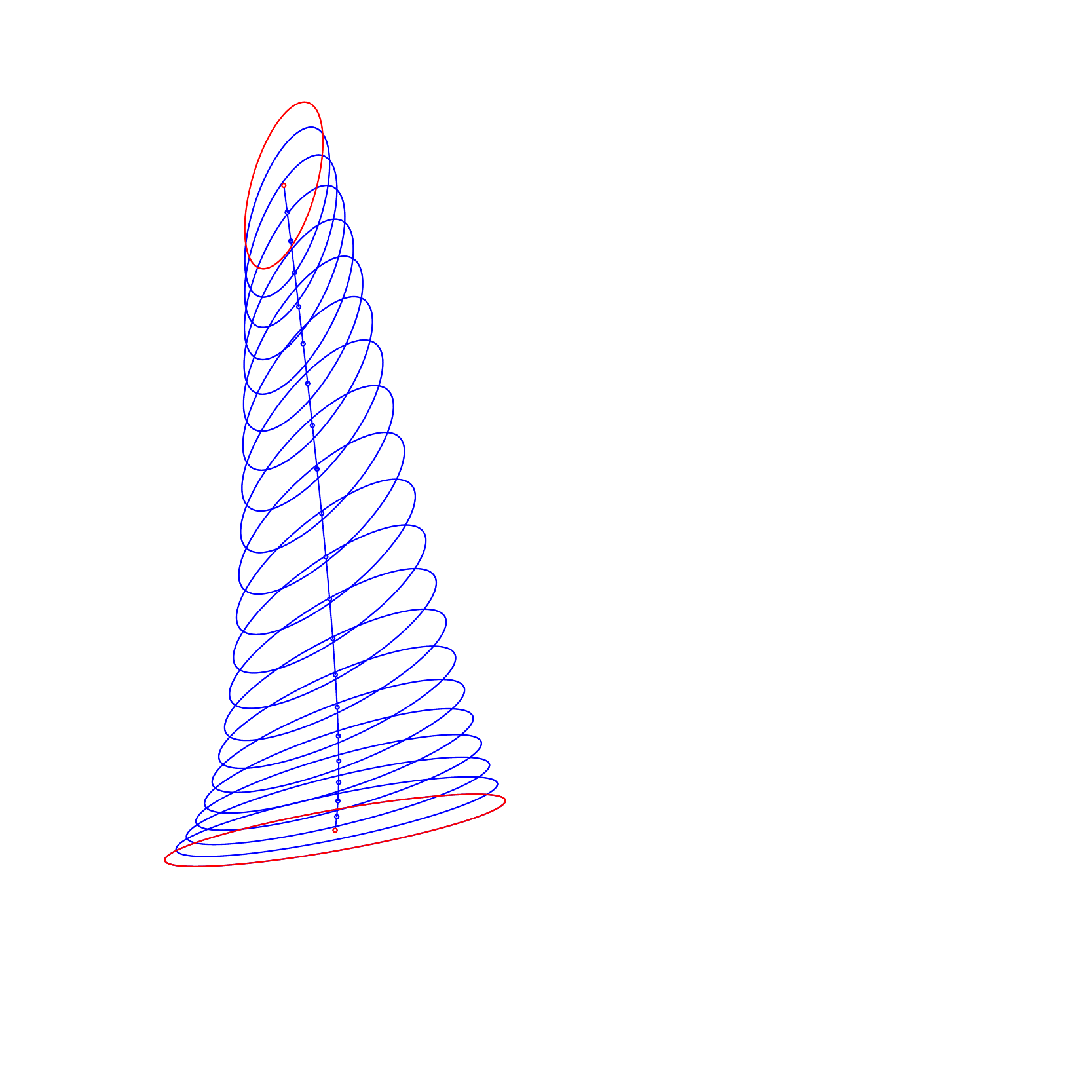}} &
\fbox{\includegraphics[width=0.45\columnwidth]{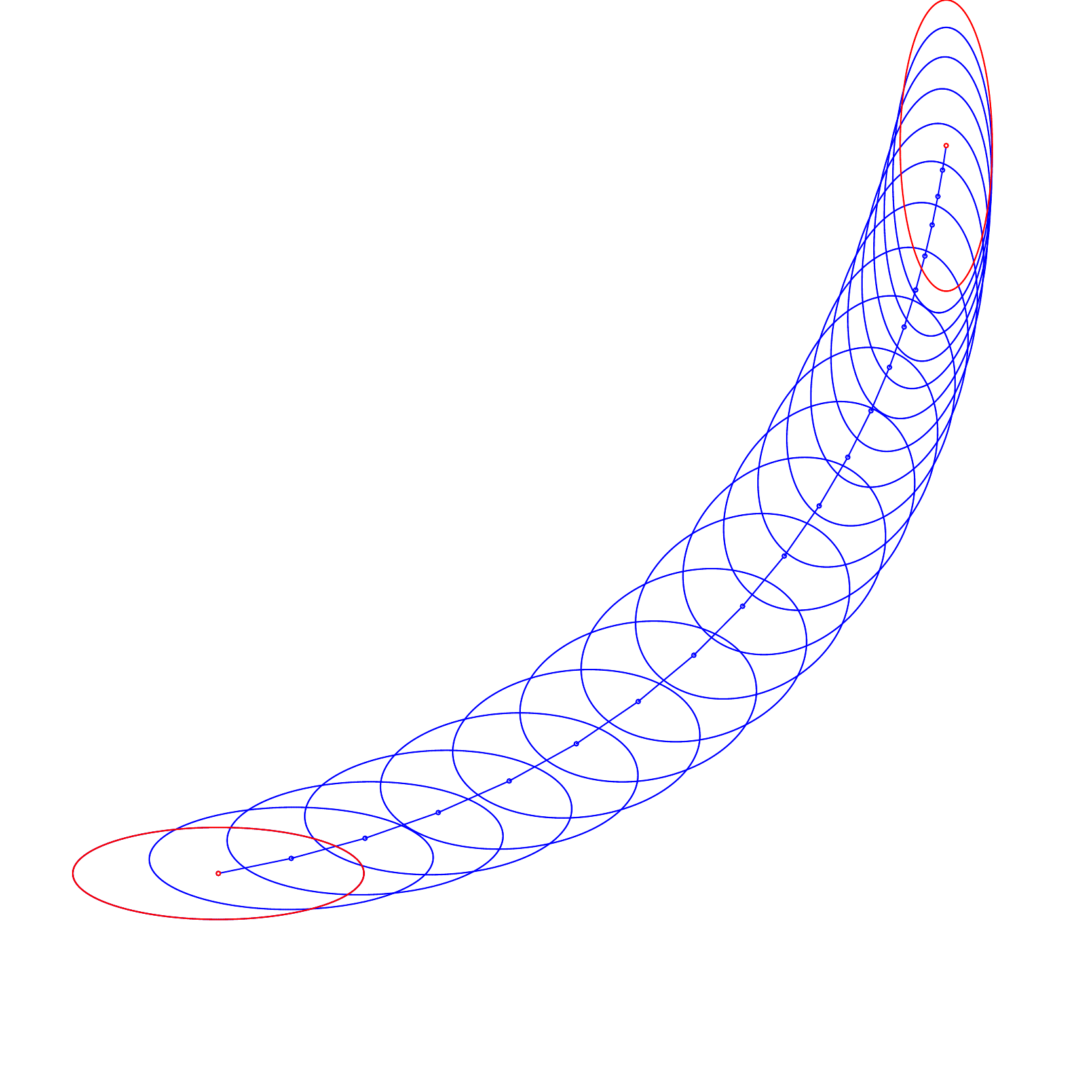}}\cr
(a) & (b)\\
\fbox{\includegraphics[width=0.45\columnwidth]{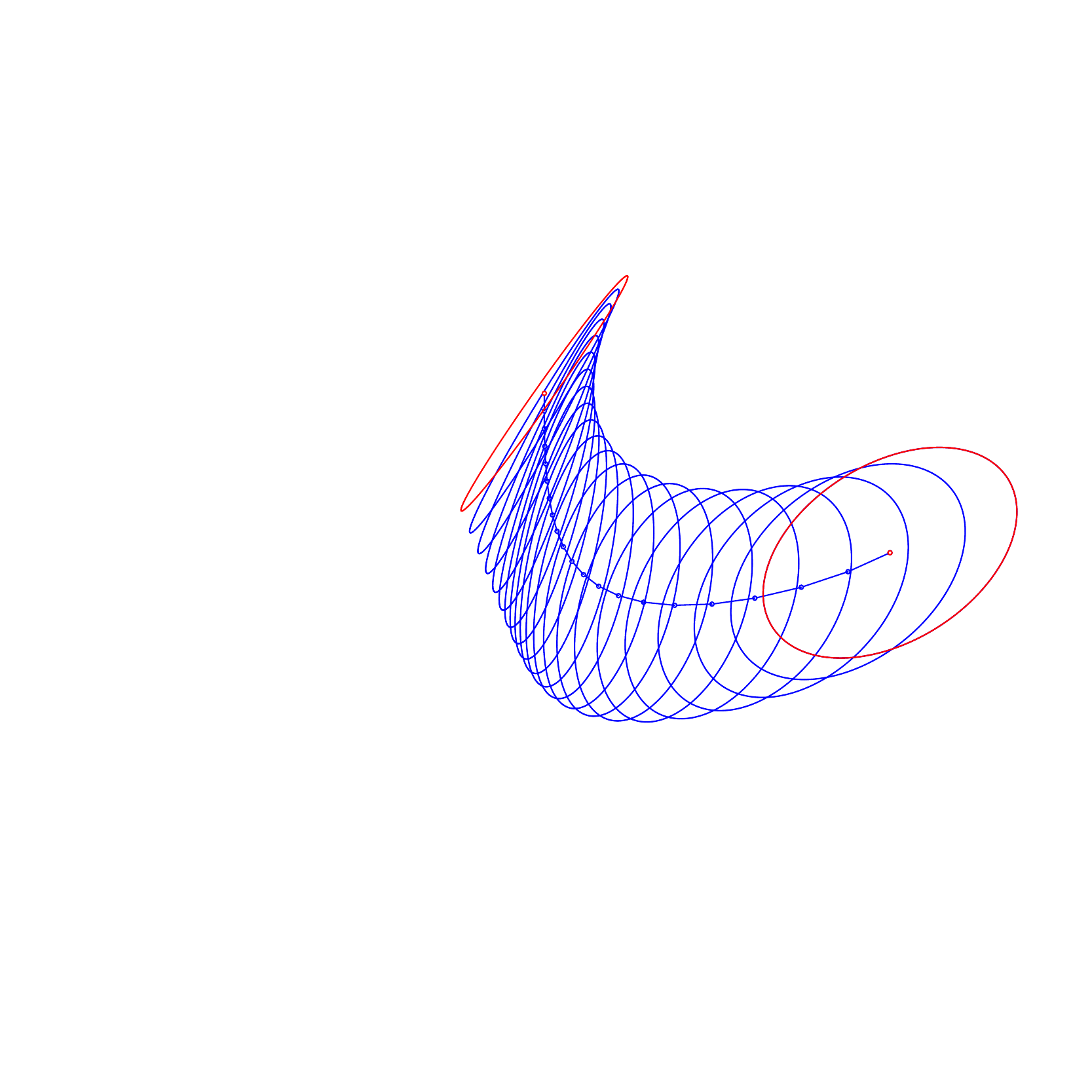}} &
\fbox{\includegraphics[width=0.45\columnwidth]{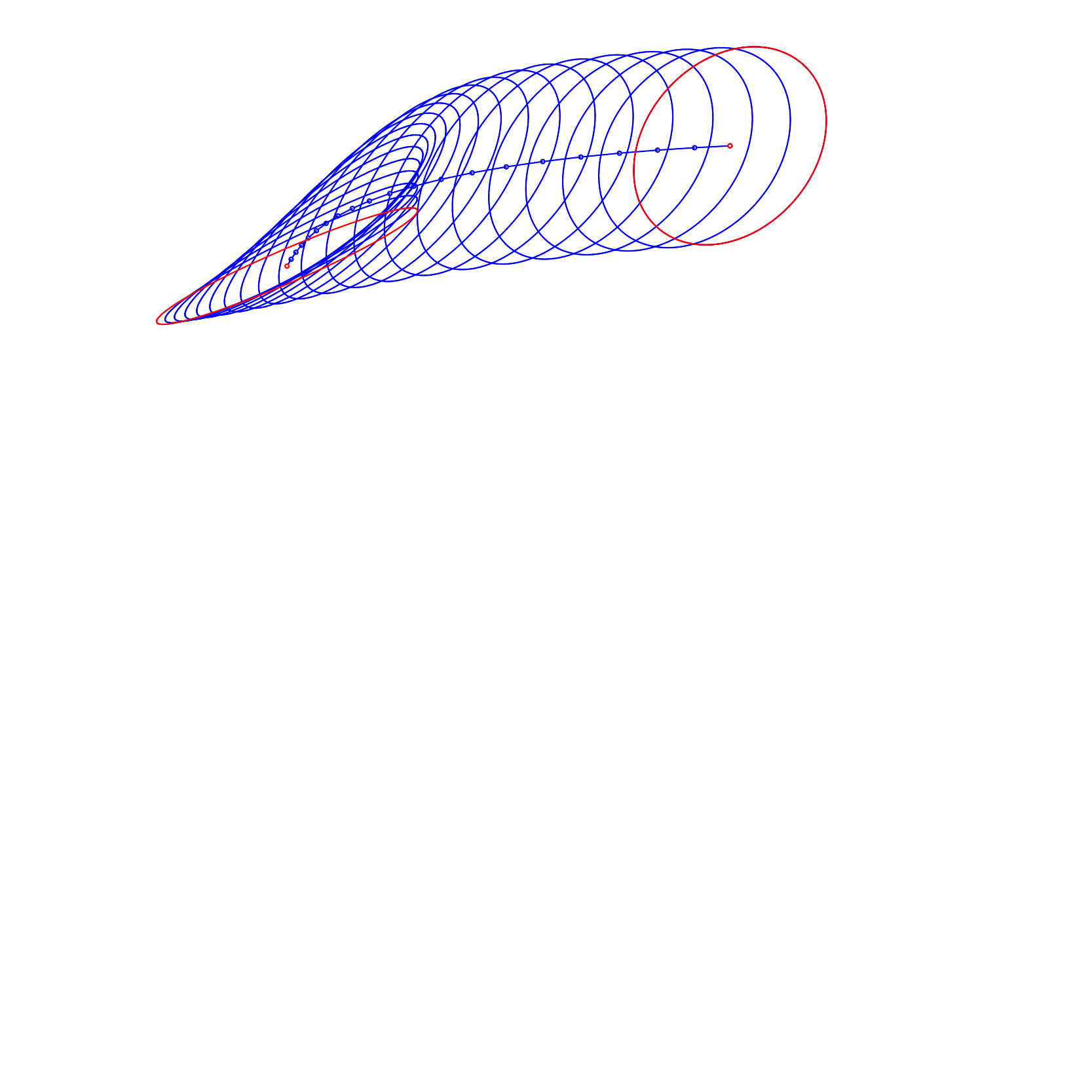}}\\
(c) & (d)
\end{tabular}

%\begin{tabular}{cc}
%\fbox{\includegraphics[width=0.4\columnwidth]{BC-1.pdf}} &
 %\fbox{\includegraphics[width=0.4\columnwidth]{BC-2.pdf}} \cr
%(a) & (b) \\
%\fbox{\includegraphics[width=0.4\columnwidth]{BC-3}} &
%\fbox{\includegraphics[width=0.4\columnwidth]{BC-4}} \cr
%(c) & (d)
%\end{tabular}

\caption{Some Fisher-Rao geodesics $\gamma_\FR^\calN(N_0,N_1)$ with boundary conditions (bivariate normals $N_0$ and $N_1$ indicated by ellipsoids displayed in red) on the bivariate Gaussian manifold.
The sample space $\bbR^2$ is visualized for the range $[-0.3, 1.2]\times [-0.3, 1.2]$.
\label{fig:FRgeodesicBC}}

\end{figure}

Eriksen~\cite{Eriksen-MVNGeodesic-1986} first reported a solution of the geodesic equation with initial conditions:
That is Fisher-Rao geodesics emanating from source $N_0$ with initial prescribed tangent vector $v_0=\dot\gamma(0)$ in the
 tangent plane $T_{N_0}$. The initial tangent vector $v_0=(a_0,B_0)\in T_{N_0}$ consists of a vector part $a_0$ and a symmetric matrix part $B_0$. 
Eriksen's solution required to compute a matrix exponential of a matrix of size $(2d+1)\times (2d+1)$ and the inner use of square matrices of dimension $2d+1$ was not fully geometrically elucidated~\cite{imai2011remarks}.
See Appendix~\ref{app:Eriksen} for details of Eriksen's method which gives a pregeodesic $\underline{\gamma}_\FR^\calN(N_0,v_0;u)$ and not a geodesic ${\gamma}_\FR^\calN(N_0,v_0;t)$ which is parameterized by arclength (constant speed).
Geodesics can be obtained from pregeodesics by smooth reparameterization with an auxiliary function $u(\cdot)$: 
${\gamma}_\FR^\calN(N_0,v_0;t)=\underline{\gamma}_\FR^\calN(N_0,v_0;u(t))$.

Calvo and Oller~\cite{calvo1991explicit} later studied a more general differential equation system than in Eq.~\ref{eq:geodesicODE} and reported a  closed-form solution without using extra dimensions (see Appendix~\ref{app:CO}).
For many years, the Fisher-Rao geodesics with boundary conditions $N_0$ and $N_1$ were not known in closed-form and had to be approximated using geodesic shooting methods~\cite{MVNGeodesicShooting-2014,GeodesicShooting-2016}:
Those geodesic shooting methods were time consuming and numerically unstable, thus limiting their use in applications~\cite{MVNGeodesicShooting-2014}.
A recent breakthrough by Kobayashi~\cite{kobayashi2023geodesics} full explains and extends geometrically the rationale of Eriksen and obtains a method to compute in closed-form the Fisher-Rao geodesic with boundary conditions.
Namely, Kobayashi~\cite{kobayashi2023geodesics} proved that the Fisher-Rao geodesics can be obtained by a Riemannian submersion of horizontal geodesics of the non-compact Riemannian symmetric space of dimension $2d+1$. %$\SO(2d+1,d)/S(O(d+1)\times O(d))\in\Sym_+(2d+1,\bbR)$ 
We report concisely below the recipe which we extracted from  Kobayashi's principled geometric method to derive $N_t=\gamma_\FR^\calN(N_0,N_1;t)$ as follows:
\\
\ \\
\vskip 0.5cm
\noindent\fbox{
\vbox{
\noindent\underline{Algorithm~1. Fisher-Rao geodesic $N_t=N(\mu(t),\Sigma(t))=\gamma_\FR^\calN(N_0,N_1;t)$:}
\begin{itemize}

% corrected in this V3 a typo in swapping M_i and D_i
\item For $i\in\{0,1\}$, let $G_i=M_i\, D_i\, M_i^\top$, where 
\begin{eqnarray}
D_i&=&\matthreethree{\Sigma^{-1}_i}{0}{0}{0}{1}{0}{0}{0}{\Sigma_i},\\
M_i&=&\matthreethree{I_d}{0}{0}{\mu_i^\top}{1}{0}{0}{-\mu_i}{I_d},
\end{eqnarray} 
where $I_d$ denotes the identity matrix of shape $d\times d$.
That is, matrices $G_0$ and $G_1\in\Sym_+(2d+1,\bbR)$ can be expressed by {\em block Cholesky factorizations}.

\item Consider the Riemannian geodesic in $\Sym_+(2d+1,\bbR)$ with respect to the trace metric:
$$
G(t)=G_0^{\frac{1}{2}} \,\left(G_0^{-\frac{1}{2}}G_1G_0^{-\frac{1}{2}} \right)^t\, G_0^{\frac{1}{2}}.
$$

In order to compute the matrix power $G^p$ for $p\in\bbR$, we first calculate the Singular Value Decomposition 
 (SVD) of $G$: $G=O\, L\, O^\top$  (where $O$ is an orthogonal matrix and $L=\diag(\lambda_1,\ldots,\lambda_{2d+1})$ a diagonal matrix) and then get the matrix power as
$G^p=O\, L^p\, O^\top$ with $L^p=\diag(\lambda_1^p,\ldots, \lambda_{2d+1}^p)$.

\item  Retrieve $N(t)=\gamma_\FR^\calN(N_0,N_1;t)=N(\mu(t),\Sigma(t))$ from $G(t)$:

\begin{eqnarray}
\Sigma(t) &=& [G(t)]_{1:d,1:d}^{-1},\\
\mu(t)&=& \Sigma(t)\, [G(t)]_{1:d,d+1},
\end{eqnarray}
where
$[G]_{1:d,1:d}$ denotes the block matrix with rows and columns ranging from $1$ to $d$ extracted from $(2d+1)\times(2d+1)$ matrix $G$, and $[G]_{1:d,d+1}$ is similarly the column vector of $\bbR^d$ extracted from $G$.
\end{itemize}
}
}
\vskip 0.5cm
Appendix~\ref{sec:java} provides the Java\texttrademark{} source code.

Note that this technique also proves that the MVN geodesics are unique although $\calN$ is not NPC.
It is proven in~\cite{furuhata2021characterization} that the Gaussian manifold admits a solvable Lie group and hence is diffeomorphic to some Euclidean space.
Figure~\ref{fig:FRgeodesicBC} displays several bivariate normal Fisher-Rao geodesics with boundary conditions obtained by implementing this method~\cite{kobayashi2023geodesics}.
We display $N(\mu,\Sigma)$ by an ellipse $\calE=\{\mu+Lx \st \|x\|_2=1\}$ where $\Sigma=LL^\top$ (Cholesky decomposition).

%%%
\subsection{Fisher-Rao distances}\label{sec:FRapprox}
%%%

\subsubsection{Approximating Fisher-Rao lengths of curves}

The previous section reported the closed-form solutions for the Fisher-Rao geodesics $\gamma_\FR^\calN(N_0,N_1;t)$.
We shall now explain a method to approximate their lengths and hence the Fisher-Rao distances:
$$
\rho_\FR(N_0,N_1)=\Length(\gamma_\FR^\calN(N_0,N_1;t)).
$$

Consider discretizing regularly $t\in[0,1]$ using $T+1$ steps:
$\frac{0}{T}=0, \frac{1}{T}, \ldots, \frac{T-1}{T}, \frac{T}{T}=1$.
Since geodesics are totally 1D submanifolds, we have 
$$
\rho_\FR(N_0,N_1)=\sum_{i=0}^{T-1} \rho_\FR(\gamma_\FR\left(N_{\frac{i}{T}},N_{\frac{i+1}{T}}\right)).
$$

By choosing $T$ large enough, we have $N=N_{\frac{i}{T}}$ close to $N'=N_{\frac{i+1}{T}}$, and we can approximate the geodesic distance as follows:
$$
\rho_\FR(N,N')\approx\ds_\Fisher(N)\approx \sqrt{\frac{2}{f''(1)}I_f(N,N')},
$$
where $I_f(p,q)$ is {\em any} $f$-divergence~\cite{fdiv-AliSilvey-1966,Csiszar-1967} between pdfs $p(x)$ and $q(x)$  induced by a strictly convex generator $f(u)$ satisfying $f(1)=0$:
$$
I_f(p,q)=\int p(x)f\left(\frac{q(x)}{p(x)}\right)\, \dx.
$$
Indeed, we have for two close distributions $p_\theta$ and $p_{\theta+\dtheta}$~\cite{IG-2016}:
$I_f(p_\theta,p_{\theta+\dtheta})\approx \frac{f''(1)}{2} \ds^2_\Fisher$.

We choose the Jeffreys $f$-divergence which is the arithmetic symmetrization of the Kullback-Leibler divergence obtained for the generator $f_J(u)=(u-1)\log u$ with $f_J''(1)=2$.
It follows that $D_J(N_1,N_2)=I_{f_J}(N_1,N_2)
=\tr\left(\frac{\Sigma_2^{-1}\Sigma_1+\Sigma_1^{-1}\Sigma_2}{2}-I\right) 
 +(\mu_2-\mu_1)^\top\frac{\Sigma_1^{-1}+\Sigma_2^{-1}}{2}(\mu_2-\mu_1)$.

Thus we get the following overall approximation of the Fisher-Rao distance:
\begin{equation}
 \tilde\rho_T(N_0,N_1)
=\sum_{i=0}^{T-1} \sqrt{D_J\left(N_{\frac{i}{T}},N_{\frac{i+1}{T}}\right)}
\approx\rho_\FR(N_0,N_1).
\end{equation}

In~\cite{gao2021information}, the authors choose $\ds_\Fisher(p)=\sqrt{2D_\KL(p_\theta,p_{\theta+\dtheta})}$ where $D_\KL=I_{f_\KL}$ is the Kullback-Leibler divergence, a $f$-divergence obtained for $f_\KL(u)=-\log u$.

\begin{figure}
\centering
\begin{tabular}{cc}
\fbox{\includegraphics[width=0.5\columnwidth]{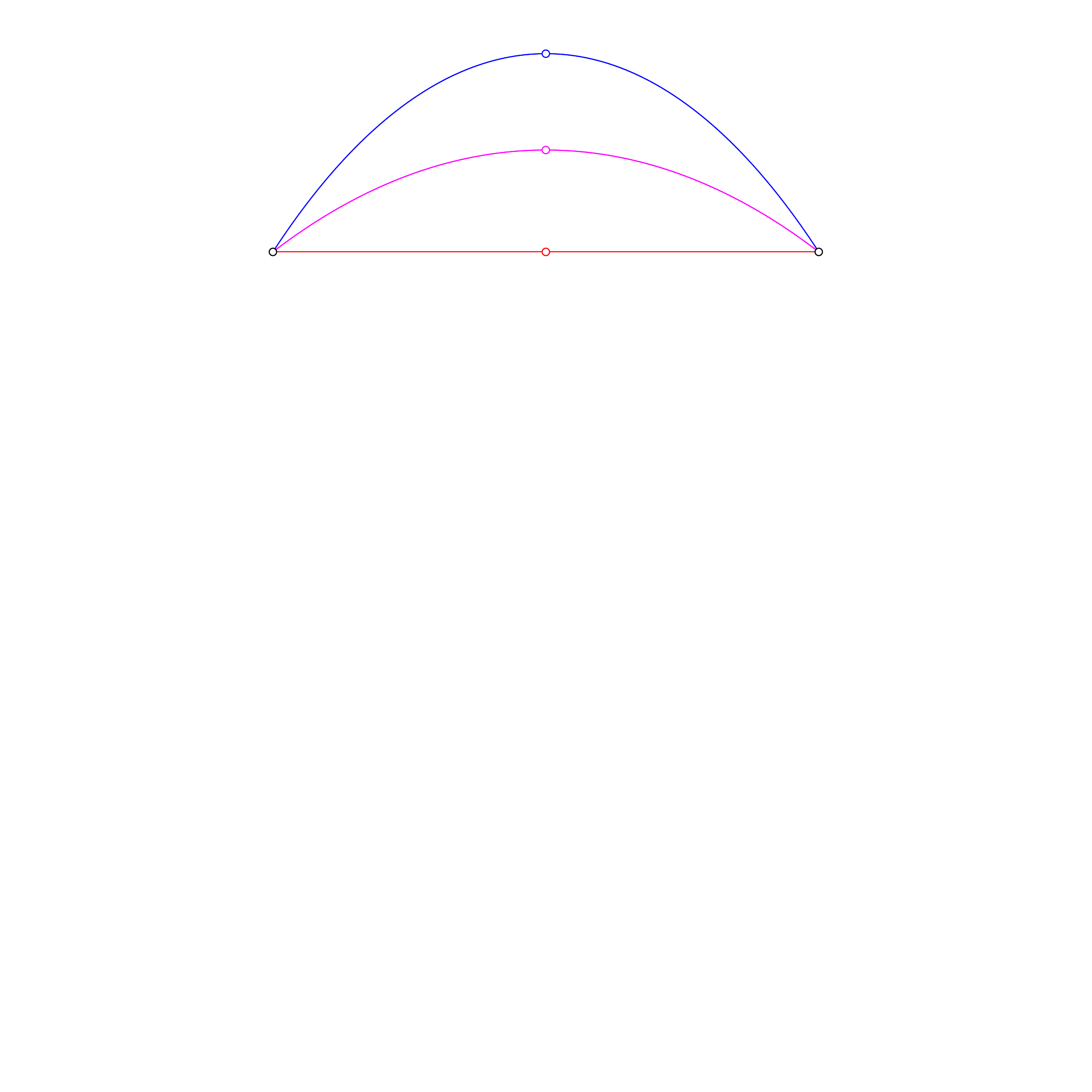}} &
\fbox{\includegraphics[width=0.5\columnwidth]{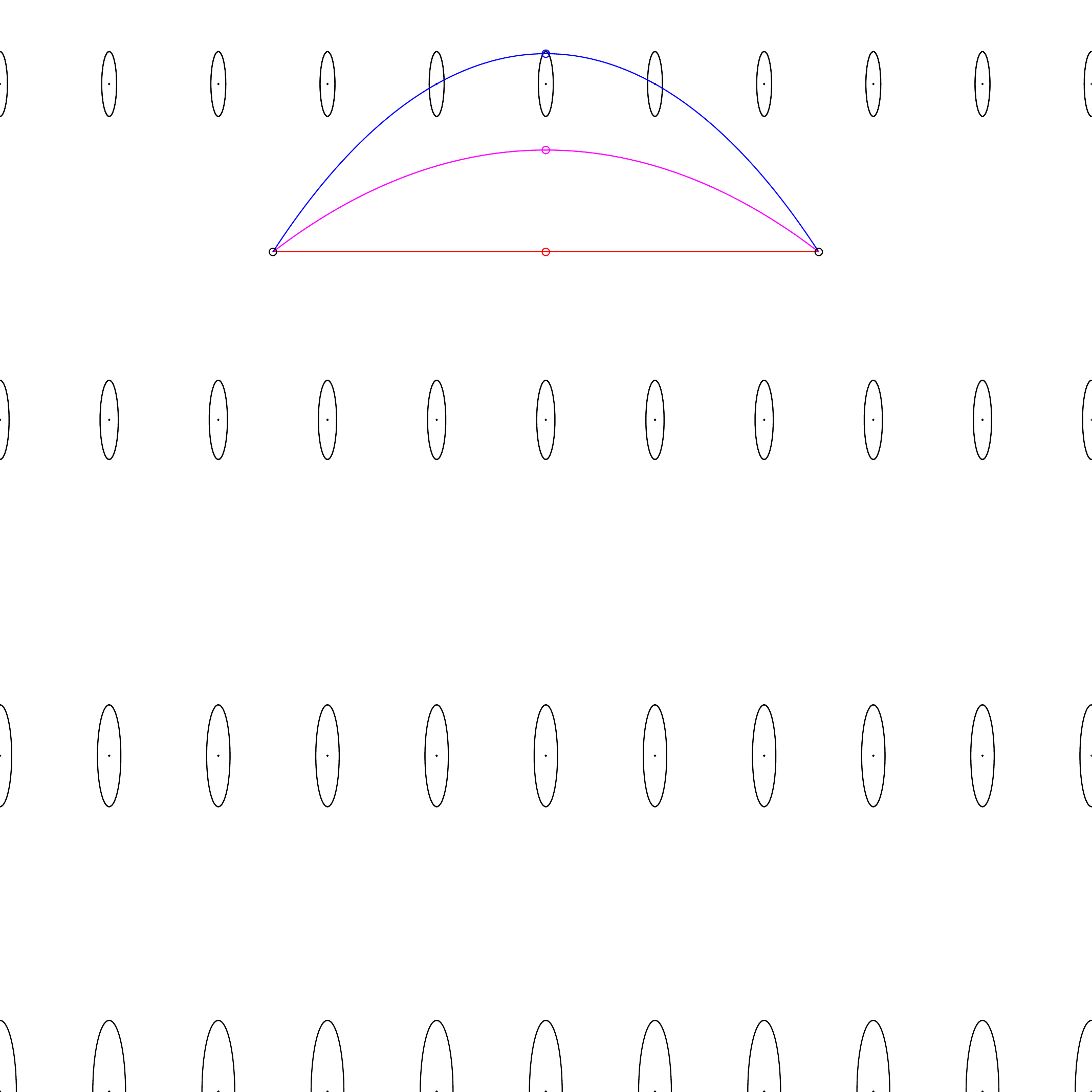}}  
\end{tabular}
\caption{Visualizing on the upper plane $(\mu,\sigma)$ the Fisher-Rao geodesic (purple), and dual exponential geodesic (red) and mixture geodesic (blue) between univariate normal distributions $N_0=N(0,1)$ and $N_1=N(1,1)$.
Left: geodesics with their corresponding midpoints. Right: same as Left with Tissot indicatrices shown
 \label{fig:geodesics}}
\end{figure}

\begin{property}[Fisher-Rao upper bound]\label{prop:UBJ}
The Fisher-Rao distance between   normal distributions is upper bounded by the square root of the Jeffreys divergence: 
$\rho_\FR(N_0,N_1) \leq\sqrt{D_J(N_0,N_1)}$.
\end{property}
The proof can be found in many places, e.g.~\cite{grosse2013annealing,IG-2016,rong2017intrinsic}.
Yet we report another proof in Appendix~\ref{sec:proof}.

Notice that we have $\rho_\FR(N_0,N_1)\leq \tilde\rho_T(N_0,N_1)$ for all $T>1$.
Define the energy of a curve $c(t)$ with $t\in[a,b]$ by 
$E(c)=\int_a^b \ds^2_\Fisher(t)\dt$.
We have 
$$
E(\gamma_e^\calN(N_0,N_1;t))=E(\gamma_m^\calN(N_0,N_1;t))=D_J(N_0,N_1),
$$
where $\gamma_e^\calN(N_0,N_1;t)=N(\mu_t^e,\Sigma_t^e)$ and $\gamma_m^\calN(N_0,N_1;t)=N(\mu_t^m,\Sigma_t^m)$ are the exponential and mixture geodesics in information geometry~\cite{IG-MVN-1999} given by $\mu_t^m =\bar\mu_t$ and $\Sigma_t^m = \bar\Sigma_t+t\mu_1\mu_1^\top+(1-t)\mu_2\mu_2^\top-\bar\mu_t\bar\mu_t^\top$
where   $\bar\mu_t=t\mu_1+(1-t)\mu_2$ and $\bar\Sigma_t=t\Sigma_1+(1-t)\Sigma_2$,
and $\mu_t^e = \bar\Sigma_t^H (t\Sigma_1^{-1}\mu_1+(1-t)\Sigma_2^{-1}\mu_2)$
and $\Sigma_t^e =\bar\Sigma^H_t$
where $\bar\Sigma^H_t=(t\Sigma_1^{-1}+(1-t)\Sigma_2^{-1})^{-1}$ is the matrix harmonic mean.
See Figure~\ref{fig:midETFR}. 
The mixture, Fisher-Rao, and exponential geodesics are $\alpha$-connection geodesics~\cite{furuhata2021characterization} for $\alpha=-1$, $\alpha=0$ and $\alpha=1$, respectively.
Notice that these e/m geodesics are computationally less intensive to evaluate than the Fisher-Rao geodesics.
Figure~\ref{fig:geodesics} displays an example of the Fisher-Rao geodesics and the dual e/m geodesics between univariate normal distributions.

%\begin{figure}
%\centering
%\includegraphics[width=\columnwidth]{FigIpe-BivariateInterpolation.pdf} 
%\caption{Visualizing the exponential and mixture geodesics between two bivariate normal distributions.\label{fig:vizemgeo2d}}
%\end{figure}

\begin{figure}
\centering
\begin{tabular}{ccc}
\fbox{\includegraphics[width=0.3\columnwidth]{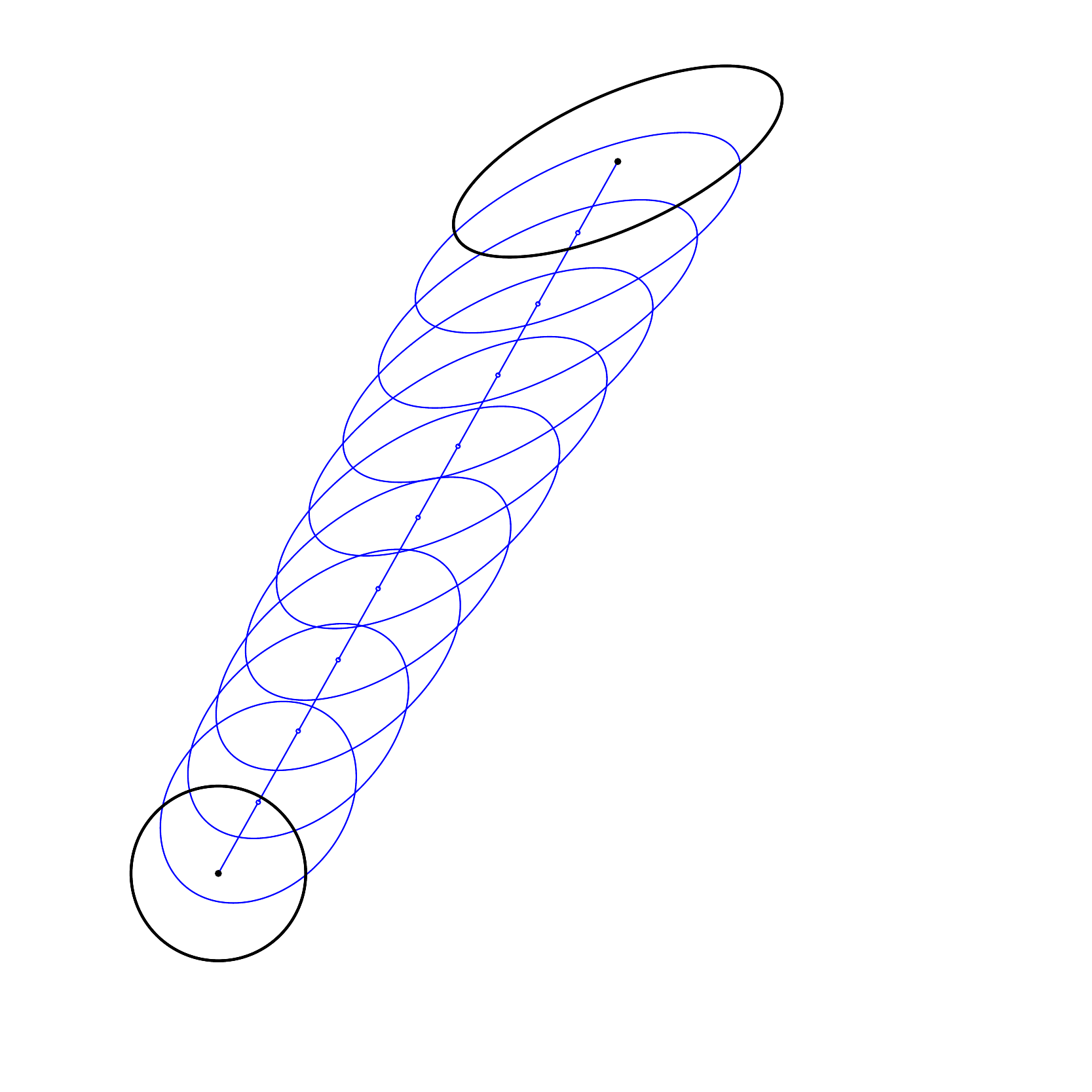}} &
\fbox{\includegraphics[width=0.3\columnwidth]{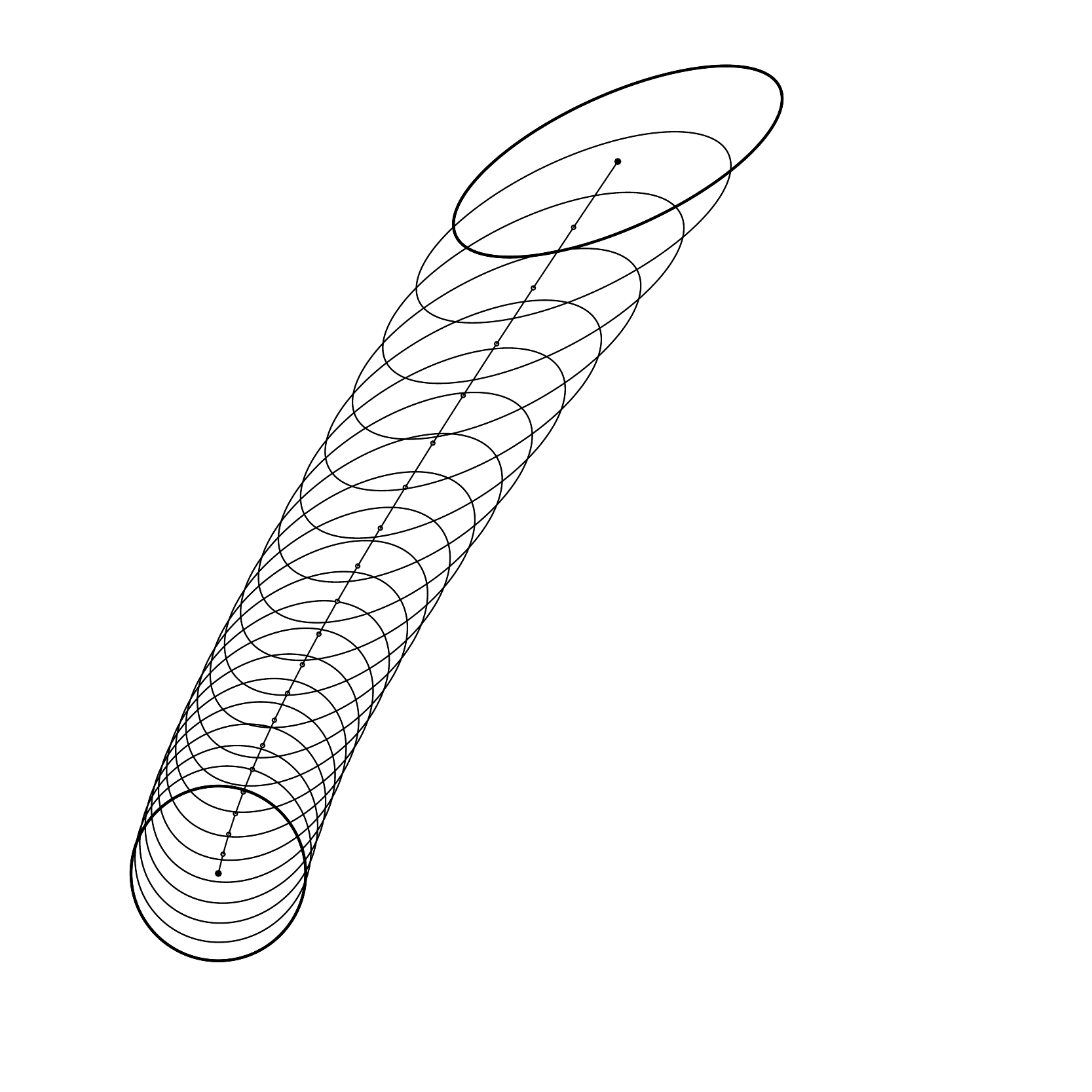}} &
\fbox{\includegraphics[width=0.3\columnwidth]{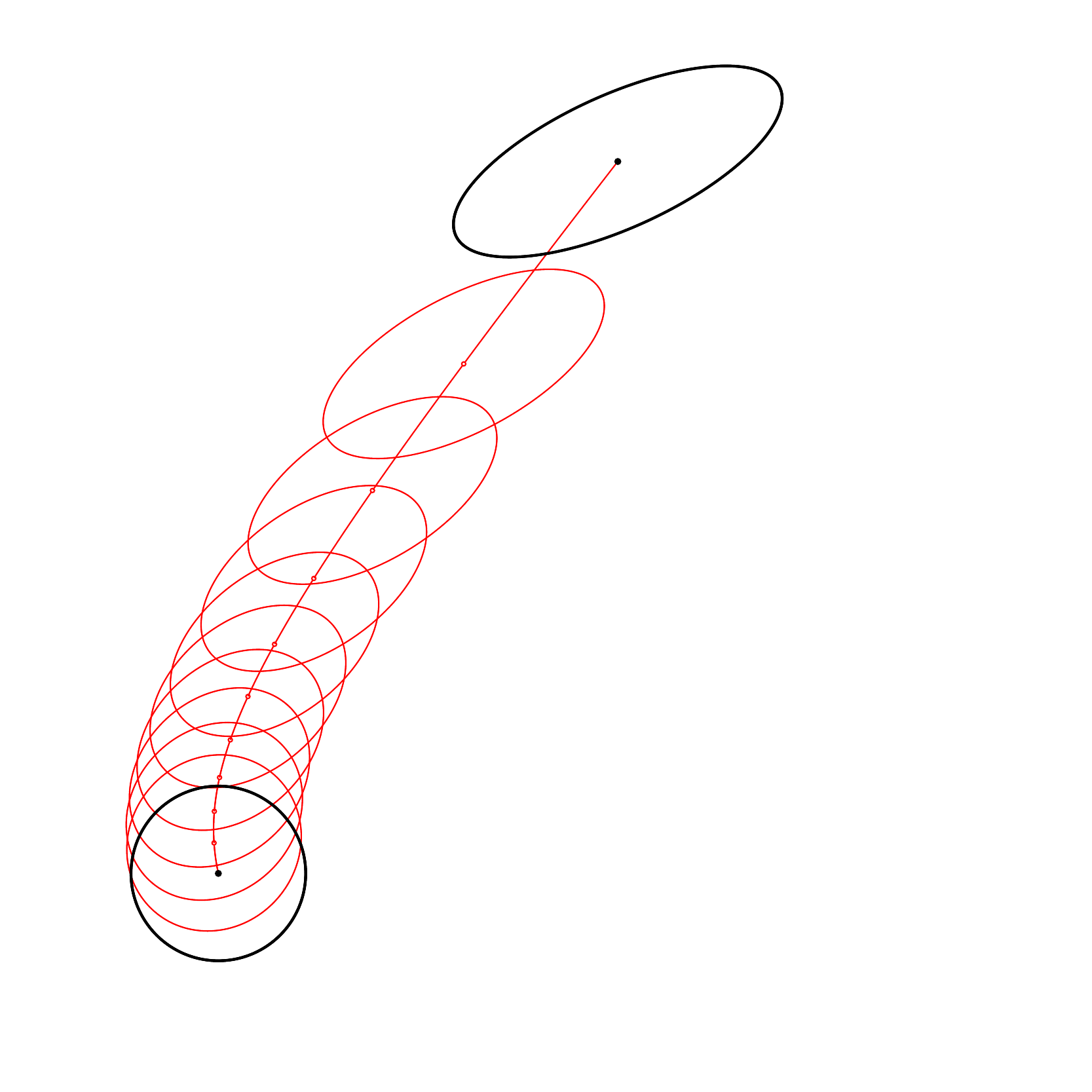}}\\
mixture geodesic $\gamma_m^\calN$ & Fisher-Rao geodesic $\gamma_\calN^\FR$ & exponential geodesic $\gamma^\calN_e$
\end{tabular}
\caption{Visualizing geodesics with respect to the mixture, Levi-Civita (Fisher-Rao), and exponential connections.
 \label{fig:midETFR}}
\end{figure}

%Thus the inequality gap $D_J(N_0,N_1)-\rho_\FR^2(N_0,N_1)$ can be understood as an energy gap between the mixture/exponential geodesic and the Fisher-Rao geodesic:
%\begin{eqnarray*}
%D_J(N_0,N_1)-\rho_\FR^2(N_0,N_1)&=& E(\gamma_e(N_0,N_1;t))-E(\gamma_\FR(N_0,N_1;t)),\\
%&=& E(\gamma_m(N_0,N_1;t))-E(\gamma_\FR(N_0,N_1;t)).
%\end{eqnarray*}
%The dissimilarity $\rho_\FR^2(N_0,N_1)$ is called the Fisher-Rao divergence (not a metric because it violates the triangle inequality).

Since the upper bound of Property~\ref{prop:UBJ} is tight infinitesimally, we get in the limit convergence to the Fisher-Rao distance:
$$
\lim_{T\rightarrow\infty} \tilde\rho_T(N_0,N_1) = \rho_\FR(N_0,N_1).
$$
%Moreover, we have a monotone convergence $\tilde\rho_{T'}(N_0,N_1)\leq \tilde\rho_T(N_0,N_1)$ whenever $T'\geq T$.

\begin{example}
Let us consider the example of Han and Park~\cite{MVNGeodesicShooting-2014} (displayed in Figure~\ref{fig:FRgeodesicBC}(b)):

\noindent $N_0=N\left(\vectortwo{0}{0},\mattwotwo{1}{0}{0}{0.1}\right)$ and  
$N_1=N\left(\vectortwo{1}{1},\mattwotwo{0.1}{0}{0}{1}\right)$.
The time consuming geodesic shooting algorithm of~\cite{MVNGeodesicShooting-2014} evaluates the Fisher-Rao distance to 
$\rho_\calN(N_0,N_1) \approx 3.1329$.
We get the following approximations:
$\tilde\rho_T(N_0,N_1)=3.1996$ for $T=100$.
See Figure~\ref{fig:convergence} for the convergence curve of $\tilde\rho_T(N_0,N_1)$ as a function of $T$.
\end{example}

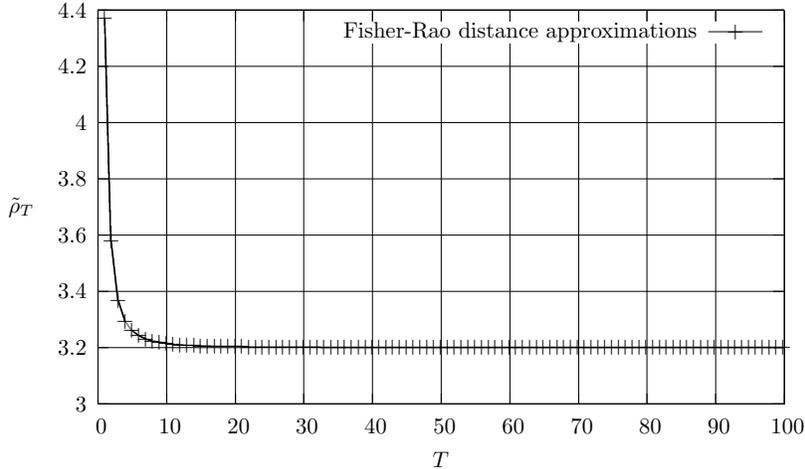
\begin {figure}
\begin{center}
\scalebox{0.85}{% GNUPLOT: LaTeX picture
\setlength{\unitlength}{0.240900pt}
\ifx\plotpoint\undefined\newsavebox{\plotpoint}\fi
\sbox{\plotpoint}{\rule[-0.200pt]{0.400pt}{0.400pt}}%
\begin{picture}(1500,900)(0,0)
\sbox{\plotpoint}{\rule[-0.200pt]{0.400pt}{0.400pt}}%
\put(171.0,131.0){\rule[-0.200pt]{305.461pt}{0.400pt}}
\put(171.0,131.0){\rule[-0.200pt]{4.818pt}{0.400pt}}
\put(151,131){\makebox(0,0)[r]{ 3}}
\put(1419.0,131.0){\rule[-0.200pt]{4.818pt}{0.400pt}}
\put(171.0,235.0){\rule[-0.200pt]{305.461pt}{0.400pt}}
\put(171.0,235.0){\rule[-0.200pt]{4.818pt}{0.400pt}}
\put(151,235){\makebox(0,0)[r]{ 3.2}}
\put(1419.0,235.0){\rule[-0.200pt]{4.818pt}{0.400pt}}
\put(171.0,339.0){\rule[-0.200pt]{305.461pt}{0.400pt}}
\put(171.0,339.0){\rule[-0.200pt]{4.818pt}{0.400pt}}
\put(151,339){\makebox(0,0)[r]{ 3.4}}
\put(1419.0,339.0){\rule[-0.200pt]{4.818pt}{0.400pt}}
\put(171.0,443.0){\rule[-0.200pt]{305.461pt}{0.400pt}}
\put(171.0,443.0){\rule[-0.200pt]{4.818pt}{0.400pt}}
\put(151,443){\makebox(0,0)[r]{ 3.6}}
\put(1419.0,443.0){\rule[-0.200pt]{4.818pt}{0.400pt}}
\put(171.0,547.0){\rule[-0.200pt]{305.461pt}{0.400pt}}
\put(171.0,547.0){\rule[-0.200pt]{4.818pt}{0.400pt}}
\put(151,547){\makebox(0,0)[r]{ 3.8}}
\put(1419.0,547.0){\rule[-0.200pt]{4.818pt}{0.400pt}}
\put(171.0,651.0){\rule[-0.200pt]{305.461pt}{0.400pt}}
\put(171.0,651.0){\rule[-0.200pt]{4.818pt}{0.400pt}}
\put(151,651){\makebox(0,0)[r]{ 4}}
\put(1419.0,651.0){\rule[-0.200pt]{4.818pt}{0.400pt}}
\put(171.0,755.0){\rule[-0.200pt]{305.461pt}{0.400pt}}
\put(171.0,755.0){\rule[-0.200pt]{4.818pt}{0.400pt}}
\put(151,755){\makebox(0,0)[r]{ 4.2}}
\put(1419.0,755.0){\rule[-0.200pt]{4.818pt}{0.400pt}}
\put(171.0,859.0){\rule[-0.200pt]{305.461pt}{0.400pt}}
\put(171.0,859.0){\rule[-0.200pt]{4.818pt}{0.400pt}}
\put(151,859){\makebox(0,0)[r]{ 4.4}}
\put(1419.0,859.0){\rule[-0.200pt]{4.818pt}{0.400pt}}
\put(171.0,131.0){\rule[-0.200pt]{0.400pt}{175.375pt}}
\put(171.0,131.0){\rule[-0.200pt]{0.400pt}{4.818pt}}
\put(171,90){\makebox(0,0){ 0}}
\put(171.0,839.0){\rule[-0.200pt]{0.400pt}{4.818pt}}
\put(298.0,131.0){\rule[-0.200pt]{0.400pt}{175.375pt}}
\put(298.0,131.0){\rule[-0.200pt]{0.400pt}{4.818pt}}
\put(298,90){\makebox(0,0){ 10}}
\put(298.0,839.0){\rule[-0.200pt]{0.400pt}{4.818pt}}
\put(425.0,131.0){\rule[-0.200pt]{0.400pt}{175.375pt}}
\put(425.0,131.0){\rule[-0.200pt]{0.400pt}{4.818pt}}
\put(425,90){\makebox(0,0){ 20}}
\put(425.0,839.0){\rule[-0.200pt]{0.400pt}{4.818pt}}
\put(551.0,131.0){\rule[-0.200pt]{0.400pt}{175.375pt}}
\put(551.0,131.0){\rule[-0.200pt]{0.400pt}{4.818pt}}
\put(551,90){\makebox(0,0){ 30}}
\put(551.0,839.0){\rule[-0.200pt]{0.400pt}{4.818pt}}
\put(678.0,131.0){\rule[-0.200pt]{0.400pt}{160.680pt}}
\put(678.0,839.0){\rule[-0.200pt]{0.400pt}{4.818pt}}
\put(678.0,131.0){\rule[-0.200pt]{0.400pt}{4.818pt}}
\put(678,90){\makebox(0,0){ 40}}
\put(678.0,839.0){\rule[-0.200pt]{0.400pt}{4.818pt}}
\put(805.0,131.0){\rule[-0.200pt]{0.400pt}{160.680pt}}
\put(805.0,839.0){\rule[-0.200pt]{0.400pt}{4.818pt}}
\put(805.0,131.0){\rule[-0.200pt]{0.400pt}{4.818pt}}
\put(805,90){\makebox(0,0){ 50}}
\put(805.0,839.0){\rule[-0.200pt]{0.400pt}{4.818pt}}
\put(932.0,131.0){\rule[-0.200pt]{0.400pt}{160.680pt}}
\put(932.0,839.0){\rule[-0.200pt]{0.400pt}{4.818pt}}
\put(932.0,131.0){\rule[-0.200pt]{0.400pt}{4.818pt}}
\put(932,90){\makebox(0,0){ 60}}
\put(932.0,839.0){\rule[-0.200pt]{0.400pt}{4.818pt}}
\put(1059.0,131.0){\rule[-0.200pt]{0.400pt}{160.680pt}}
\put(1059.0,839.0){\rule[-0.200pt]{0.400pt}{4.818pt}}
\put(1059.0,131.0){\rule[-0.200pt]{0.400pt}{4.818pt}}
\put(1059,90){\makebox(0,0){ 70}}
\put(1059.0,839.0){\rule[-0.200pt]{0.400pt}{4.818pt}}
\put(1185.0,131.0){\rule[-0.200pt]{0.400pt}{160.680pt}}
\put(1185.0,839.0){\rule[-0.200pt]{0.400pt}{4.818pt}}
\put(1185.0,131.0){\rule[-0.200pt]{0.400pt}{4.818pt}}
\put(1185,90){\makebox(0,0){ 80}}
\put(1185.0,839.0){\rule[-0.200pt]{0.400pt}{4.818pt}}
\put(1312.0,131.0){\rule[-0.200pt]{0.400pt}{160.680pt}}
\put(1312.0,839.0){\rule[-0.200pt]{0.400pt}{4.818pt}}
\put(1312.0,131.0){\rule[-0.200pt]{0.400pt}{4.818pt}}
\put(1312,90){\makebox(0,0){ 90}}
\put(1312.0,839.0){\rule[-0.200pt]{0.400pt}{4.818pt}}
\put(1439.0,131.0){\rule[-0.200pt]{0.400pt}{175.375pt}}
\put(1439.0,131.0){\rule[-0.200pt]{0.400pt}{4.818pt}}
\put(1439,90){\makebox(0,0){ 100}}
\put(1439.0,839.0){\rule[-0.200pt]{0.400pt}{4.818pt}}
\put(171.0,131.0){\rule[-0.200pt]{0.400pt}{175.375pt}}
\put(171.0,131.0){\rule[-0.200pt]{305.461pt}{0.400pt}}
\put(1439.0,131.0){\rule[-0.200pt]{0.400pt}{175.375pt}}
\put(171.0,859.0){\rule[-0.200pt]{305.461pt}{0.400pt}}
\put(30,495){\makebox(0,0){$\tilde\rho_T$}}
\put(805,29){\makebox(0,0){$T$}}
\put(1279,819){\makebox(0,0)[r]{Fisher-Rao distance approximations}}
\put(1299.0,819.0){\rule[-0.200pt]{24.090pt}{0.400pt}}
\put(184,844){\usebox{\plotpoint}}
\multiput(184.58,786.58)(0.492,-17.729){21}{\rule{0.119pt}{13.833pt}}
\multiput(183.17,815.29)(12.000,-383.288){2}{\rule{0.400pt}{6.917pt}}
\multiput(196.58,417.66)(0.493,-4.303){23}{\rule{0.119pt}{3.454pt}}
\multiput(195.17,424.83)(13.000,-101.831){2}{\rule{0.400pt}{1.727pt}}
\multiput(209.58,317.60)(0.493,-1.527){23}{\rule{0.119pt}{1.300pt}}
\multiput(208.17,320.30)(13.000,-36.302){2}{\rule{0.400pt}{0.650pt}}
\multiput(222.58,281.23)(0.492,-0.712){21}{\rule{0.119pt}{0.667pt}}
\multiput(221.17,282.62)(12.000,-15.616){2}{\rule{0.400pt}{0.333pt}}
\multiput(234.00,265.92)(0.652,-0.491){17}{\rule{0.620pt}{0.118pt}}
\multiput(234.00,266.17)(11.713,-10.000){2}{\rule{0.310pt}{0.400pt}}
\multiput(247.00,255.93)(1.123,-0.482){9}{\rule{0.967pt}{0.116pt}}
\multiput(247.00,256.17)(10.994,-6.000){2}{\rule{0.483pt}{0.400pt}}
\multiput(260.00,249.94)(1.651,-0.468){5}{\rule{1.300pt}{0.113pt}}
\multiput(260.00,250.17)(9.302,-4.000){2}{\rule{0.650pt}{0.400pt}}
\put(272,245.17){\rule{2.700pt}{0.400pt}}
\multiput(272.00,246.17)(7.396,-2.000){2}{\rule{1.350pt}{0.400pt}}
\put(285,243.17){\rule{2.700pt}{0.400pt}}
\multiput(285.00,244.17)(7.396,-2.000){2}{\rule{1.350pt}{0.400pt}}
\put(298,241.17){\rule{2.500pt}{0.400pt}}
\multiput(298.00,242.17)(6.811,-2.000){2}{\rule{1.250pt}{0.400pt}}
\put(310,239.67){\rule{3.132pt}{0.400pt}}
\multiput(310.00,240.17)(6.500,-1.000){2}{\rule{1.566pt}{0.400pt}}
\put(323,238.67){\rule{3.132pt}{0.400pt}}
\multiput(323.00,239.17)(6.500,-1.000){2}{\rule{1.566pt}{0.400pt}}
\put(349,237.67){\rule{2.891pt}{0.400pt}}
\multiput(349.00,238.17)(6.000,-1.000){2}{\rule{1.445pt}{0.400pt}}
\put(336.0,239.0){\rule[-0.200pt]{3.132pt}{0.400pt}}
\put(374,236.67){\rule{3.132pt}{0.400pt}}
\multiput(374.00,237.17)(6.500,-1.000){2}{\rule{1.566pt}{0.400pt}}
\put(361.0,238.0){\rule[-0.200pt]{3.132pt}{0.400pt}}
\put(437,235.67){\rule{3.132pt}{0.400pt}}
\multiput(437.00,236.17)(6.500,-1.000){2}{\rule{1.566pt}{0.400pt}}
\put(387.0,237.0){\rule[-0.200pt]{12.045pt}{0.400pt}}
\put(577,234.67){\rule{2.891pt}{0.400pt}}
\multiput(577.00,235.17)(6.000,-1.000){2}{\rule{1.445pt}{0.400pt}}
\put(450.0,236.0){\rule[-0.200pt]{30.594pt}{0.400pt}}
\put(184,844){\makebox(0,0){$+$}}
\put(196,432){\makebox(0,0){$+$}}
\put(209,323){\makebox(0,0){$+$}}
\put(222,284){\makebox(0,0){$+$}}
\put(234,267){\makebox(0,0){$+$}}
\put(247,257){\makebox(0,0){$+$}}
\put(260,251){\makebox(0,0){$+$}}
\put(272,247){\makebox(0,0){$+$}}
\put(285,245){\makebox(0,0){$+$}}
\put(298,243){\makebox(0,0){$+$}}
\put(310,241){\makebox(0,0){$+$}}
\put(323,240){\makebox(0,0){$+$}}
\put(336,239){\makebox(0,0){$+$}}
\put(349,239){\makebox(0,0){$+$}}
\put(361,238){\makebox(0,0){$+$}}
\put(374,238){\makebox(0,0){$+$}}
\put(387,237){\makebox(0,0){$+$}}
\put(399,237){\makebox(0,0){$+$}}
\put(412,237){\makebox(0,0){$+$}}
\put(425,237){\makebox(0,0){$+$}}
\put(437,237){\makebox(0,0){$+$}}
\put(450,236){\makebox(0,0){$+$}}
\put(463,236){\makebox(0,0){$+$}}
\put(475,236){\makebox(0,0){$+$}}
\put(488,236){\makebox(0,0){$+$}}
\put(501,236){\makebox(0,0){$+$}}
\put(513,236){\makebox(0,0){$+$}}
\put(526,236){\makebox(0,0){$+$}}
\put(539,236){\makebox(0,0){$+$}}
\put(551,236){\makebox(0,0){$+$}}
\put(564,236){\makebox(0,0){$+$}}
\put(577,236){\makebox(0,0){$+$}}
\put(589,235){\makebox(0,0){$+$}}
\put(602,235){\makebox(0,0){$+$}}
\put(615,235){\makebox(0,0){$+$}}
\put(627,235){\makebox(0,0){$+$}}
\put(640,235){\makebox(0,0){$+$}}
\put(653,235){\makebox(0,0){$+$}}
\put(666,235){\makebox(0,0){$+$}}
\put(678,235){\makebox(0,0){$+$}}
\put(691,235){\makebox(0,0){$+$}}
\put(704,235){\makebox(0,0){$+$}}
\put(716,235){\makebox(0,0){$+$}}
\put(729,235){\makebox(0,0){$+$}}
\put(742,235){\makebox(0,0){$+$}}
\put(754,235){\makebox(0,0){$+$}}
\put(767,235){\makebox(0,0){$+$}}
\put(780,235){\makebox(0,0){$+$}}
\put(792,235){\makebox(0,0){$+$}}
\put(805,235){\makebox(0,0){$+$}}
\put(818,235){\makebox(0,0){$+$}}
\put(830,235){\makebox(0,0){$+$}}
\put(843,235){\makebox(0,0){$+$}}
\put(856,235){\makebox(0,0){$+$}}
\put(868,235){\makebox(0,0){$+$}}
\put(881,235){\makebox(0,0){$+$}}
\put(894,235){\makebox(0,0){$+$}}
\put(906,235){\makebox(0,0){$+$}}
\put(919,235){\makebox(0,0){$+$}}
\put(932,235){\makebox(0,0){$+$}}
\put(944,235){\makebox(0,0){$+$}}
\put(957,235){\makebox(0,0){$+$}}
\put(970,235){\makebox(0,0){$+$}}
\put(983,235){\makebox(0,0){$+$}}
\put(995,235){\makebox(0,0){$+$}}
\put(1008,235){\makebox(0,0){$+$}}
\put(1021,235){\makebox(0,0){$+$}}
\put(1033,235){\makebox(0,0){$+$}}
\put(1046,235){\makebox(0,0){$+$}}
\put(1059,235){\makebox(0,0){$+$}}
\put(1071,235){\makebox(0,0){$+$}}
\put(1084,235){\makebox(0,0){$+$}}
\put(1097,235){\makebox(0,0){$+$}}
\put(1109,235){\makebox(0,0){$+$}}
\put(1122,235){\makebox(0,0){$+$}}
\put(1135,235){\makebox(0,0){$+$}}
\put(1147,235){\makebox(0,0){$+$}}
\put(1160,235){\makebox(0,0){$+$}}
\put(1173,235){\makebox(0,0){$+$}}
\put(1185,235){\makebox(0,0){$+$}}
\put(1198,235){\makebox(0,0){$+$}}
\put(1211,235){\makebox(0,0){$+$}}
\put(1223,235){\makebox(0,0){$+$}}
\put(1236,235){\makebox(0,0){$+$}}
\put(1249,235){\makebox(0,0){$+$}}
\put(1261,235){\makebox(0,0){$+$}}
\put(1274,235){\makebox(0,0){$+$}}
\put(1287,235){\makebox(0,0){$+$}}
\put(1300,235){\makebox(0,0){$+$}}
\put(1312,235){\makebox(0,0){$+$}}
\put(1325,235){\makebox(0,0){$+$}}
\put(1338,235){\makebox(0,0){$+$}}
\put(1350,235){\makebox(0,0){$+$}}
\put(1363,235){\makebox(0,0){$+$}}
\put(1376,235){\makebox(0,0){$+$}}
\put(1388,235){\makebox(0,0){$+$}}
\put(1401,235){\makebox(0,0){$+$}}
\put(1414,235){\makebox(0,0){$+$}}
\put(1426,235){\makebox(0,0){$+$}}
\put(1439,235){\makebox(0,0){$+$}}
\put(1349,819){\makebox(0,0){$+$}}
\put(589.0,235.0){\rule[-0.200pt]{204.765pt}{0.400pt}}
\put(171.0,131.0){\rule[-0.200pt]{0.400pt}{175.375pt}}
\put(171.0,131.0){\rule[-0.200pt]{305.461pt}{0.400pt}}
\put(1439.0,131.0){\rule[-0.200pt]{0.400pt}{175.375pt}}
\put(171.0,859.0){\rule[-0.200pt]{305.461pt}{0.400pt}}
\end{picture}
}  
\end{center}
\caption{Convergence of $\tilde\rho_T(N_0,N_1)$ to $\rho_\FR(N_0,N_1)$.}\label{fig:convergence}
\end {figure}

A lower bound for the Fisher-Rao MVN distance between MVNs was given in~\cite{SDPMVN-1990} by using a Fisher isometric embedding into the SPD cone of dimension $d+1$.
Calvo and Oller~\cite{SDPMVN-1990} showed that  mapping $N=N(\mu,\Sigma)$ by
\begin{equation}
\barN= f(N) \eqdef
  \mattwotwo{\Sigma+\mu\mu^\top}{\mu}{\mu^\top}{1}\in\calP(d+1),
\end{equation}
embeds the Gaussian manifold into a submanifold $\barcalN$ in the SPD cone $\calP(d+1)$ of codimension $1$.
Moreoever, when $\calP(d+1)$ is equipped with half the trace metric $\frac{1}{2}g_\trace$, the embedded MVN submanifold is Fisher isometric.  
However, the submanifold $\barcalN$ is not totally geodesic.
Thus by taking the Riemannian SPD distance in $\calP(d+1)$ induced by $\frac{1}{2}g_\trace$, we get a lower bound on the Fisher-Rao distance:

\begin{eqnarray*}
\rho_\CO(N_0,N_1) =  \frac{1}{\sqrt{2}} \sum_{i=1}^{d+1} \log^2 \lambda_i(\barN_0^{-\frac{1}{2}}\barN_1\barN_0^{-\frac{1}{2}}).
\end{eqnarray*}

\begin{proposition}\label{prop:LBCO}
A lower bound on the Fisher-Rao distance between $N_0=N(\mu_0,\Sigma_0)$ and $N_1=N(\mu_1,\Sigma_1)$ is
$$
\rho_\CO(N_0,N_1) =  \frac{1}{\sqrt{2}} \sum_{i=1}^{d+1} \log^2 \lambda_i(\barN_0^{-\frac{1}{2}}\barN_1\barN_0^{-\frac{1}{2}}),
$$
where $\barN_i=\mattwotwo{\Sigma_i+\mu_i\mu_i^\top}{\mu_i}{\mu_i^\top}{1}$ for $i\in\{0,1\}$, and the $\lambda_i(M)$'s denote the eigenvalues of matrix $M$.
\end{proposition}

Notice that by Nash embedding theorems, any Riemannian manifold $(M,g)$ can be embedded as a submanifold of the Euclidean manifold of higher dimension.
 
%%%
\subsubsection{A guaranteed $(1+\epsilon)$-approximation of the Fisher-Rao MVN distance}\label{sec:guar}
%%%
Using the lower and upper bounds on the Fisher-Rao distance and the closed-form solution for the Fisher-Rao geodesics with boundary value conditions, we can design a  guaranteed $(1+\epsilon)$-approximation recursive algorithm for the Fisher-Rao MVN distance as follows:

\begin{figure}
\begin{tabular}{cc}
\fbox{\includegraphics[width=0.45\textwidth]{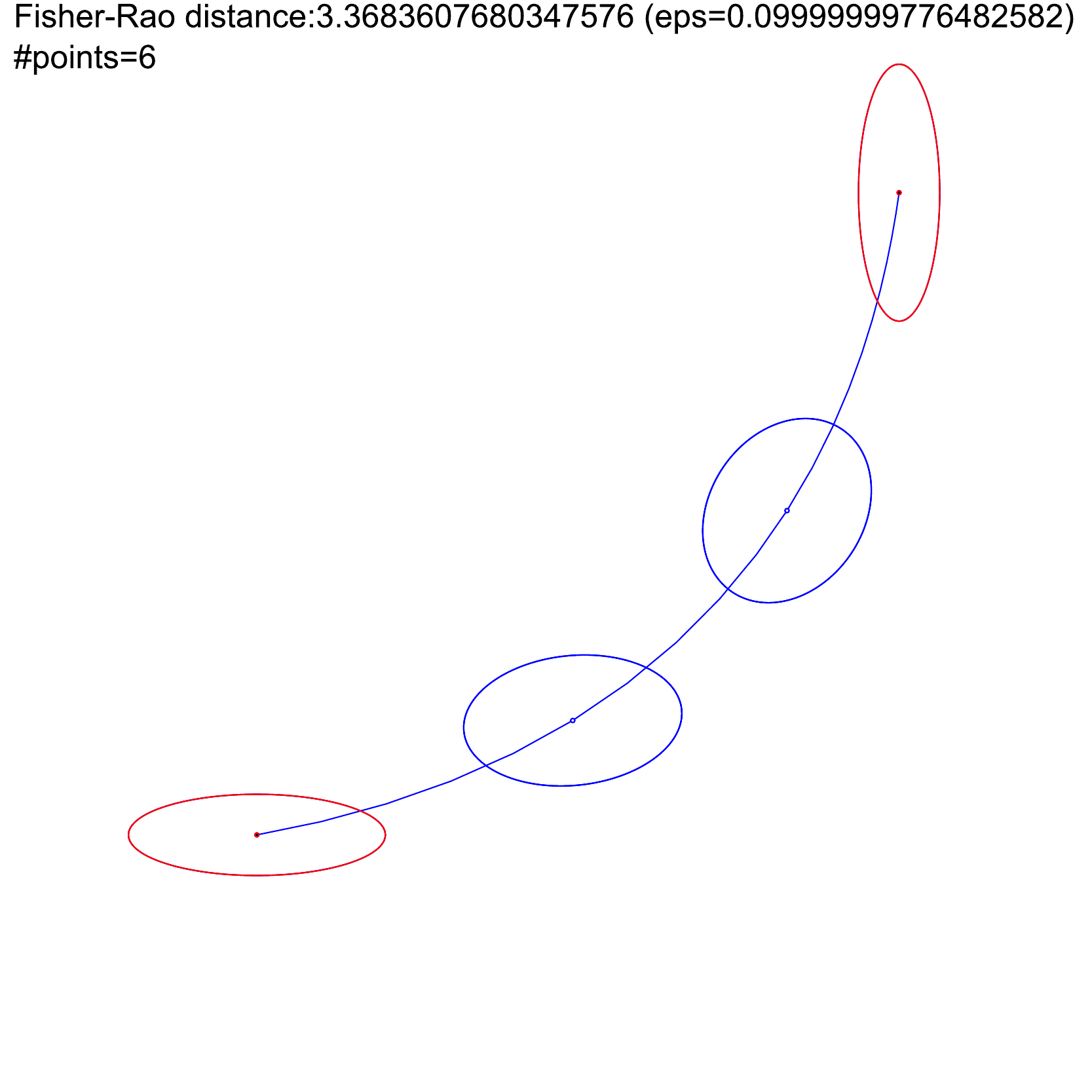}}&
\fbox{\includegraphics[width=0.45\textwidth]{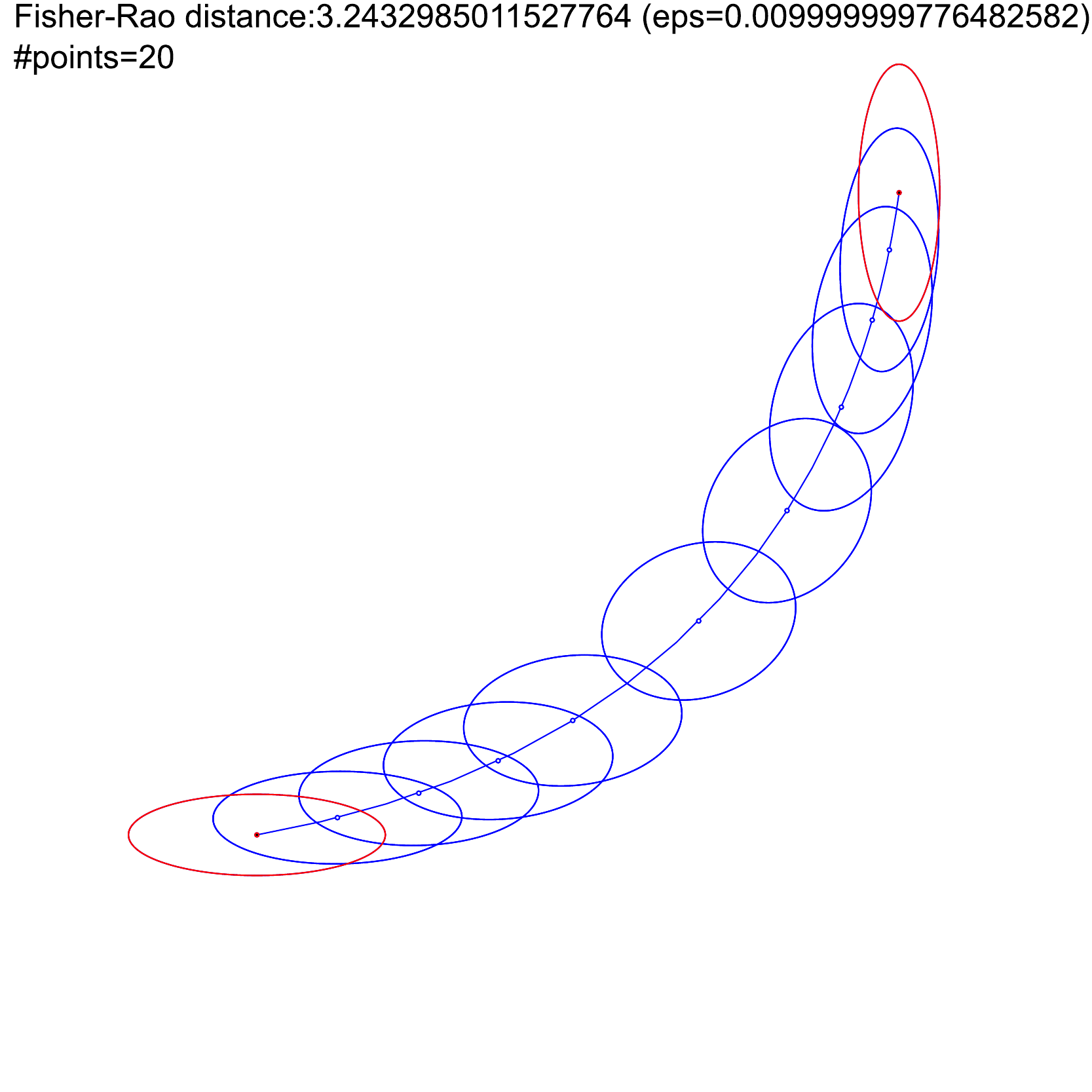}}\\
\fbox{\includegraphics[width=0.45\textwidth]{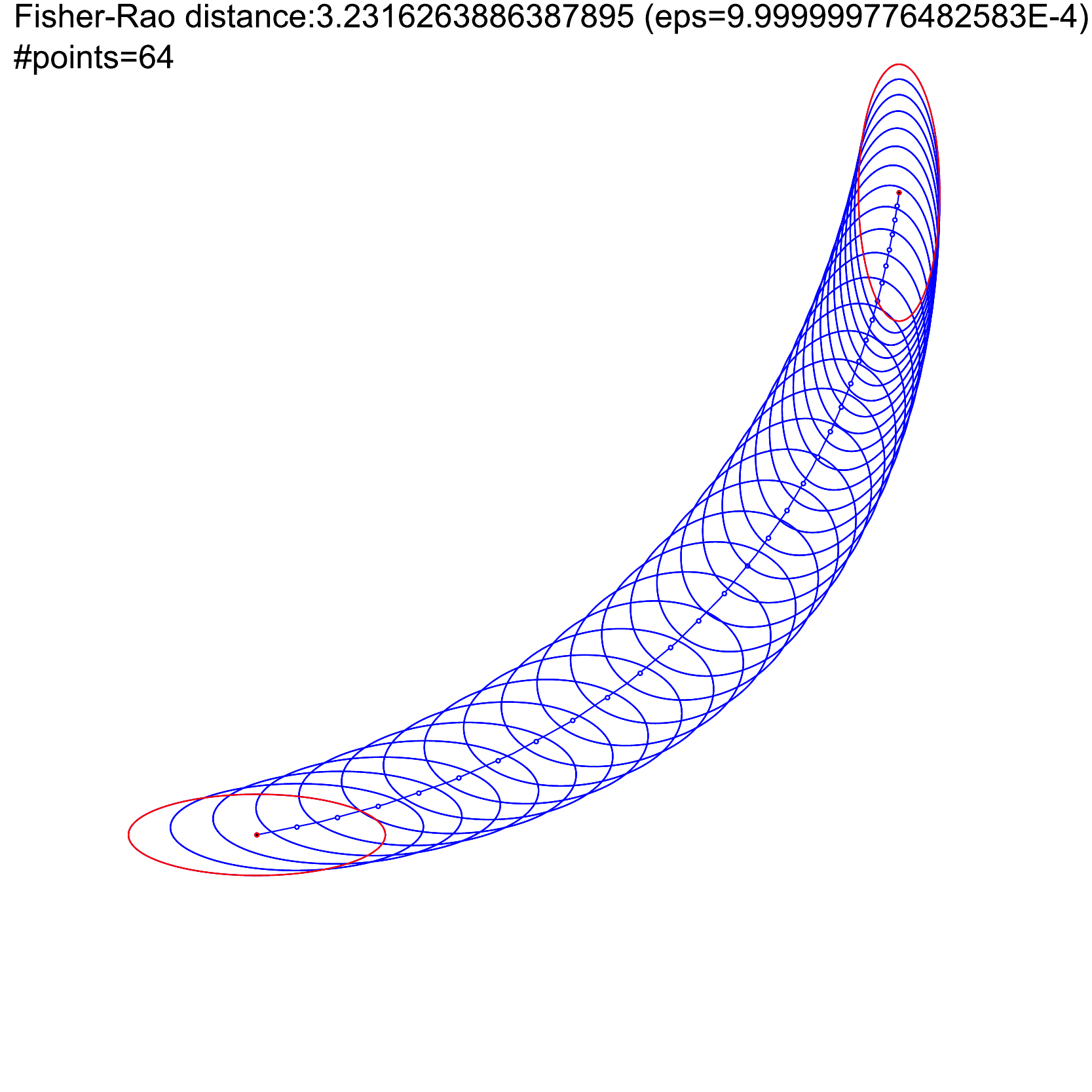}}&
\fbox{\includegraphics[width=0.45\textwidth]{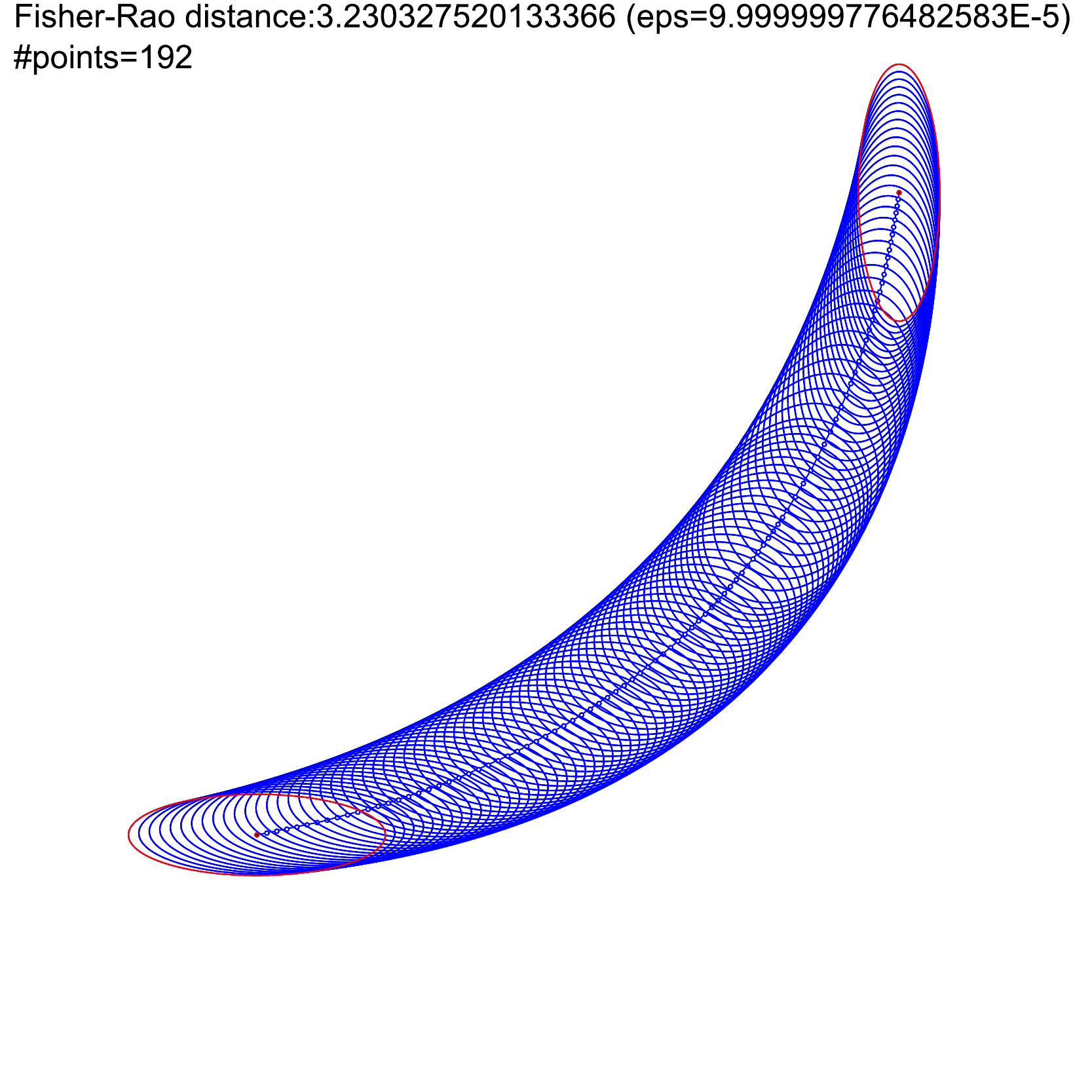}}
\end{tabular}
\caption{Some discretizations of the Fisher-Rao geodesics between two bivariate normals guaranteeing a $(1+\epsilon)$-approximation of the Fisher-Rao distance for $\epsilon\in\{0.1,0.01,0.001,0.0001\}$.\label{fig:ex1sample}}
\end{figure}

\vskip 0.5cm
\noindent\fbox{
\vbox{
\underline{Algorithm~2. $\tilde\rho_\FR(N_0,N_1)=\mathrm{ApproximateRaoMVN}(N_0,N_1,\epsilon)$}:

\begin{itemize}
	\item $l=\rho_\CO(N_0,N_1)$; /* Calvo \& Oller lower bound (Proposition~\ref{prop:LBCO}) */
	\item $u=\sqrt{D_J(N_0,N_1)}$; /* Jeffreys divergence $D_J$ (Proposition~\ref{prop:UBJ}) */
	\item if $\left(\frac{u}{l}>1+\epsilon\right)$
	\begin{itemize}
		\item $N=\mathrm{GeodesicMidpoint}(N_0,N_1)$;  /* see Algorithm~1 for $t=\frac{1}{2}$. */
		\item return $\mathrm{ApproximateRaoMVN}(N_0,N,\epsilon)+\mathrm{ApproximateRaoMVN}(N,N_1,\epsilon)$;
	\end{itemize}
	else
	return u;
\end{itemize}
}}

\begin{theorem}\label{thm:guar}
We can compute a $(1+\epsilon)$-approximation of the Fisher-Rao distance between two multivariate normal distributions for any $\epsilon>0$.
\end{theorem}
 
\begin{proof}
The proof follows from the fact that when $\frac{u}{l}\leq 1+\epsilon$ then we have 
$$
\sqrt{D_J(N_0,N_1)}\leq (1+\epsilon) \rho_\CO(N_0,N_1) \leq \leq (1+\epsilon) \rho_\FR(N_0,N_1).
$$

Since we cut recursively along the geodesics $\gamma_\FR(N_0,N_1)$ into overall $T+1$ pieces $N_{\frac{i}{T}}$, we have
$$
\tilde\rho_\FR(N_0,N_1)=\sum_{i=0}^{T-1} \tilde\rho(N_{\frac{i}{T}},_{\frac{i+1}{T}})\leq \sum_{i=0}^{T-1}  (1+\epsilon) \rho_\FR((N_{\frac{i}{T}},_{\frac{i+1}{T}}).
$$
But $\gamma_\FR(N_0,N_1)$ is totally geodesic so that $\sum_{i=0}^{T-1} \rho_\FR((N_{\frac{i}{T}},_{\frac{i+1}{T}})=\rho_\FR(N_0,N_1)$.
It follows that 
$$
\tilde\rho_\FR(N_0,N_1)\leq (1+\epsilon)\rho_\FR(N_0,N_1).
$$
\end{proof}

Instead of computing exactly the geodesic midpoint $\gamma_\FR(N_0,N_1;\frac{1}{2})$, we may use the matrix arithmetic-harmonic mean algorithm of~\cite{nakamura2001algorithms} which convergence quadratically to the geometric matrix mean of $\barN_0$ and $\barN_1$ (i.e., the geodesic midpoint). The matrix AHM only requires to compute matrix inverses for calculating the matrix harmonic means
 (see Appendix~\ref{sec:ahm}).

Figure~\ref{fig:ex1sample} shows the discretization steps of the Fisher-Rao geodesic between two bivariate normal distributions for  
 $\epsilon\in\{0.1,0.01,0.001,0.0001\}$ with the number of control points on the geodesics and the guaranteed $(1+\epsilon)$-approximation of the Fisher-Rao distance. Figure~\ref{fig:ex2sample} illustrates several examples of discretized  Fisher-Rao geodesics between bivariate normal distributions obtained for $\epsilon=0.01$.

\begin{figure}
\begin{tabular}{cc}
\fbox{\includegraphics[width=0.45\textwidth]{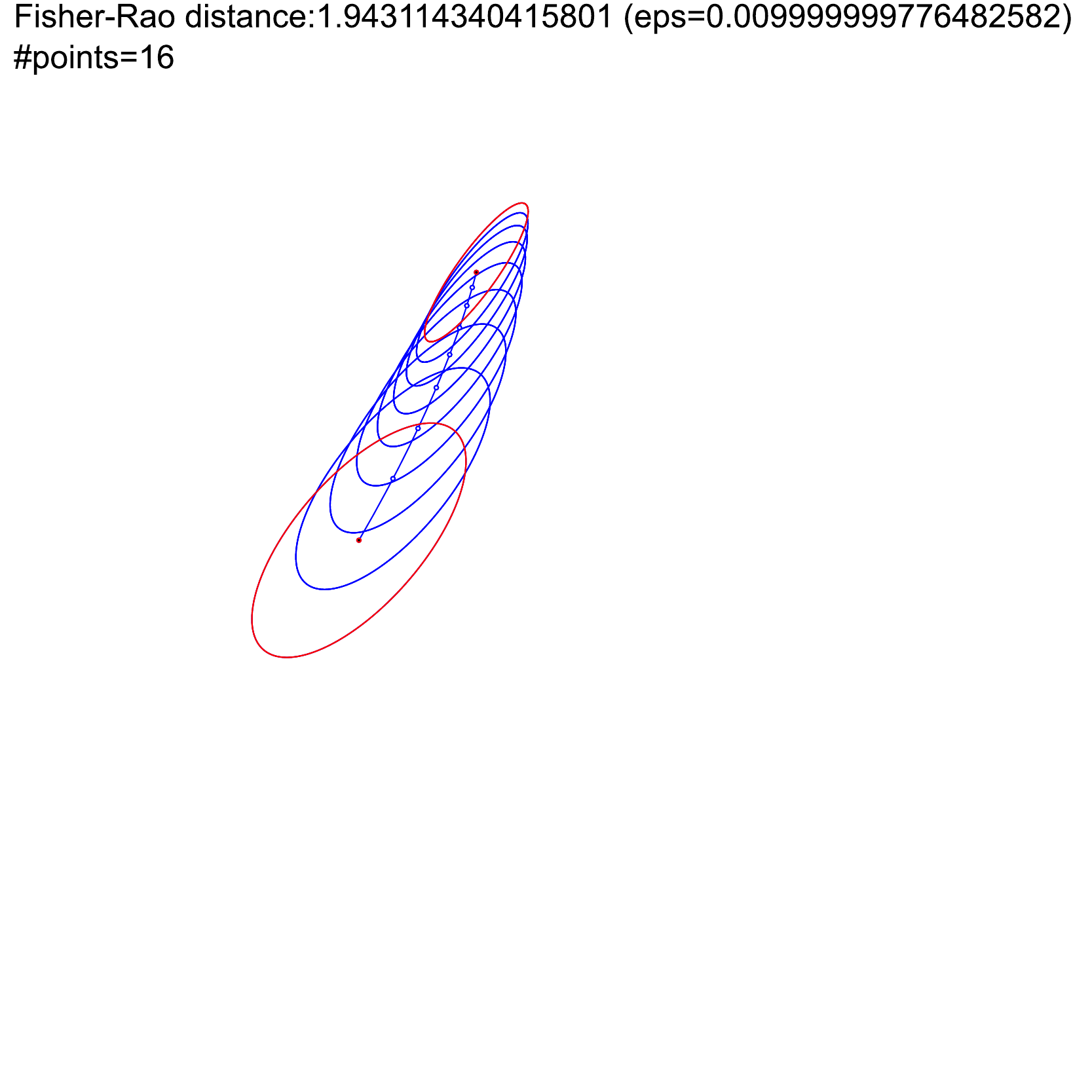}}&
\fbox{\includegraphics[width=0.45\textwidth]{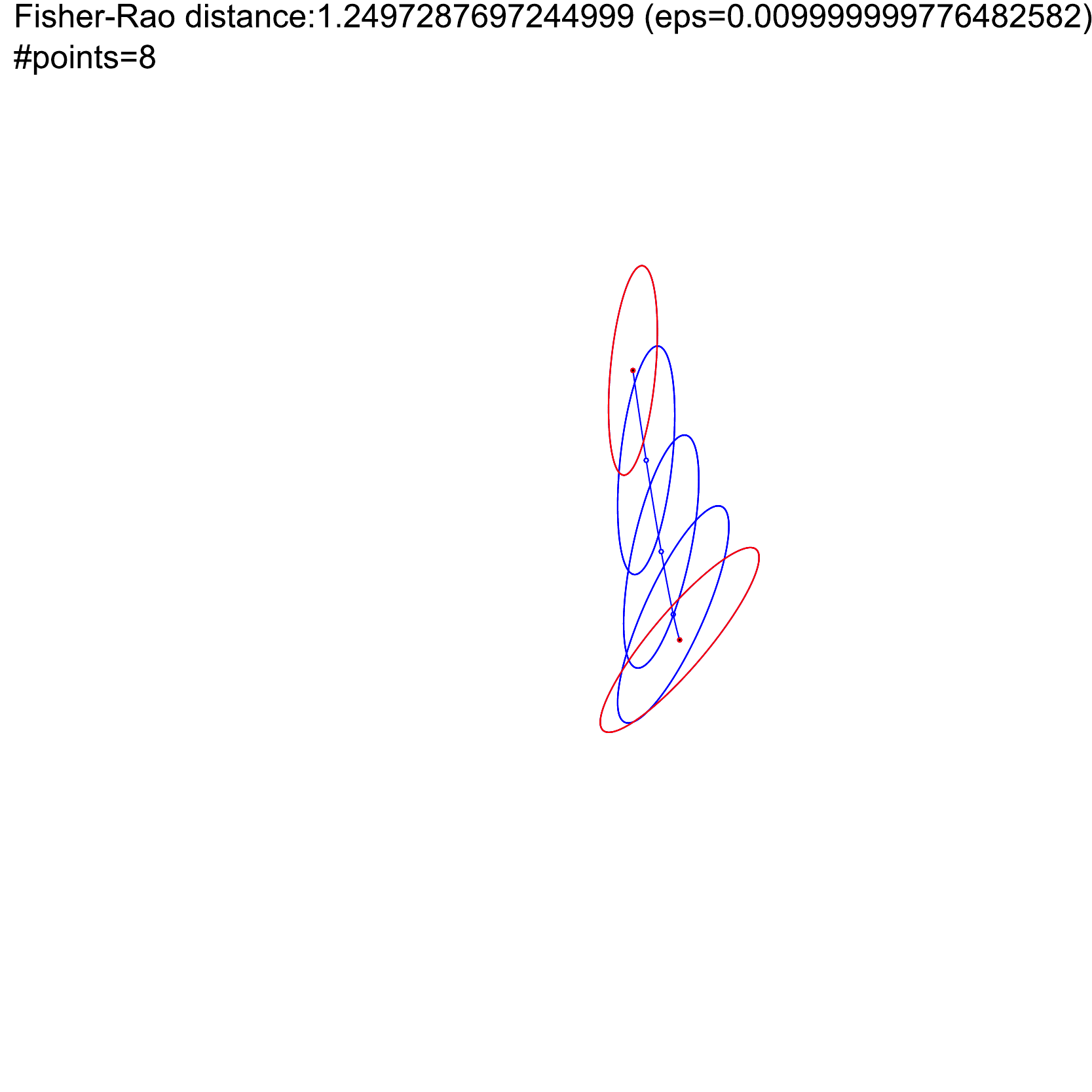}}\\
\fbox{\includegraphics[width=0.45\textwidth]{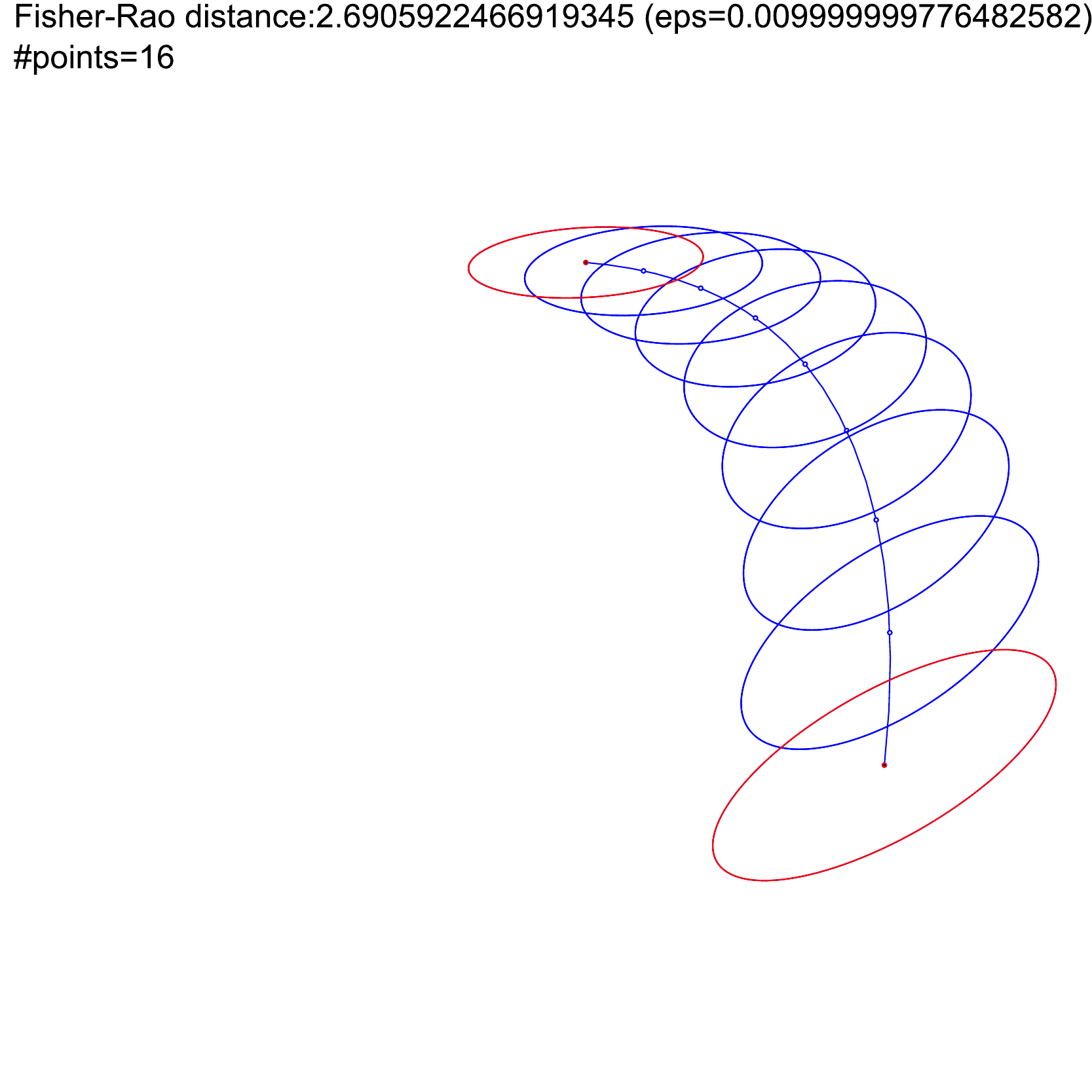}}&
\fbox{\includegraphics[width=0.45\textwidth]{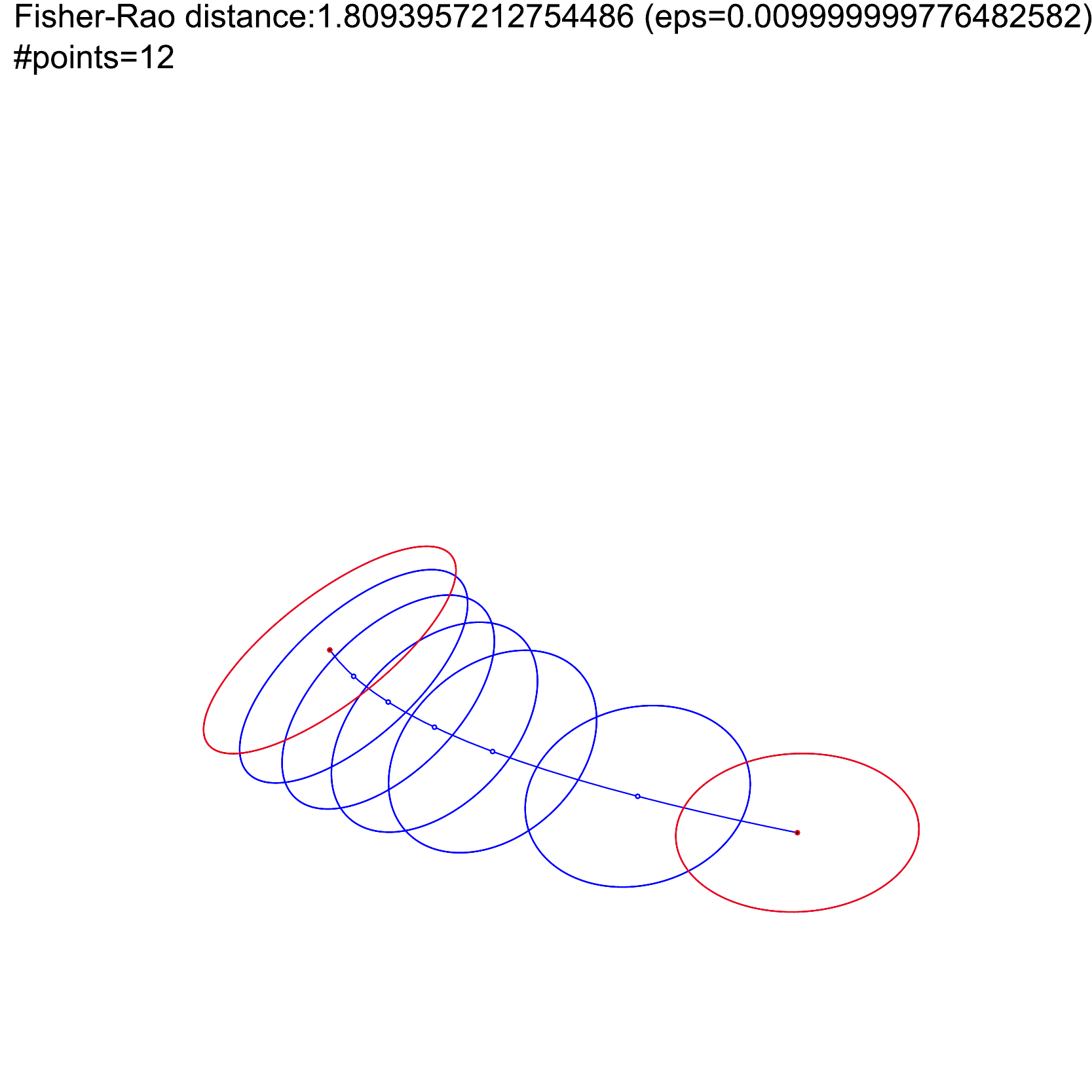}}
\end{tabular}
\caption{Some discretizations of the Fisher-Rao geodesics between two random bivariate normals guaranteeing a $(1+\epsilon)$-approximation of the Fisher-Rao distance for $\epsilon=0.01$.\label{fig:ex2sample}}
\end{figure}

%%% 
\section{Fisher-Rao clustering}\label{sec:clustering}
%%%
We shall consider two applications of the Fisher-Rao distance between MVNs using clustering:
 
The first application considers clustering weighted MVNs which is useful to simplify Gaussian Mixture Models~\cite{davis2006differential,strapasson2016clustering} (GMMs): A GMM $m(x)=\sum_{i=1}^n w_i p_{\mu_i,\Sigma_i}(x)$ with $n$ components is a weighted set of $n$ MVNs $N(\mu_i,\Sigma_i)$ and clustering this set into $k$-clusters allows one to simplify the GMM $m(x)$. 
For this task, we may use the $k$-means clustering~\cite{lloyd1982least} when centroids are available in closed-form~\cite{davis2006differential} (using the Kullback-Leibler divergence) or the $k$-medioid clustering~\cite{kaufman1990partitioning} when we choose the representative of clusters from the input otherwise (using the Fisher-Rao distance).

 The second application considers the quantization of sets of MVNs which is useful to further compress a set $\{m_1,\ldots, m_n\}$ of $n$ GMMs 
$m_i(x)=\sum_{j=1}^{n_i} w_{i,j} p_{\mu_{i,j},\Sigma_{i,j}}(x)$ with overall $N=\sum_{i=1}^n n_i$ MVNs $N(\mu_{i,j},\Sigma_{i,j})$.
We build a codebook of $k$ MVNs $N(m_i,S_i)$ by quantizing the $N$ non-weighted MVNs using the guaranteed $k$-center clustering of~\cite{gonzalez1985clustering} (also called $k$-centers clustering~\cite{dueck2007non}). Then each mixture $m_i(x)$ is quantized into a mixture 
$\tilde m_{w_i}(x)=\sum_{i=1}^{k} w_{i,j} p_{m_{i},S_{i}}(x)$. The advantage of quantization is that the original set 
$\{m_1,\ldots, m_n\}$ of GMMs is compactly represented by $n$ points in the $(k-1)$-dimensional standard simplex $\Delta_{k-1}$ encoding $\{\tilde m_1,\ldots, \tilde m_n\}$ since they share the same components. The set $\{\tilde m_w \st w\in\Delta_{k-1}\}$ form a mixture family in information geometry~\cite{IG-2016,nielsen2019monte} with a dually flat space which can be exploited algorithmically.

Notice that minimizing the objective functions of these $k$-means, $k$-medioid and $k$-center clustering objective are NP-hard when dealing with MVNs.

\subsection{Nearest neighbor queries}
In order to speed up these center-based clustering, we shall find  for a given MVN $N(\mu,\Sigma)$ (a query) its closest cluster center among $k$ MVNs $\{N(m_i,S_i)\}$ using Nearest Neighbor (NN) query search~\cite{andoni2009nearest,bhatia2010survey}.
There exist many data-structures for exact and approximate NN queries.
For example, the vantage point (VP) tree structure is well-suited in metric spaces~\cite{yianilos1993data} and has also been considered for NN queries with respect to the Kullback-Leibler divergence between MVNs~\cite{nielsen2009bregman}.
Although NN queries based on VP-trees still require linear time in the worst-case, they can also achieve logarithmic time in best cases.
At the heart of NN search using VP-trees, we are given a query ball $\Ball(p,r)$ with center $p$ and radius $r$, and we need to find potential intersections with balls $\Ball(v,r_v)$ stored at nodes $v$ of the VP tree.
Thus when using the (fine approximation $\tilde\rho_T$) Fisher-Rao metric distance, we need to answer predicates of whether two Fisher-Rao balls $\Ball_\FR(N,r)$ and $\Ball_\FR(N',r')$ intersect or not:
This can be done by determining the sign of
$\rho_\FR(N,N')-(r+r')$.
When positive the balls do not intersect and when negative the balls intersect.
Since we handle some approximation errors by using $\tilde\rho_T$ instead of $\rho_\FR$, 
but since $\tilde\rho_T\geq \rho_\FR$ we 
need to explore both branches of a VP-tree if the balls $\Ball_\FR(N,r)$ and $\Ball_\FR(N',r')$ stored at the two siblings of a node $v$ are such that
$\tilde\rho_T(N,N')\leq r+r'$.

%%%
\subsection{\protect $k$-centers clustering and miniball}
%%%
For the quantization tasks, the $k$-center clustering heuristic of Gonzalez~\cite{gonzalez1985clustering} guarantees to find a good $k$-center clustering in metric spaces with an approximation factor upper bounded by $2$.
We can further refine the cluster representative of each cluster by computing approximations of the {\em smallest enclosing Fisher-Rao balls} (miniballs) of clusters.

A simple Riemannian approximation technique has been reported for approximating the smallest enclosing ball of $n$ points $\{p_1,\ldots,p_n\}$ on a Riemannian manifold $(M,g)$ with geodesic distance $\rho_g(p,p')$ and geodesics $\gamma_g(p,p';t)$ in~\cite{RieMinimax-2013}:

\noindent\underline{Miniball($\{p_1,\ldots,p_n\},\rho_g,T$):}
\begin{itemize}
	\item Let $c_1\leftarrow p_1$
	\item For $t=1$ to $T$
	\begin{itemize}
		\item Compute the index of the point which is farthest to current circumcenter $c_t$:
		$$
		f_t=\arg\max_{i\in\{1,\ldots,n\}} \rho_g(c_t,p_i)
		$$
		\item Update the circumcenter by walking along the geodesic linking $c_t$ to $p_{f_t}$:
		$$
		c_{t+1}=\gamma_g\left(c_t,p_{f_t};\frac{1}{t+1}\right)
		$$
		Recall that geodesics are parameterized by normalized arc length so that $\rho_g(c_t,c_{t+1})=\frac{1}{t+1}\rho_g(c_t,p_{f_t})$.
	\end{itemize}
	\item Return $c_T$
\end{itemize}

Conditions of convergence are analyzed in~\cite{RieMinimax-2013}:
For example, it always converge for Cartan-Hadamard manifolds (complete simply connected NPC manifolds like the SPD cone).

The Fisher-Rao distance $\rho_\FR(N_\Sigma(\mu_0),N_\Sigma(\mu_1))$ between two MVNs with same covariance matrix $\Sigma$ is
$$
\rho_\FR(N_\Sigma(\mu_0),N_\Sigma(\mu_1))=\sqrt{2}\, \arccosh\left(1+\frac{1}{4}\Delta_\Sigma^2(\mu_0,\mu_1)\right),
$$
where $\Delta_\Sigma(\mu_0,\mu_1)=\sqrt{(\mu_0-\mu_1)^\top\Sigma^{-1}(\mu_20-\mu_1)}$ is the Mahalanobis distance.
Therefore when all MVNs belong to the non-totally flat submanifold $\calN_\Sigma=\{N(\mu,\Sigma) \st \mu\in\bbR^d\}$, the smallest enclosing ball amounts to an Euclidean smallest enclosing ball~\cite{welzl2005smallest} since in that case $\rho_\FR$ is an increasing function of the Mahalanobis distance.

%A Java code {\tt jFisherRaoMVN.java} is available in the additional material zip for reproducible research. 

Since the computations of $\tilde\rho_T$ approximating $\rho_\FR$ is costly, the following section shall consider a new fast metric distance on $\calN$ which further relates to the Fisher-Rao distance.

\section{Pullback Hilbert cone distance}\label{sec:Hilbert}
Let us define dissimilarities and paths on $\calN(d)$ from dissimilarities and geodesics on $\calP(d+1)=\calN_0(d)$ by considering the following family of {\em diffeomorphic embeddings} $f_a:\calN(d)\rightarrow\calP(d+1)$ for $a\in\bbR_{>0}$ proposed in~\cite{SDPMVN-1990}:
\begin{equation}
f_{a}(N(\mu,\Sigma)) \eqdef
  \mattwotwo{\Sigma+a\mu\mu^\top}{a\mu}{a\mu^\top}{a}\in\calP(d+1).
\end{equation}
Let $\barcalN_a(d)=\{f_a(N) \st N\in\calN(d)\}\subset\calP(d+1)$ denote the embedded Gaussian submanifold in $\calP(d+1)$ of codimension $1$ that is obtained by pushing the normals to SPD matrices (Figure~\ref{fig:pushpull}).
We let $f_a^\inv:\barcalN_a(d)\rightarrow\calN(d)$ denote the functional inverse  so that $f_a\circ f_a^\inv=id_\calN$ is  the identity function $\id_\calN: \calN\rightarrow\calN$. Function $f_a^\inv$ pulls back the SPD matrices to the normal distributions (Figure~\ref{fig:pushpull}).
The notation $\inv$ in $f_a^\inv$ is chosen to avoid confusion with the matrix inverse $f_{a}(N(\mu,\Sigma))^{-1}$:
$$
f_{a}(N(\mu,\Sigma))^{-1}=\mattwotwo{\Sigma^{-1}}{-\Sigma^{-1}\mu}{-\mu^\top\Sigma^{-1}}{\mu^\top\Sigma^{-1}\mu+\frac{1}{a}}.
$$

\begin{figure}
\centering
 \includegraphics[width=0.750\columnwidth]{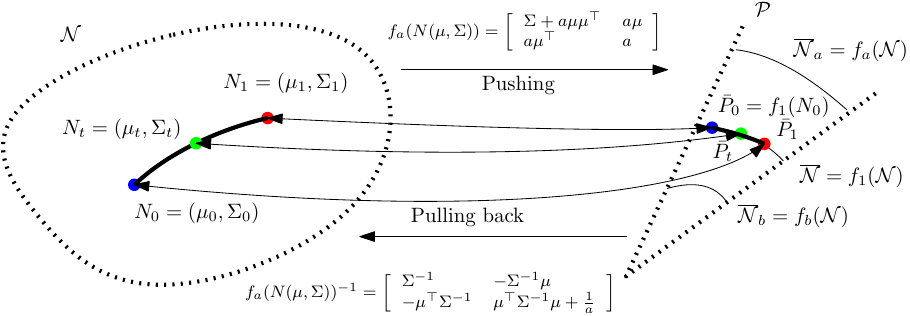}

\caption{Pushing the $d$-dimensional normals $\calN$ to a submanifold $\barcalN$ of the $(d+1)\times(d+1)$-dimensional SPD cone via a diffeomorphic embedding $f$, and pulling back a geodesic/curve on the SPD cone to the normal manifold with $f^{-1}$.}\label{fig:pushpull}
\end{figure}

The open SPD cone $\calP(d+1)$ can thus be foliated by the family of submanifolds $\barcalN_a$~\cite{SDPMVN-1990}:
$\calP(d+1) = \left\{a \times \barcalN_a \st a\in\bbR_{>0}\right\}$.
We let $f=f_1$ and $f^\inv=f_1^\inv$, and $\barcalN=\barcalN_1$.

 Calvo and Oller~\cite{SDPMVN-1990} proved that $(\calN(d),g_\Fisher)$ is isometrically embedded into $(\calP(d+1),\frac{1}{2}g_\trace)$ but that $\barcalN$ is not totally geodesic.
Thus we have 
\begin{eqnarray*}
\rho_\CO(N_0,N_1) &=& \rho_\FR(N(0,f(N_0)),N(0,f(N_1))),\\
&=& \rho_\calP(\barN_0,\barN_1)\geq \rho_\FR(N_0,N_1),
\end{eqnarray*}
where $\barN_i=f(N_i)$. See Eq.~\ref{eq:FRsamemu}.
It follows that we get a series of lower bound for $\rho_\FR(N_0,N_1)$:
$$
\rho_{\CO,T}(N_0,N_1) = \sum_{i=0}^{T-1} \rho_\calP\left(N_{\frac{i}{T}},N_{\frac{i+1}{T}}\right),
$$
such that for all $T$, $\tilde\rho_T\geq \rho_\FR \geq \rho_{\CO,T}$.

We can also approximate the smallest enclosing Fisher-Rao ball of $\{N(\mu_i,\Sigma_i)\}$ on $\calN(d)$ by embedding the normals into $\barcalN$ as $\{\barP_i=f(N(\mu_i,\Sigma_i))\}$. We then apply the above iterative smallest enclosing ball approximation Miniball~\cite{RieMinimax-2013} to get $\tilde C_{T}\in\calP(d+1)$ after $T$ iterations. Then we project orthogonally with respect to the trace metric $\tilde C_T$ onto $\barcalN$ as $\bar C_T=\proj_{\barN}(\tilde C_T)$ and maps back to the Gaussian manifold using $f^\inv$ to get the approximate normal circumcenter.

The following proposition describes the orthogonal projection operation $\barP_\perp=\proj_{\barN}(P)$ 
of $P=[P_{i,j}]\in\calP(d+1)$ onto $\barN$ based on the analysis reported in the Appendix of~\cite{SDPMVN-1990} (page 239):

\begin{proposition}\label{prop:proj}
Let $\beta=P_{d+1,d+1}$ and write  $P=\mattwotwo{\Sigma+\beta\mu\mu^\top}{\beta\mu}{\beta\mu^\top}{\beta}$.
Then the orthogonal projection at $P\in\calP$ onto $\barN$ is: 
\begin{equation}
\barP_\perp:=\proj_{\barN}(P)=\mattwotwo{\Sigma+\mu\mu^\top}{\mu^\top}{\mu}{1},
\end{equation}
and the SPD trace distance between $P$ and $\barP_\perp$  is 
\begin{equation}
\rho_\calP(P,\barP_\perp)=|\log\beta|.
\end{equation}
\end{proposition}

Consider pulling back SPD cone dissimilarities and geodesics of $\calP(d+1)$ onto $\calN(d)$ as follows:

\begin{definition}[Pullback dissimilarities]\label{def:SPDdis}
A dissimilarity $D(N_0,N_1)$ (not necessarily be a metric distance nor a smooth divergence)
on $\calN(d)$ (with $N_0\eqdef N(\mu_0,\Sigma_0)$ and $N_1 \eqdef N(\mu_1,\Sigma_1)$) 
can be obtained from any dissimilarity $D_\SPD(\cdot,\cdot)$ on the SPD cone by pulling back the SPD matrix cone dissimilarity using 
 $f$:
\begin{equation}
D(N_0,N_1)\equaldef D_\SPD(f(N_0),f(N_1)).
\end{equation}
\end{definition}

Similarly, we pullback cone geodesics onto $\calN$:
 \begin{definition}[Pullback curves]\label{def:SPDcurve}
A path $c_{\gamma}(N_0,N_1;t)$ joining $N_0=c_{\gamma}(N_0,N_1;0)$ and $N_1=c_{\gamma}(N_0,N_1;1)$ can be defined by the pullback of any geodesic $\gamma(f(N_0),f(N_1);t)$ on the SPD cone:
\begin{equation}
c_{\gamma}(N_0,N_1;t)\equaldef f^\inv(\gamma(f(N_0),f(N_1);t)).
\end{equation}
\end{definition}

Hence, we can leverage the rich literature on dissimilarities and geodesics on the SPD cone (e.g.,~\cite{hero2001alpha,chebbi2012means,sra2016positive,baggio2018conal,CTP-2021}).
Note that the Riemannian SPD trace metric geodesic is also the geodesic for Finslerian distances
$\rho_h(P_0,P_1)\equaldef \left\| 
\Log \left(P_0^{-\frac{1}{2}}\, P_1\, P_0^{-\frac{1}{2}} \right)
  \right\|_h$ where $h$ is a totally symmetric gauge function (i.e., $h(x_1,\ldots, x_n)=h(\sigma(x_1,\ldots, x_n))$ for any permutation $\sigma$)) and $\|P\|_h \equaldef h(\lambda_1(P),\ldots,\lambda_d(P))$.
When $h(x)=h_p(x)=\|x\|_p=\left(\sum_{i=1}^d x_i^p\right)^{\frac{1}{p}}$ is the $p$-norm for $1\leq p<\infty$, we get the Schatten matrix $p$-norms~\cite{SPD-Bhatia-2009}.

\begin{figure}
\centering
 \fbox{\includegraphics[width=0.50\columnwidth]{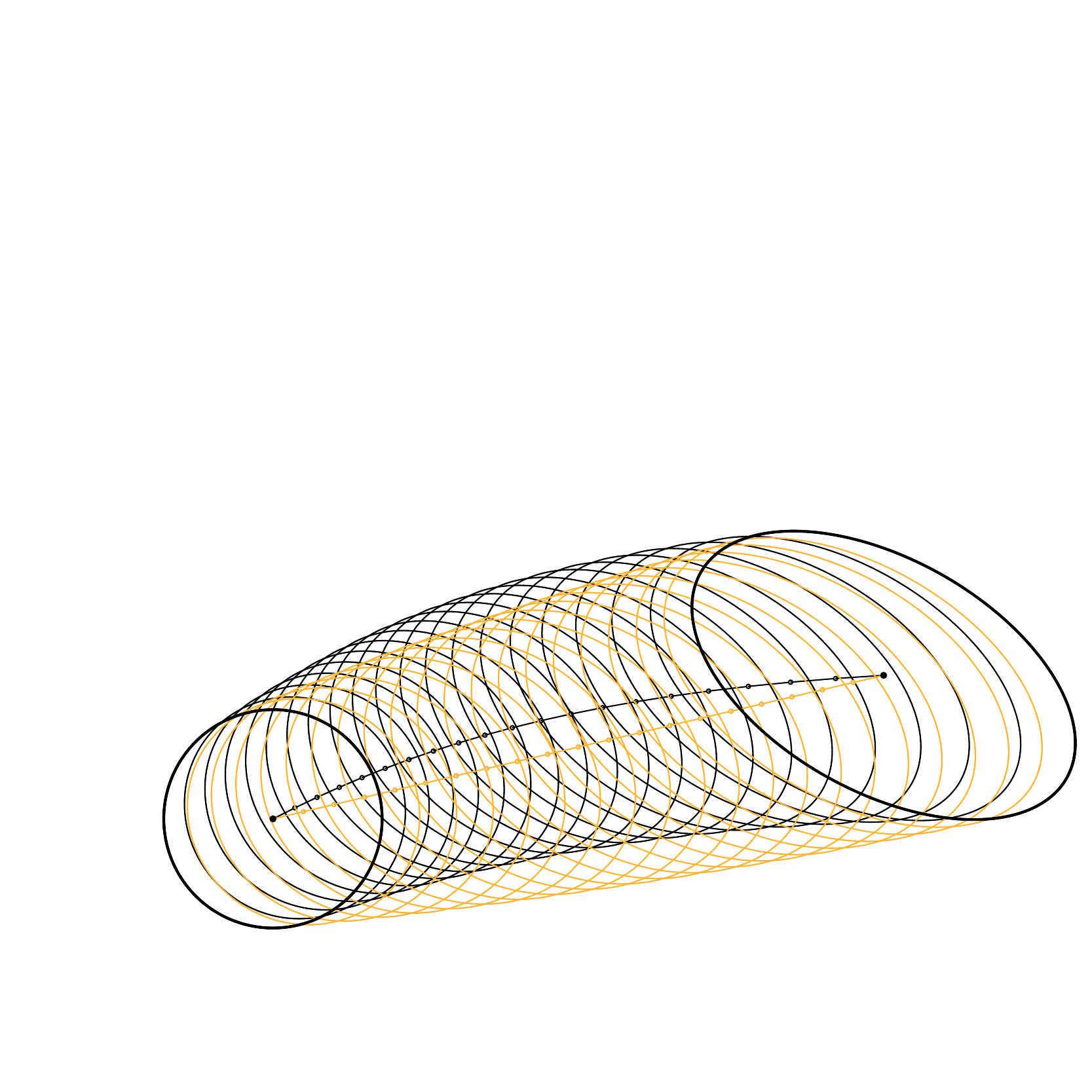}}
%\fbox{\includegraphics[width=0.55\columnwidth]{Fig-Ex-Hilbert.pdf}}
 
%
\caption{Comparing the pullback Hilbert geodesic (orange, coinciding with the mixture geodesic) with the exact Fisher-Rao geodesic displayed in black.}\label{fig:COHilbert}
\end{figure}

The {\em Hilbert projective distance} also called {\em Birkhoff projective distance}~\cite{Hilbert-1895,Birkhoff-1957,CTP-2021} on the SPD cone $\Sym_+(d,\bbR)$ is defined by
\begin{eqnarray*}
\rho_\Hilbert(P_0,P_1)
&=& \log\left(\frac{\lambda_{\mmax}(P_0^{-\frac{1}{2}}P_1 P_0^{-\frac{1}{2}})}{\lambda_{\mmin}(P_0^{-\frac{1}{2}}P_1 P_0^{-\frac{1}{2}})}\right),\\
 &=& \log\left(\frac{\lambda_{\mmax}(P_0^{-1}P_1)}{\lambda_{\mmin}(P_0^{-1}P_1)}\right).
\end{eqnarray*}
It is a {\em projective distance} (or quasi-metric distance) because it is symmetric and satisfies the triangular inequality but we have $\rho_\Hilbert(P_0,P_1)=0$ if and only if $P_0=\lambda P_1$ for some $\lambda>0$.
It was proven in~\cite{UniversalHilbertProjectiveMetric-1982} (Theorem 4.1) that
 any projective distance with strict contraction ratio under linear mapping is necessarily a scalar function of the Hilbert's projective metric, and that the Hilbert projective distance has the lowest  possible contraction ratio.

However, the pullback Hilbert distance on $\calN$, $\rho_{\Hilbert}(N_0,N_1):=\rho_\Hilbert(f(N_0),f(N_1))$,
 is a proper metric distance on $\barcalN$
since $f(N_0)=f(N_1)$ if and only if $\lambda=1$ because the array  element at last row and last column 
$[f(N_0)]_{d+1,d+1}=[f(N_1)]_{d+1,d+1}=1$ is identical. Thus $f(N_0)=\lambda f(N_1)$ for $\lambda=1$.
The pullback Hilbert cone distance only requires to calculate the {\em extreme eigenvalues} of the matrix product $f(N_0)^{-1}f(N_1)$. Thus we can bypass a costly SVD and compute approximately these extreme eigenvalues using the power method~\cite{trevisan2017lecture} (Appendix~\ref{sec:powermethod}).

The geodesic in the Hilbert SPD cone are straight lines~\cite{nussbaum1994finsler} parameterized as follows:
$$
\gamma_\Hilbert(P_0,P_1;t):=\left(\frac{\beta\alpha^t-\alpha\beta^t}{\beta-\alpha}\right)P_0+
\left(\frac{\beta^t-\alpha^t}{\beta-\alpha}\right)P_1,
$$
where $\alpha=\lambda_\mmin(P_1^{-1}P_0)$ and $\beta=\lambda_\mmax(P_1^{-1}P_0)$.
Figure~\ref{fig:COHilbert} compares the pullback Hilbert geodesic curve with the Fisher-Rao geodesic.

A pregeodesic is a geodesic which may be arbitrarily reparameterized by another parameter $u=r(t)$ for some smooth function $r$.
That is, a pregeodesic is not necessarily parameterized by arc length.
Let us notice that the weighted arithmetic mean $\LERP(P_0,P_1;u)=(1-u)P_0+uP_1$ is a pregeodesic of $\gamma_\Hilbert(P_0,P_1;t)$.
Although Hilbert SPD space is not a Riemannian space, it enjoys non-positive curvature properties according to various definitions of curvatures~\cite{alabdulsada2019non,karlsson2000hilbert}.

We can adapt the approximation of the minimum enclosing Hilbert ball by replacing $\rho_\FR$ by $\rho_\Hilbert$ and cutting  metric geodesic $\gamma_\Hilbert$ instead of geodesics $\gamma_\FR^\calN$ (see Figure~\ref{fig:Miniball}).

\begin{figure}
\centering
%\fbox{\includegraphics[width=0.55\columnwidth]{HilbertMiniMax-2.pdf}}

\fbox{\includegraphics[width=0.5\columnwidth]{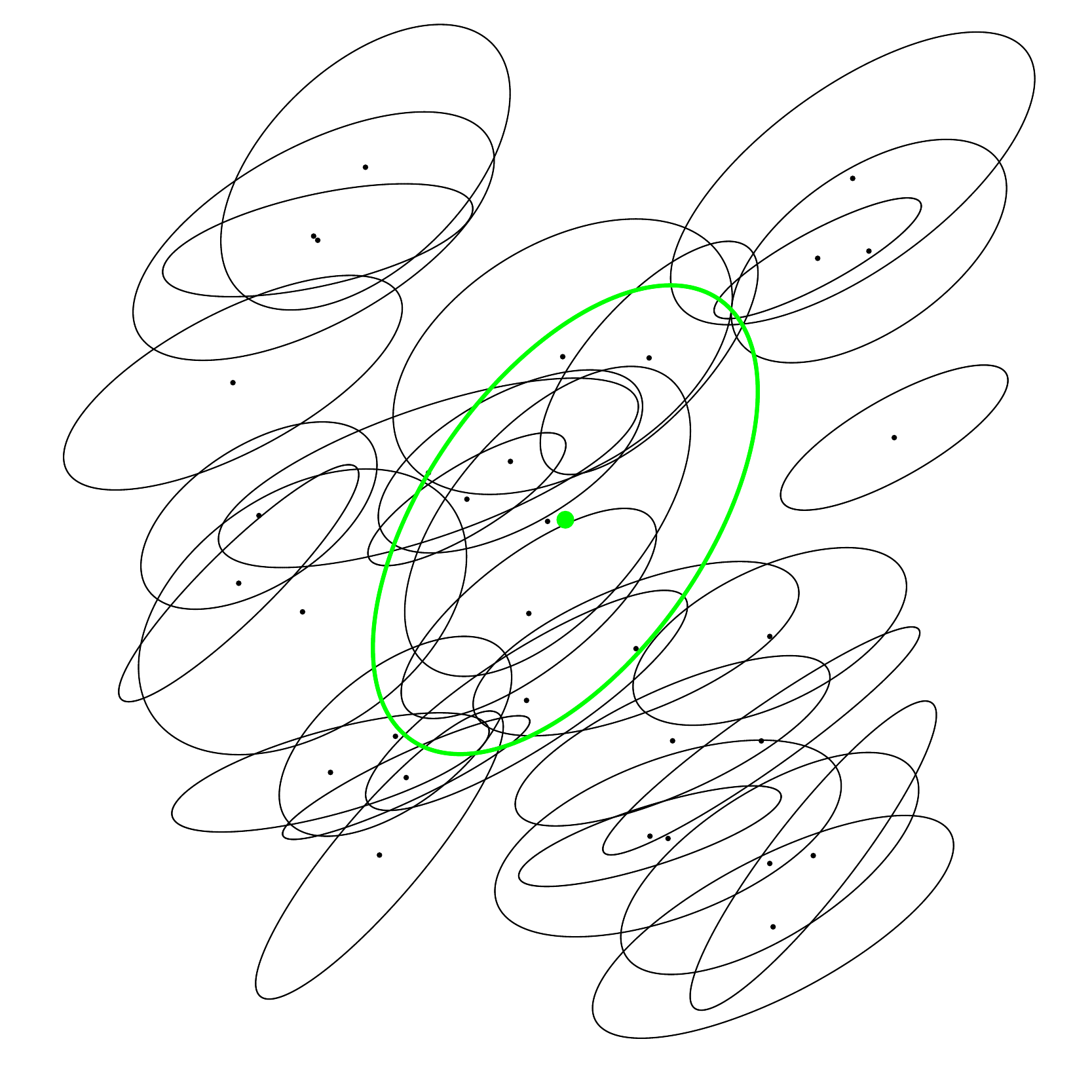}}
\caption{Approximating the Hilbert smallest enclosing ball of a set of bivariate normal distributions.
 The approximated minimax center is shown in green.}\label{fig:Miniball}
\end{figure}

First, the diffeomorphic embedding $f$ exhibits several interesting features:
\begin{proposition}\label{prop:sameJD}
The Jeffreys divergence   between 
$p_{\mu_1,\Sigma_1}$ and $p_{\mu_2,\Sigma_2}$ amounts to the Jeffreys divergence
between $q_{\bar P_1}=p_{0,f(\mu_1,\Sigma_1)}$ and $q_{\bar P_2}=p_{0,f(\mu_2,\Sigma_2)}$ where $\bar P_i=f(\mu_i,\Sigma_i)$:
$D_J(p_{\mu_1,\Sigma_1},p_{\mu_2,\Sigma_2})=D_J(q_{\bar P_1},q_{\bar P_2})$.
\end{proposition}

\begin{proof}
Since $D_J(p,q)=D_\KL(p,q)+D_\KL(q,p)$, we shall prove that $D_\KL(p_{\mu_1,\Sigma_1},p_{\mu_2,\Sigma_2})=D_\KL(q_{\bar P_1},q_{\bar P_2})$.
The KLD between two centered $(d+1)$-variate normals $q_{P_1}=p_{0,P_1}$ and $q_{P_2}=p_{0,P_2}$ is
$$
D_\KL(q_{P_1},q_{P_2})
=\frac{1}{2}\left(
\tr(P_2^{-1}P_1)-d-1+\log\frac{|P_2|}{|P_1|}
\right).$$
This divergence can be interpreted as the matrix version of the Itakura-Saito divergence~\cite{davis2006differential}.
It is a matrix spectral distance since we can write
$D_\KL(q_{P_1},q_{P_2})=(h_\KL\circ\lambda^\sp)(\Sigma_2^{-1}\Sigma_1)$, where $\lambda^\sp(S)=(\lambda_1(S),\ldots,\lambda_d(S))$
 and $h_\KL(u_1,\ldots,u_d)=\frac{1}{2}\left(u_i-1-\log u_i\right)$ (a gauge function).
Similarly, the Jeffreys divergence between two centered MVNs is a matrix spectral distance with gauge function $h_J(u)=\sum_{i=1}^d \left(\sqrt{u_i}-\frac{1}{\sqrt{u_i}}\right)^2$.

The SPD cone equipped with $\frac{1}{2}$ of the trace metric can be interpreted as Fisher-Rao centered normal manifolds (isometry): 
$\forall\mu, (\calN_\mu,g^\Fisher_{\calN_\mu})\cong(\calP,\frac{1}{2}g^\trace)$.

Since the determinant of a block matrix is
$
\det\left(\mattwotwo{A}{B}{C}{D}\right)=\det\left(A-BD^{-1}C\right),
$
 we get with $D=1$: $\det(f(\mu,\Sigma)) = \det(\Sigma+\mu\mu^\top-\mu\mu^\top)=\det(\Sigma)$.

Let $\bar P_1=f(\mu_1,\Sigma_1)$ and 
$\bar P_2=f(\mu_2,\Sigma_2)$.
Checking $D_\KL[p_{\mu_1,\Sigma_1}:p_{\mu_2,\Sigma_2}]=D_\KL[q_{\bar P_1}:q_{\bar P_2}]$ where $q_{\barP}=p_{0,\barP}$ amounts to verify that
$
\tr(\bar P_2^{-1}\bar P_1)=1+\tr(\Sigma_2^{-1}\Sigma_1+\Delta_\mu^\top\Sigma_2^{-1}\Delta_\mu)$.
Indeed, using the inverse matrix  
$$f(\mu,\Sigma)^{-1}=
\mattwotwo{\Sigma^{-1}}{-\Sigma^{-1}\mu}{-\mu^\top\Sigma^{-1}}{1+\mu^\top \Sigma^{-1}\mu}
,$$
we have 
$\tr(\bar P_2^{-1}\bar P_1)=\tr\left(
\mattwotwo{\Sigma^{-1}_2}{-\Sigma^{-1}_2\mu_2}{-\mu_2^\top\Sigma_2^{-1}}{1+\mu_2^\top \Sigma^{-1}_2\mu_2}\ \mattwotwo{\Sigma_1+\mu_1\mu_1^\top}{\mu_1}{\mu_1^\top}{1}
\right)
= 1+\tr(\Sigma_2^{-1}\Sigma_1+\Delta_\mu^\top\Sigma_2^{-1}\Delta_\mu)$.
Thus even if the dimension of the sample spaces of $p_{\mu,\Sigma}$ and $q_{\barP=f(\mu,\Sigma)}$ differs by one, we get the same KLD and Jeffreys divergence by Calvo and Oller's isometric mapping $f$.
\end{proof}

Second, the mixture geodesics are preserved by the embedding $f$:
\begin{proposition}\label{prop:geo}
The mixture geodesics are preserved by the embedding $f$: 
$$
f(\gamma_m^\calN(N_0,N_1;t))=\gamma^\calP_m(f(N_0),f(N_1);t).
$$
\end{proposition}
We check that $f(\LERP(N_0,N_1;t))=\LERP(\barP_0,\barP_1;t)$.
Thus the pullback of the Hilbert cone geodesics are thus coinciding with the mixture geodesics on $\calN$.

Therefore all algorithms on $\calN$ which only require $m$-geodesics or $m$-projections~\cite{IG-2016} by minimizing the right-hand side of the KLD can be implemented by algorithms on $\calP$ by using the $f$-embedding. 
On $\calP$, the minimizing problems amounts to a logdet minimization problem well-studied in  the both optimization community  and information geometry community information projections~\cite{tsuda2003algorithm}.

However, the exponential geodesics are preserved only for submanifolds $\calN_\mu$ of $\calN$ with fixed mean $\mu$.
Thus $\barN_\mu$ preserve both mixture and exponential geodesics: 
The submanifolds $\barN_\mu$ are said to be {\em doubly auto-parallel}~\cite{ohara2019doubly}.

Instead of considering the minimax center (i.e., circumcenter), we can also compute iteratively the Riemannian SPD Fr\'echet mean 
of a finite set of $n$ SPD matrices $S_1,\ldots, S_n$:\\

\underline{{\sc RieStoCentroid}}:\\
First, we let $C_1=S_{f_1}$ where $f_1$ is chosen uniformly randomly in $\{1,\ldots, n\}$. Then we iteratively update $C_i=C_{i-1}\#_{\frac{1}{i}} S_{f_i}$  for $i>1$, where $f_i$ is chosen uniformly randomly in $\{1,\ldots, n\}$.
Convergence in probability is reported for non-positive curvature spaces in~\cite{cheng2016recursive}, and more generally for length spaces in~\cite{sturm2003probability}. 
The SPD cone is NPC but not the normal manifold which has some positive sectional curvatures~\cite{Skovgaard-1984}.
To approximate the Riemannian Fisher normal centroid of n MVNs $N(\mu_1,\Sigma_1),\ldots, N(\mu_n,\Sigma_n)$, we lift them onto the higher-dimensional SPD cone: $\bar N_i=f(N(\mu_i,\Sigma_i))$. We then apply the {\sc RieStoCentroid} algorithm for $T$ iterations, and pullback $C_T$: $\tilde N=f^{-1}(C_T)$.
See Figure~\ref{fig:FRcentroid} for some illustrations.

\begin{figure}
\centering

\begin{tabular}{cc}
\fbox{\includegraphics[width=0.45\columnwidth]{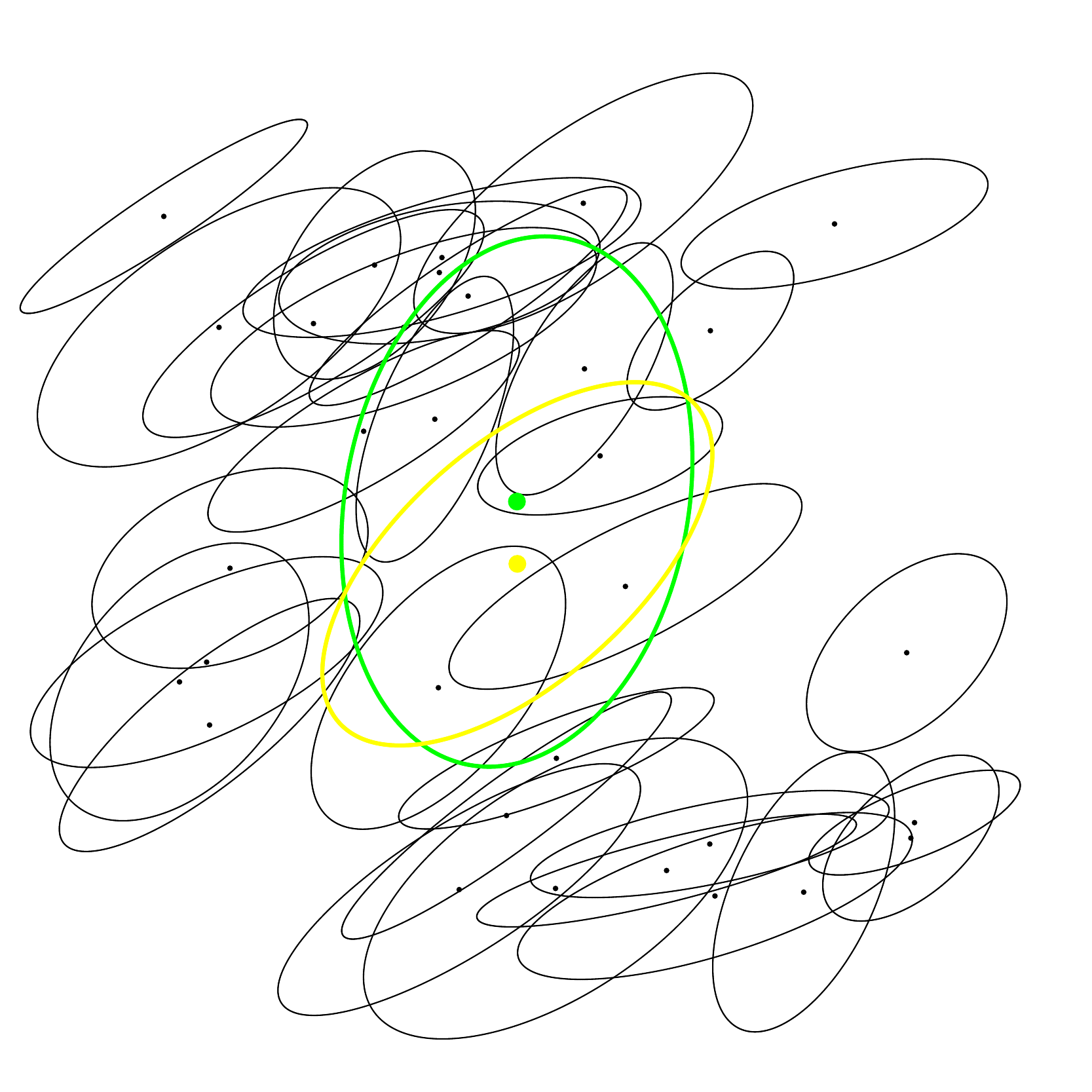}} &
\fbox{\includegraphics[width=0.45\columnwidth]{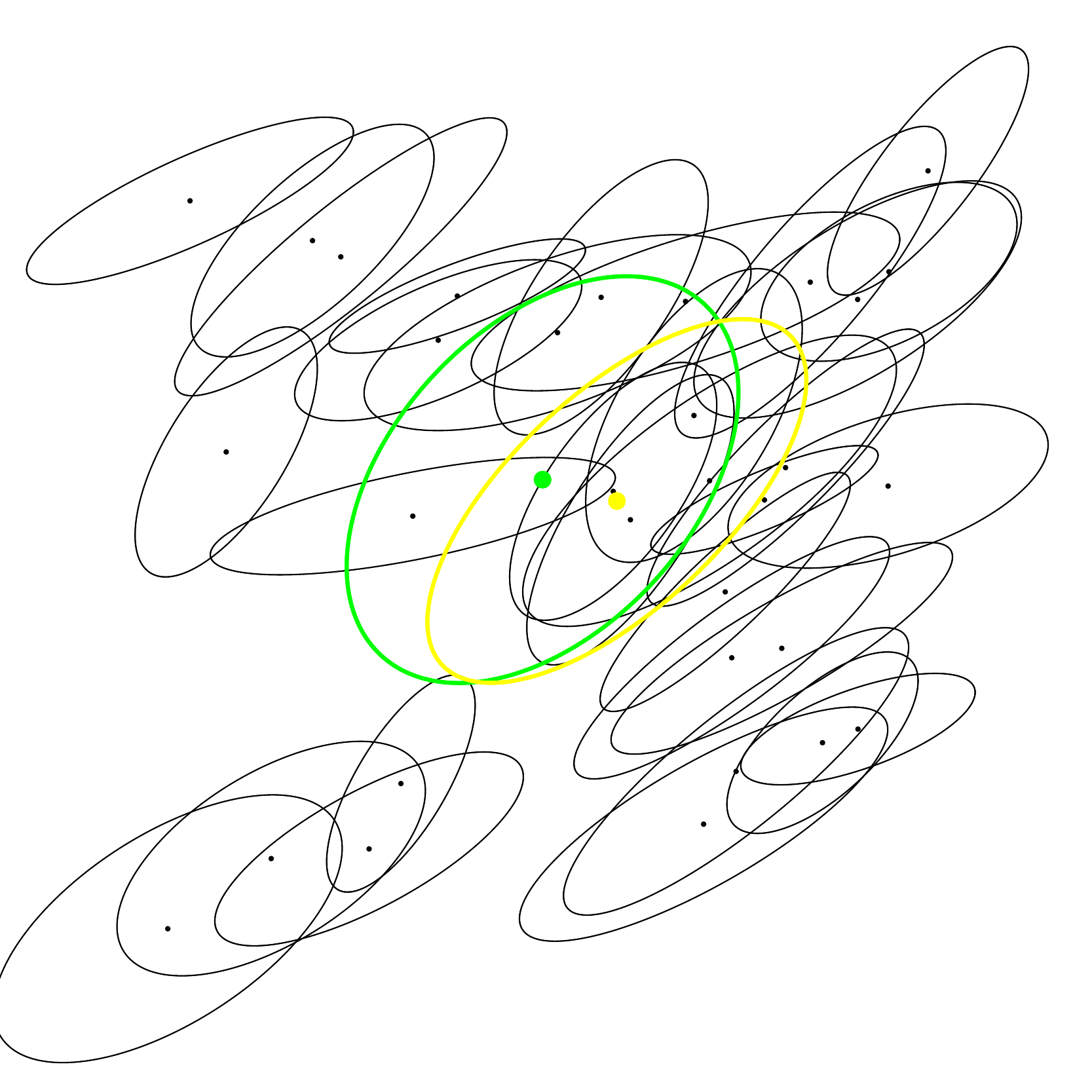}}\cr
(a) & (b)
\end{tabular}

\caption{Approximate Fisher centroid of $n$ bivariate normals (yellow) vs approximate Hilbert minimax center (green).
\label{fig:FRcentroid}}

\end{figure}

\vskip 0.5cm
\noindent Online materials are available at \url{https://franknielsen.github.io/FisherRaoMVN/}

\bibliographystyle{plain}

\bibliography{RaoMVNTAGBIB-arxiv}

%%%%%%%%%%%%%%%%%%%%%%%%%%%%%%%%%%%%%%%%%%%%%%%%%%%%%%%%%%%%%%%%%%%%%%%%%%%%%%%
 
\appendix

%%%
\section{Fisher-Rao geodesics between MVNs with initial value conditions}\label{sec:gs}
%%%

The Fisher-Rao geodesics $\gamma(t)$ are smooth curves which are autoparallel with respect to the Levi-Civita connection $\nabla^g$ induced by the metric tensor $g$: $\nabla^g_{\dot\gamma} \dot\gamma=0$.
On the MVN manifold, the system of Riemannian geodesic equations~\cite{Skovgaard-1984} is
$$
\left\{ \begin{array}{lcl}
\ddot\mu-\dot\Sigma\Sigma^{-1}\dot\mu &=& 0,\\
\ddot\Sigma+\dot\mu\dot\mu^\top-\dot\Sigma\Sigma^{-1}\dot\Sigma &=& 0.
\end{array}
\right.
$$

We may solve the above differential equation system either using initial value conditions (IVPs) by prescribing $N_0=(\mu(0),\Sigma(0))$ and a tangent vector $\dot N(0)=(\dot\mu(0),\dot\Sigma(0))$, or with boundary value conditions (BVPs) by prescribing 
$N_0=(\mu(0),\Sigma(0))$  and $N_1=(\mu(1),\Sigma(1))$.
Let $\gamma_\calN^\Fisher(N_0,\dot N_0;t)$ and $\gamma_\calN^\Fisher(N_0,N_01;t)$ denote these two types of geodesics.

Without loss of generality, let us assume  $N_0=N(0,I)$ (standard normal distribution).
The task is to perform geodesic shooting, i.e., calculate $N(t)=\gamma_\calN^\Fisher(N_0,v_0;t)$ with 
some prescribed initial condition $(v,S)=\dot\gamma_\calN^\Fisher(N_0,v_0;0)\in T_{N\std}\calM$ and $t\geq 0$.

%%%%
\subsection{Eriksen's homogeneous symmetric space solution}\label{app:Eriksen}
%%%%

We use the following natural parameterization  of the normal distributions:
$$
\left(\xi=\Sigma^{-1}\mu,\Xi=\Sigma^{-1}\right).
$$ 
The initial conditions are given by $(a=\dot\xi(0),B=\dot\Xi(0))=\dot\gamma_\calN^\Fisher(N_\std,v_0;0)$ at the standard normal distribution $N_\std$.
Eriksen's method proceeds as follows to calculate $\underline{\gamma}_\FR^\calN(N_\std,v_0;t)$:
\begin{itemize}
\item Build a $(2d+1)\times (2d+1)$ matrix $A$:
$$
A=\matthreethree{-B}{a}{0}{a^\top}{0}{-a^\top}{0}{-a}{B}.
$$
\item Compute matrix $M(t)=\exp(tA)$ for the prescribed value $t>0$.
To calculate $\exp(M)$, we first compute the eigen decomposition of 
$AM=O\, \diag(\lambda_1,\ldots,\lambda_d)\, O^\top$ and then reports $O\, \diag(\exp(\lambda_1),\ldots,\exp(\lambda_d))\, O^\top$.

\item Extract $(\xi(t),\Xi(t))$ from $M(t)$ and convert to $(\mu(t),\Sigma(t))$:
\begin{eqnarray*}
\Sigma(t) &=& [M(t)]^{-1}_{1:d,1:d},\\
\mu(t) &=& \Sigma(t) [M(t)]_{1:d,d+1}.
\end{eqnarray*}

\item Report $\underline{\gamma}_\FR^\calN(N_\std,v_0;t)=(\mu(t),\Sigma(t))$.
\end{itemize}

Eriksen's method yields a pregeodesic and not a geodesic.
For geodesics, we have the following Fisher-Rao distance property $\rho_\calN({\gamma}_\FR^\calN(N_\std,v_0;s),{\gamma}_\FR^\calN(N_\std,v_0;t))=|s-t|\, \rho_\calN({\gamma}_\FR^\calN(N_0,N_1)$ where $N_0=N_\std$ and $N_1={\gamma}_\FR^\calN(N_\std,v_0;1)$.

%%%
\subsection{Calvo and Oller's direction solution}\label{app:CO}
%%%
We report the solution given in~\cite{calvo1991explicit} which relies on the following natural parameterization 
of the normal distributions
$$
\left(\xi=\Sigma^{-1}\mu,\Xi=\Sigma^{-1}\right).
$$ 
The initial conditions are given by $(a=\dot\xi(0),B=\dot\Xi(0))=\dot\gamma_\calN^\Fisher(N_0,v_0;0)$.

The method of~\cite{calvo1991explicit}  first calculate those quantities:
\begin{eqnarray*}
B &=& -\Xi(0)^{-\frac{1}{2}}\, \dot\Xi(0)\, \Xi(0)^{-\frac{1}{2}},\\
a &=& \Xi(0)^{-\frac{1}{2}}\dot\xi(0)+B\Xi_0^{-\frac{1}{2}}\xi(0),\\
G &=& (B^2+2aa^\top)^{\frac{1}{2}}.
\end{eqnarray*}

Furthermore, let $G^\dagger=G^{-1}$ when $G$ is invertible or $G^\dagger=(G^\top G)^{-1}G^\top$  the Moore-Penrose generalized pseudo-inverse matrix of $G$ otherwise (or any kind of generalized matrix inverse $G^-$~\cite{calvo1991explicit}, see).

Then we have $(\xi(t),\Xi(t))=\gamma_\calN^\Fisher(N_0,v_0;t)$ with
\begin{eqnarray*}
\Xi(t) &=& \Xi(0)^{\frac{1}{2}}\, R(t)R(t)^\top\, \Xi(0)^{\frac{1}{2}},\\
\xi(t) &=& 2\Xi(0)^{\frac{1}{2}}\, R(t)\Sinh\left(\frac{1}{2} G  t\right)G^\dagger a+\Xi(t)\Xi^{-1}(0)\xi(0),
\end{eqnarray*}
and 
$$
R(t)=\Cosh\left(\frac{1}{2} G  t\right) -BG^\dagger\Sinh\left(\frac{1}{2} G  t\right).
$$

The matrix hyperbolic cosine and sinus functions of $M$ are calculated from the eigen decomposition of 
$M=O\, \diag(\lambda_1,\ldots,\lambda_d)\, O^\top$ as follows:
\begin{eqnarray*}
\Sinh(M) &=& O\, \diag(\sinh(\lambda_1),\ldots,\sinh(\lambda_d))\, O^\top,\quad \sinh(u)=\frac{e^u-e^{-u}}{2}=
\sum_{i=0}^\infty \frac{u^{2i+1}}{(2i+1)!}, \\
\Cosh(M) &=& O\, \diag(\cosh(\lambda_1),\ldots,\cosh(\lambda_d))\, O^\top,\quad \cosh(u)=\frac{e^u+e^{-u}}{2}
=\sum_{i=0}^\infty \frac{u^{2i}}{(2i)!}.
\end{eqnarray*}

For the general case $\gamma_\calN^\Fisher(N,v_0;t)$ with arbitrary $N=(\Sigma,\mu)$,  
we use the affine equivariance property of the Fisher-Rao geodesics with $P=\Sigma^{-\frac{1}{2}}$:

\begin{equation}
\gamma_\calN^\Fisher(N,v_0;t) =  (-P\mu,P^{-1}).\gamma_\calN^\Fisher(N_\std,(Pa,-PBP^\top);t).
\end{equation}

Figure~\ref{fig:geoIVP} displays several examples of geodesics from the standard normal distribution with various initial value conditions.  

\begin{figure}
\centering
\begin{tabular}{cc}
\fbox{\includegraphics[width=0.45\textwidth]{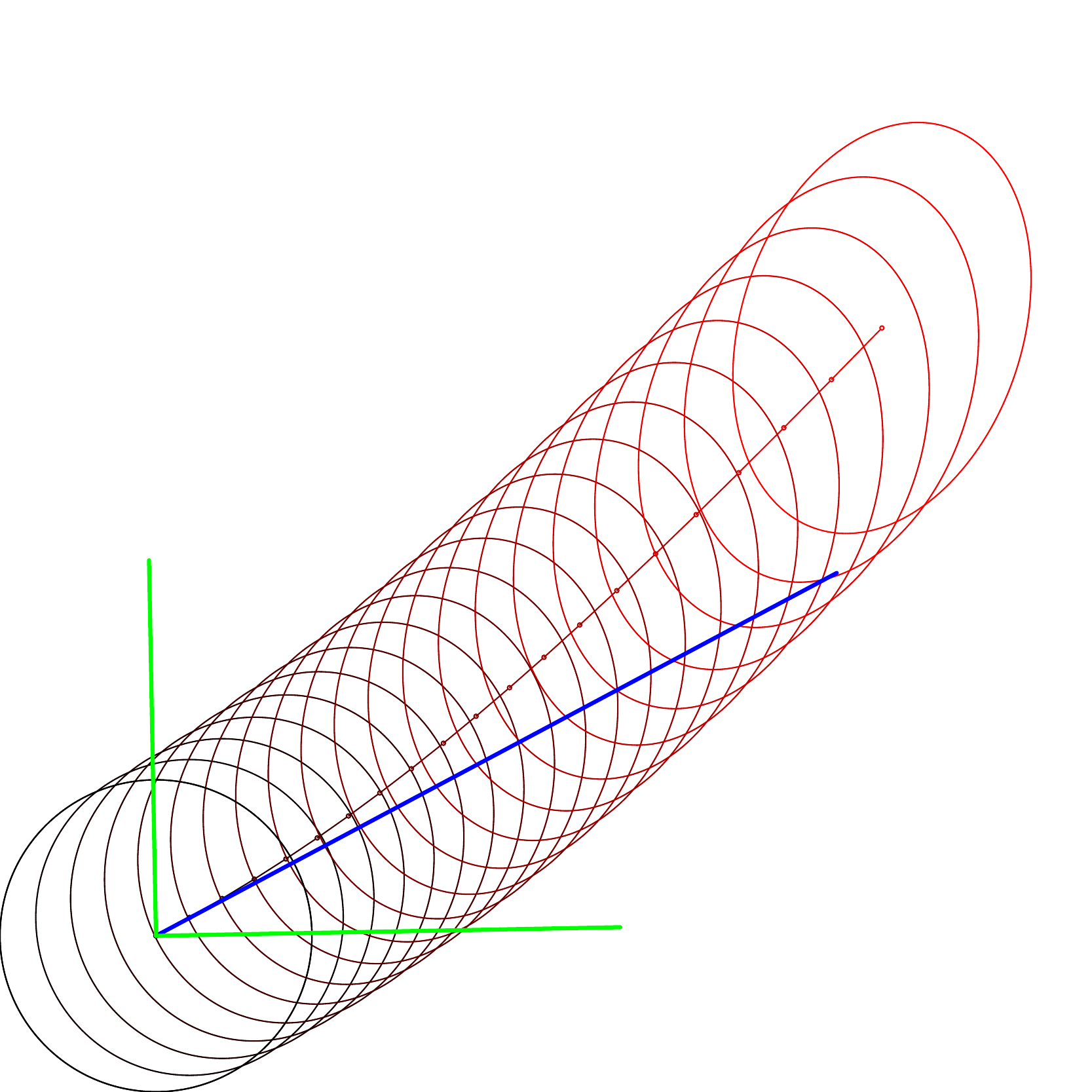}} &
\fbox{\includegraphics[width=0.45\textwidth]{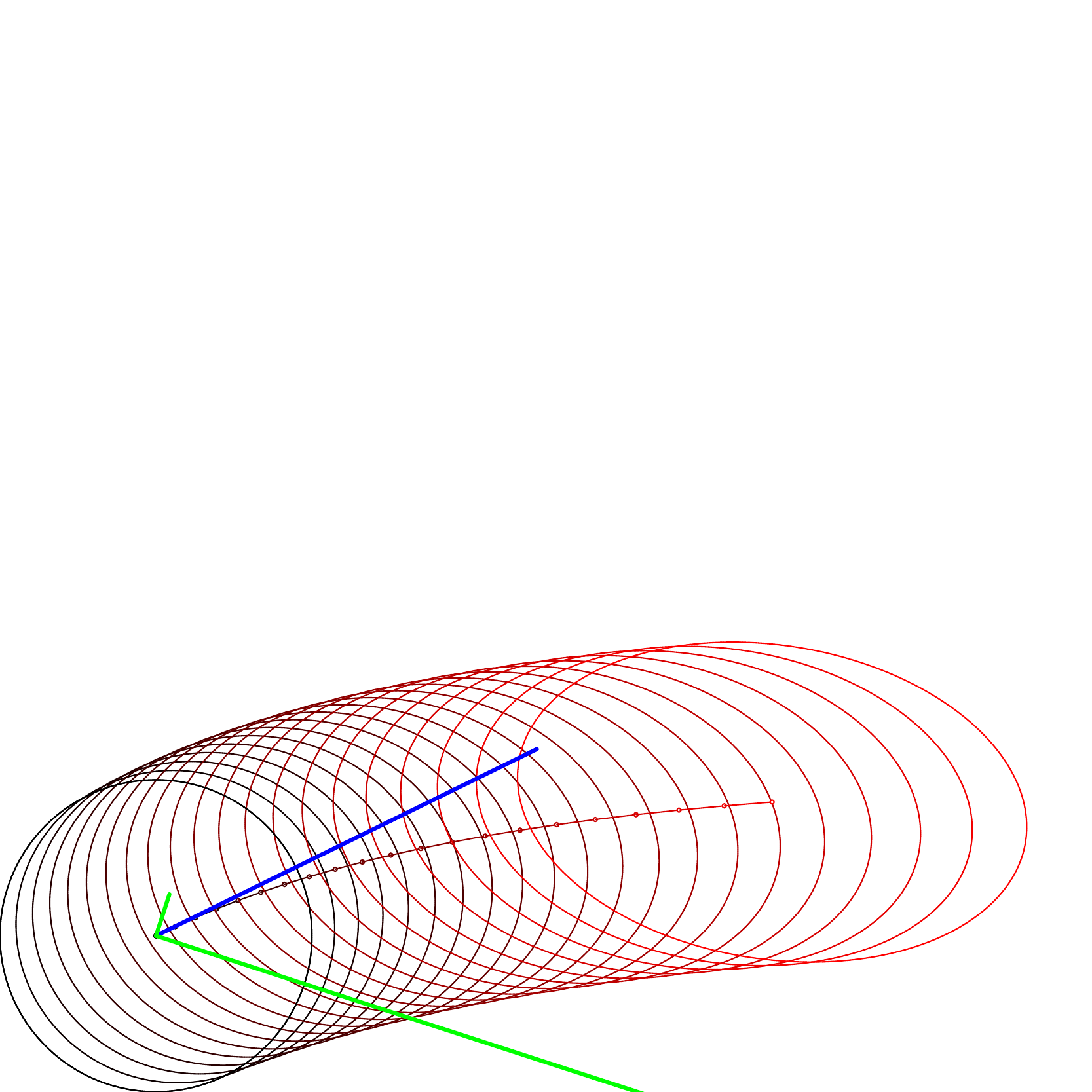}} \cr
\fbox{\includegraphics[width=0.45\textwidth]{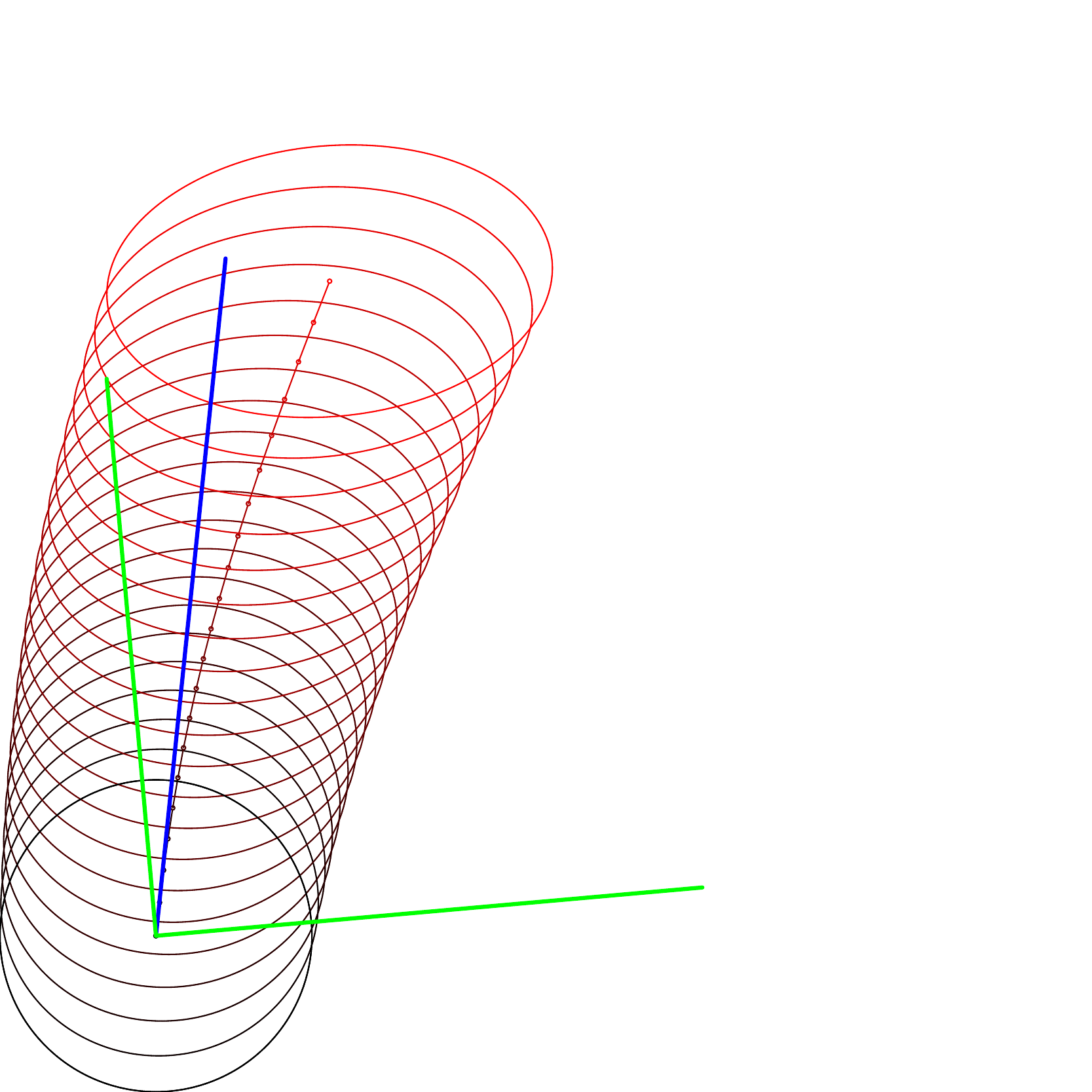}} &
\fbox{\includegraphics[width=0.45\textwidth]{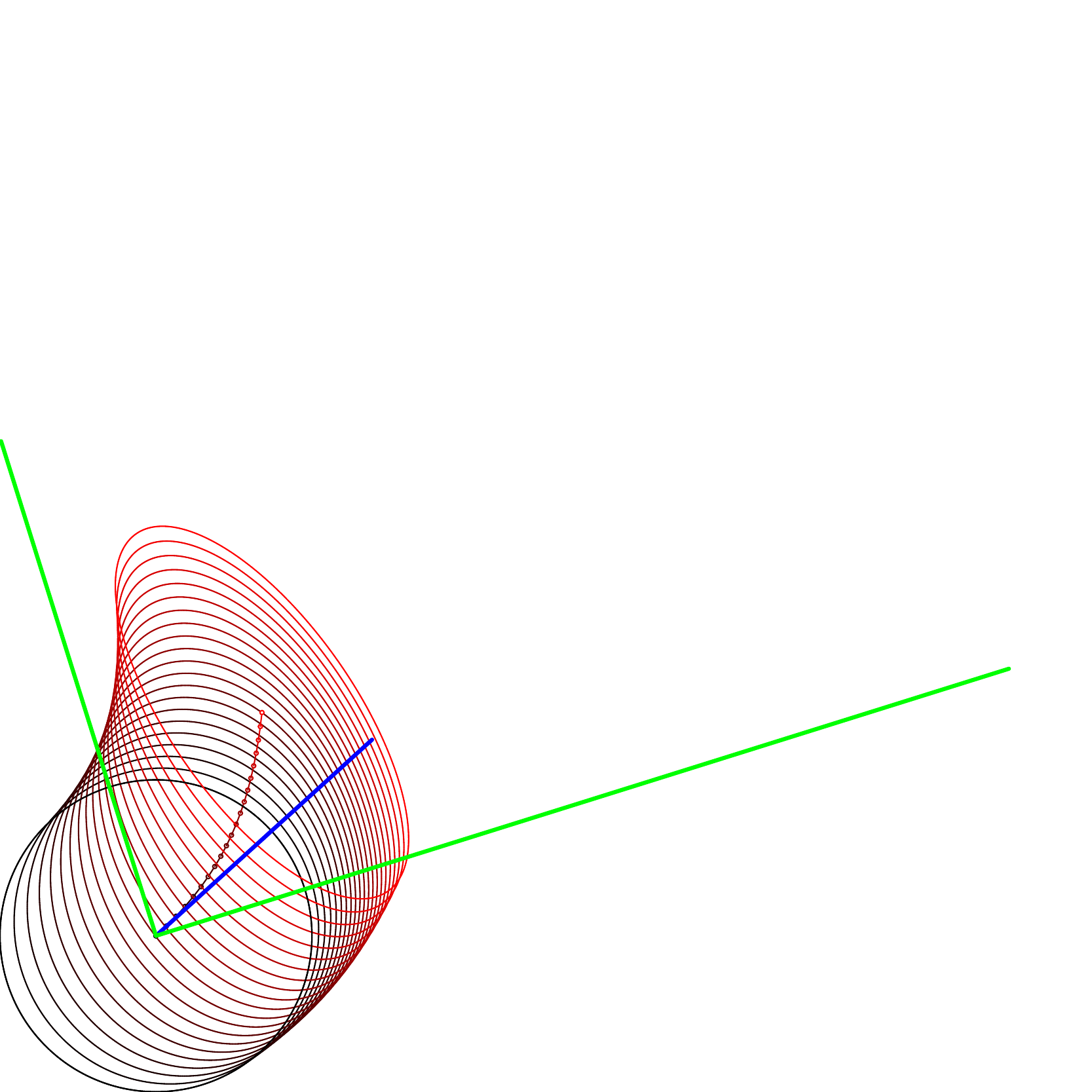}} \cr
\end{tabular}
\caption{
Examples of Fisher-Rao geodesics $(\mu(t),\Sigma(t))$ emanating from the standard bivariate normal distribution 
$(\mu(0),\Sigma(0))=N(0,I)$ with initial value conditions $(\dot\mu(0),\dot\Sigma(0))$. 
Vectors $\dot\mu(0)$ are shown in blue and symmetric matrices $\dot\Sigma(0)=\lambda_1v_1v_1^\top+\lambda_2v_2v_2^\top$ are visualized by
their two scaled eigenvectors $\lambda_1 v_1$ and $\lambda_2 v_2$ shown in green.
 \label{fig:geoIVP}}
\end{figure}

The geodesics with initial values let us define the Riemannian exponential map $\exp: T_N\calM\rightarrow \calM$:
$$
\exp_N(v)=\gamma_\calN^\Fisher(N,v;1).
$$
The inverse map is the Riemannian logarithm map.
In a geodesically complete manifold (e.g., $\calN_\mu$), we can express the Fisher-Rao distance as:
$$
\rho_\calN(N_1,N_2)=\|\Log_{N_1}(N_2)\|_{N_1}.
$$
Thus computing the Fisher-Rao distance can be done by computing the Riemannian MVN logarithm.

%%%
\section{Proof of square root of Jeffreys upper bound}\label{sec:proof}
%%%%

Let us prove that the Fisher-Rao distance between  normal distributions is upper bounded by the square root of the Jeffreys divergence: 
$$
\rho_\calN(N_1,N_2) \leq \sqrt{D_J(N_1,N_2)}.
$$

\begin{property}
We have 
$$
D_J[p_{\lambda_1},p_{\lambda_2}]=\int_0^1 \ds_\calN^2(\gamma^m_\calN(p_{\lambda_1},p_{\lambda_2};t))\dt =
\int_0^1 \ds_\calN^2(\gamma^e_\calN(p_{\lambda_1},p_{\lambda_2};t))\dt.
$$
\end{property}

Let $S_F(\theta_1;\theta_2)=B_F(\theta_1:\theta_2)+B_F(\theta_2:\theta_1)$ be a symmetrized Bregman divergence.
Let $\ds^2=\dtheta^\top \nabla^2 F(\theta)\dtheta$ denote the squared length element on the Bregman manifold 
and denote by $\gamma(t)$ and $\gamma^*(t)$ the dual geodesics connecting $\theta_1$ to $\theta_2$.
We can express $S_F(\theta_1;\theta_2)$ as integral energies on dual geodesics:

\begin{property}
We have $S_F(\theta_1;\theta_2)=\int_0^1 \ds^2(\gamma(t))\dt=\int_0^1 \ds^2(\gamma^*(t))\dt$.
\end{property}

\begin{proof}
The proof that the symmetrized Bregman divergence amount to these energy integrals
is based on the first-order and second-order directional derivatives.
The first-order directional derivative $\nabla_u F(\theta)$ with respect to vector $u$ is defined by 
$$
\nabla_u F(\theta)=\lim_{t\rightarrow 0} \frac{F(\theta+tv)-F(\theta)}{t}=v^\top \nabla F(\theta).
$$
 
The second-order directional derivatives $\nabla_{u,v}^2 F(\theta)$ is
\begin{eqnarray*}
\nabla_{u,v}^2 F(\theta) &=& \nabla_{u} \nabla_v F(\theta),\\
 &=& \lim_{t\rightarrow 0} \frac{v^\top \nabla F(\theta+tu)-v^\top\nabla F(\theta)}{t},\\
&=& u^\top \nabla^2 F(\theta) v.
\end{eqnarray*}

Now consider the squared length element $\ds^2(\gamma(t))$ on the primal geodesic $\gamma(t)$ expressed using the primal coordinate system $\theta$:
$\ds^2(\gamma(t))=\dtheta(t)^\top \nabla^2F(\theta(t)) \dtheta(t)$ with $\theta(\gamma(t))=\theta_1+t(\theta_2-\theta_1)$ and $\dtheta(t)=\theta_2-\theta_1$.
Let us express the $\ds^2(\gamma(t))$   using the second-order directional derivative:
$$
\ds^2(\gamma(t))=\nabla^2_{\theta_2-\theta_1}  F(\theta(t)).
$$
Thus we have $\int_0^1 \ds^2(\gamma(t))\dt=[\nabla_{\theta_2-\theta_1}  F(\theta(t))]_0^1$,
where the first-order directional derivative is $\nabla_{\theta_2-\theta_1}  F(\theta(t))=(\theta_2-\theta_1)^\top \nabla F(\theta(t))$.
Therefore we get $\int_0^1 \ds^2(\gamma(t))\dt=(\theta_2-\theta_1)^\top (\nabla F(\theta_2)-\nabla F(\theta_1))=S_F(\theta_1;\theta_2)$.

Similarly, we express the squared length element $\ds^2(\gamma^*(t))$ using the dual coordinate system $\eta$ as the second-order directional derivative of $F^*(\eta(t))$ with $\eta(\gamma^*(t))=\eta_1+t(\eta_2-\eta_1)$:
$$
\ds^2(\gamma^*(t))=\nabla^2_{\eta_2-\eta_1}  F^*(\eta(t)).
$$
Therefore, we have  $\int_0^1 \ds^2(\gamma^*(t))\dt=[\nabla_{\eta_2-\eta_1}  F^*(\eta(t))]_0^1=S_{F^*}(\eta_1;\eta2)$.
Since $S_{F^*}(\eta_1;\eta_2)=S_F(\theta_1;\theta_2)$, we conclude that
$$
S_F(\theta_1;\theta_2)=\int_0^1 \ds^2(\gamma(t))\dt=\int_0^1 \ds^2(\gamma^*(t))\dt
$$

Note that in 1D, both pregeodesics $\gamma(t)$ and $\gamma^*(t)$ coincide. We have $\ds^2(t)=(\theta_2-\theta_1)^2 f''(\theta(t))=(\eta_2-\eta_1){f^*}''(\eta(t))$ so that we check that $S_F(\theta_1;\theta_2)=\int_0^1 \ds^2(\gamma(t))\dt=(\theta_2-\theta_1)[f'(\theta(t))]_0^1=(\eta_2-\eta_1)[{f^*}'(\eta(t))]_0^1=(\eta_2-\eta_1)(\theta_2-\theta_2)$.
 \end{proof}

\begin{property}[\cite{IG-2016}]\label{prop:geolengthJeffreys}
We have $$
D_J[p_{\lambda_1},p_{\lambda_2}]=\int_0^1 \ds_\calN^2(\gamma^m_\calN(p_{\lambda_1},p_{\lambda_2};t))\dt =
\int_0^1 \ds_\calN^2(\gamma^e_\calN(p_{\lambda_1},p_{\lambda_2};t))\dt.
$$
\end{property}

\begin{proof}
Let us report a proof of this remarkable fact  in the general setting of Bregman manifolds.
Indeed, since 
$$
D_J[p_{\lambda_1},p_{\lambda_2}]=D_\KL[p_{\lambda_1}:p_{\lambda_2}]+D_\KL[p_{\lambda_2}:p_{\lambda_1}],
$$
and $D_\KL[p_{\lambda_1}:p_{\lambda_2}]=B_F(\theta(\lambda_2):\theta(\lambda_1))$, where $B_F$ denotes the Bregman divergence induced by the cumulant function of the multivariate normals and $\theta(\lambda)$ is the natural parameter corresponding to $\lambda$, we have
\begin{eqnarray*}
D_J[p_{\lambda_1},p_{\lambda_2}]&=&B_F(\theta_1:\theta_2)+B_F(\theta_2:\theta_1),\\
&=& S_F(\theta_1;\theta_2)=(\theta_2-\theta_1)^\top (\eta_2-\eta_1)=S_{F^*}(\eta_1;\eta_2),
\end{eqnarray*}
where $\eta=\nabla F(\theta)$ and $\theta=\nabla F^*(\eta)$ denote the dual parameterizations obtained by the Legendre-Fenchel convex conjugate  $F^*(\eta)$ of $F(\theta)$. Moreover, we have $F^*(\eta)=-h(p_{\mu,\Sigma})$~\cite{IG-2016}, i.e., the convex conjugate function is Shannon negentropy.

Then we conclude by using the fact that $S_F(\theta_1;\theta_2)=\int_0^1 \ds^2(\gamma(t))\dt=\int_0^1 \ds^2(\gamma^*(t))\dt$, 
i.e., the symmetrized Bregman divergence amounts to integral energies on dual geodesics on a Bregman manifold.
The proof of this general property is reported in Appendix~\ref{sec:proof}.
\end{proof}
 
 \begin{property}[Fisher--Rao upper bound] 
The Fisher-Rao distance between   normal distributions is upper bounded by the square root of the Jeffreys divergence: $\rho_\calN(N_1,N_2) \leq\sqrt{D_J(N_1,N_2)}$.
\end{property}

\begin{proof}
Consider the Cauchy-Schwarz inequality for positive functions $f(t)$ and $g(t)$:
 $$
\int_0^1 f(t)\, g(t)\dt\leq\sqrt{(\int_0^1 f(t)^2\dt)\, (\int_0^1 g(t)^2\dt)},
$$
and let $f(t)=\ds_\calN(\gamma^c_\calN(p_{\lambda_1},p_{\lambda_2};t)$ and $g(t)=1$. 
Then we get: 
$$
\left(\int_0^1 \ds_\calN(\gamma^c_\calN(p_{\lambda_1},p_{\lambda_2};t)\dt\right)^2
\leq \left(\int_0^1 \ds_\calN^2(\gamma^c_\calN(p_{\lambda_1},p_{\lambda_2};t)\dt\right) 
\left( \underbrace{\int_0^1 1^2 \dt}_{=1} \right).
$$
Furthermore since by definition of $\gamma_\calN^{\mathrm{FR}}$, we have 
$$
\int_0^1 \ds_\calN(\gamma^c_\calN(p_{\lambda_1},p_{\lambda_2};t)\dt\geq \int_0^1 \ds_\calN(\gamma^{\mathrm{FR}}_\calN(p_{\lambda_1},p_{\lambda_2};t)\dt=:\rho_\calN(N_1,N_2).
$$

It follows for $c=\gamma^e_\calN$ (i.e., $e$-geodesic) using Property~\ref{prop:geolengthJeffreys} that we have:

$$
\rho_\calN(N_1,N_2)^2 \leq \int_0^1 \ds_\calN^2(\gamma^e_\calN(p_{\lambda_1},p_{\lambda_2};t)\dt = D_J(N_1,N_2).
$$
Thus we conclude that $\rho_\calN(N_1,N_2) \leq\sqrt{D_J(N_1,N_2)}$.

Note that in Riemannian geometry, a curve $\gamma$ minimizes the energy $E(\gamma)=\int_0^1 \|\dot\gamma(t)\|^2\dt$ if it minimizes the length $\Length(\gamma)=\int_0^1 \|\dot\gamma(t)\|\dt$ and $\|\dot\gamma(t)\|$ is constant. Using Cauchy-Schwartz inequality, we can show that 
$\Length(\gamma)\leq E(\gamma)$.
\end{proof}

\section{Riemannian SPD geodesic and the arithmetic-harmonic inductive mean}\label{sec:ahm}

\def\AHM{\mathrm{AHM}}
The Riemannian SPD geodesic $\gamma(X,Y;t)$ joining two SPD matrices $X$ and $Y$ with respect to the trace metric can be expressed using the weighted matrix geometric mean:
\begin{equation}\label{eq:nwmm}
\gamma(X,Y;t)=X\#_t Y= 
X^{\frac{1}{2}}\, \left(X^{-\frac{1}{2}}\, Y\, X^{-\frac{1}{2}}\right)^t\, X^{\frac{1}{2}}.
\end{equation}
We denote by $X\# Y=X\#_{\frac{1}{2}} Y=$.
For $t\in (0,1)$, $X\#_t Y$ defines a matrix mean~\cite{bhatia2006noncommutative} of $X$ and $Y$ which is asymmetric when $t\not=\frac{1}{2}$.

The  matrix geometric mean can be computed inductively using the following arithmetic-harmonic sequence:
\begin{eqnarray*}
A_{t+1} &=& A(A_t,H_t),\\
H_{t+1} &=& H(A_t,H_t),
\end{eqnarray*}
where the matrix arithmetic mean is $A(X,Y)=\frac{X+Y}{2}$ 
and the matrix harmonic mean is $H(X,Y)=2(X^{-1}+Y^{-1})^{-1}$.
The sequence is initialized with $A_0=X$ and $H_0=Y$.
We have $\AHM(X,Y)=\lim_{t\rightarrow\infty} A_t=\lim_{t\rightarrow\infty} H_t=X\#_{\frac{1}{2}} Y$, 
and the convergence is of quadratic order~\cite{nakamura2001algorithms}.
This iterative method converges to $X\# Y$, the non-weighted matrix geometric mean. 
In general, taking weighted arithmetic and harmonic means $A(X,Y)=(1-\alpha)X+\alpha Y$ and 
$H(X,Y)=((1-\alpha)X^{-1}+\alpha Y^{-1})^{-1}$ yields convergence to a matrix which is not the weighted geometric mean $X\#_\alpha Y$ (except when $\alpha=\frac{1}{2}$.
The method requires to compute the matrix harmonic mean which requires to inverse matrices.
The closed-form formula of the matrix weighted geometric mean of Eq.~\ref{eq:nwmm} requires to compute a matrix fractional power which can be done from a matrix eigen decomposition.

%%%
\section{Power method to approximate the largest and smallest eigenvalues}\label{sec:powermethod}
%%%%
We concisely recall the power method and its computational complexity to approximate the largest eigenvalue $\lambda_1$ of a $d\times d$ symmetric positive-definite matrix $P$ following~\cite{trevisan2017lecture}:
\begin{itemize}
	\item Pick uniformly at random $x^{(0)}\in \{-1,1\}^d$
	\item For $t\in (1,\ldots,T)$ do $x^{(t)}\leftarrow P\, x^{(t-1)}$
	\item Return $\tilde\lambda_1=\frac{ \inner{x^{(T)}}{Px^{(T)}}}{\inner{x^{(T)}}{x^{(T)}}}$, where $\inner{x}{y}=x^\top y$ denotes the dot product.
\end{itemize}
The complexity of the power method with $T$ iterations is $O(T(d+m))$ where $m=O(d^2)$ is the number of non-zero entries of $P$. 
Furthermore, with probability $\geq\frac{3}{16}$, the iterative power method with $T=O\left(\frac{d}{\epsilon}d\right)$ iterations yields 
$\tilde\lambda_1\geq (1-\epsilon)\lambda_1$ for any $\epsilon>0$~\cite{trevisan2017lecture}.
Due to its vector-matrix product operations, the power method can be efficiently implemented on GPU~\cite{ballard2011efficiently}.

To compute an approximation $\tilde\lambda_d$ of the smallest eigenvalues $\lambda_d$ of $P$, we first compute the matrix inverse $P^{-1}$ and then compute the approximation of the largest eigenvalue of $P^{-1}$.
We report $\tilde\lambda_d(P)=\tilde\lambda_1(P^{-1})$.
The complexity of computing a matrix inverse is as hard as computing the matrix product~\cite{cormen2022introduction}.
The current best algorithm requires $O(d^\omega)$ operations with $\omega=2.373$~\cite{alman2021refined}.

\section{Source code for Fisher-Rao normal geodesics}\label{sec:java}

\begin{lstlisting}[language = Java]
// (c) 2023 Frank.Nielsen@acm.org

import Jama.*;

class FisherRaoMVN
{
  public static double sqr(double x) {
    return x*x;
  }

  public static double min(double a, double b) {
    if (a<b) return a;
    else return b;
  }
  public static double max(double a, double b) {
    if (a>b) return a;
    else return b;
  }
  
  public static double arccosh(double x)
  {
    return Math.log(x+Math.sqrt(x*x-1.0));
  }
  
  
  public static Matrix InductiveAHM(Matrix X, Matrix Y, int nbiter)
{
 Matrix XX,YY;
 int i;
 
 for(i=0;i<nbiter;i++)
 {
   XX=ArithmeticMean(X,Y);
   YY=HarmonicMean(X,Y);
   X=XX;Y=YY;
 }
 return X;
}


  public static double MobiusDistance(double a, double b, double c, double d)
  {
    return Math.sqrt((sqr(c-a)+2.0*sqr(d-b))/(sqr(c-a)+2.0*sqr(d+b)));
  }
  
    // OK
  public static double RaoDistance1D(double m1, double s1, double m2, double s2)
  {
    return Math.sqrt(2)*Math.log((1.0+MobiusDistance(m1, s1, m2, s2))/(1.0-MobiusDistance(m1, s1, m2, s2)));
  }

public static Matrix ArithmeticMean(Matrix X, Matrix Y)
{
return (X.times(0.5)).plus(Y.times(.05));
}

public static Matrix HarmonicMean(Matrix X, Matrix Y)
{
return (X.inverse().times(0.5)).plus(Y.inverse().times(0.5)).inverse();
}


  
  public static  Matrix IsometricSPDEmbeddingCalvoOller(vM N)
  {
    Matrix mu=N.v;
    Matrix Sigma=N.M;
    int d=mu.getRowDimension();
    Matrix res=new Matrix(d+1, d+1);
    Matrix tmp=new Matrix(d, d);

    int i, j;

    tmp=Sigma.plus(  (mu.times(mu.transpose()))  );

    // Left top corner
    for (i=0; i<d; i++)
      for (j=0; j<d; j++)
      {
        res.set(i, j, tmp.get(i, j));
      }

    // Left column
    for (i=0; i<d; i++)
    {
      res.set(d, i, mu.get(i, 0));
    }

    // Bottom row
    for (i=0; i<d; i++)
    {
      res.set(i, d, mu.get(i, 0));
    }

    // Last element is one
    res.set(d, d, 1);

    return res;
  }

  public static double LowerBoundCalvoOller(vM N1, vM N2)
  {
    Matrix P1bar=IsometricSPDEmbeddingCalvoOller(N1);
    Matrix P2bar=IsometricSPDEmbeddingCalvoOller(N2);
    // We calculated (1/sqrt(2))*RiemannSPD(d+1)
    return ScaledRiemannianSPDDistance(P1bar, P2bar);
  }

  public static double UpperBoundSqrtJeffreys(vM N1, vM N2)
  {
    return Math.sqrt(JeffreysMVN(N1, N2));
  }

  public static double JeffreysMVN(vM N1, vM N2)
  {
    return KLMVN(N1.v, N1.M, N2.v, N2.M)+KLMVN(N2.v, N2.M, N1.v, N1.M);
  }


  public static double KLMVN(Matrix m1, Matrix S1, Matrix m2, Matrix S2)
  {
    int d=S1.getColumnDimension();
    Matrix Deltam=m1.minus(m2);

    return 0.5*(
      ((S2.inverse()).times(S1)).trace()-d+
      ((Deltam.transpose().times(S2.inverse())).times(Deltam)).trace()
      +Math.log(S2.det()/S1.det())
      );
  }

  // scale the trace metric by 1/2
  public static double ScaledRiemannianSPDDistance(Matrix P, Matrix Q)
  {
    double result=0.0d;
    Matrix M=P.inverse().times(Q);
    EigenvalueDecomposition evd=M.eig();
    Matrix D=evd.getD();

    for (int i=0; i<D.getColumnDimension(); i++)
    {
      result+=sqr(Math.log(D.get(i, i)));
    }

    return Math.sqrt(result/2.0);// ok beta=1/2
  }



  public static vM KobayashiMVNGeodesic(vM N0, vM N1, double t)
  {
    vM result;
    Matrix G0, G1, Gt;

    G0=EmbedBlockCholesky(N0);
    G1=EmbedBlockCholesky(N1);

    Gt=RiemannianGeodesic(G0, G1, t); // Rie geodesic in SPD(2d+1)
    // Extract
    result=L2MVN(Gt);

    return result;
  }

  Matrix RiemannianGeodesic(Matrix P0, Matrix P1)
  {
    return RiemannianGeodesic(P0, P1, 0.5);
  }


  public static Matrix RiemannianGeodesic(Matrix P, Matrix Q, double lambda)
  {
    if (lambda==0.0) return P;
    if (lambda==1.0) return Q;

    Matrix result;
    Matrix Phalf=power(P, 0.5);
    Matrix Phalfinv=power(P, -0.5);

    result=Phalfinv.times(Q).times(Phalfinv);
    result=power(result, lambda);



    result=(Phalf.times(result)).times(Phalf);



    return result;
  }

  // Non-integer power of a matrix
  public static Matrix power(Matrix M, double p)
  {
    EigenvalueDecomposition evd=M.eig();
    Matrix D=evd.getD();
    // power on the positive eigen values
    for (int i=0; i<D.getColumnDimension(); i++)
      D.set(i, i, Math.pow(D.get(i, i), p));

    Matrix V=evd.getV();

    return V.times(D.times(V.transpose()));
  }

  public static Matrix EmbedBlockCholesky(vM N)
  {
    int i, j, d=N.M.getRowDimension();
    Matrix Sigma=N.M;
    Matrix SigmaInv=N.M.inverse();
    Matrix mut=N.v.transpose();
    Matrix minusmu=N.v.times(-1.0);


    Matrix M, D, L; // Block Cholesky
    M=new Matrix(2*d+1, 2*d+1);
    D=new Matrix(2*d+1, 2*d+1);

    D.set(d, d, 1.0);

    // Diagonal matrix
    for (i=0; i<d; i++)
      for (j=0; j<d; j++)
      {
        D.set(i, j, SigmaInv.get(i, j));
        D.set(d+1+i, d+1+j, Sigma.get(i, j));
      }



    M.set(d, d, 1.0);

    for (j=0; j<d; j++)
    {
      M.set(d, j, mut.get(0, j)); // row vector
      M.set(d+1+j, d, minusmu.get(j, 0)); // column vector
    }

    // add identity matrices
    for (i=0; i<d; i++)
      for (j=0; j<d; j++)
      {
        M.set(i, i, 1.0);
        M.set(i+d+1, i+d+1, 1.0);
      }

    // Block Cholesky decomposition
    L=M.times(D).times(M.transpose());

    return L;
  }


  static vM L2MVN(Matrix L)
  {
    int i, j, d;
    d= (L.getRowDimension()-1)/2;

    // System.out.println("extract mvn d="+d);

    Matrix Delta=new Matrix(d, d), delta=new Matrix(d, 1);
    for (i=0; i<d; i++)
      for (j=0; j<d; j++)
      {
        Delta.set(i, j, L.get(i, j));
      }

    for (j=0; j<d; j++)
    {
      delta.set(j, 0, L.get(d, j));
    }

    Matrix mu, Sigma;

    // convert natural parameters to ordinary parameterization
    Sigma=Delta.inverse();
    mu=Sigma.times(delta);

    return new vM(mu, Sigma);
  }


  public static double FisherRaoMVN(vM N0, vM N1, double eps)
  {
    double lb, ub;
    lb=LowerBoundCalvoOller(N0, N1);
    ub=UpperBoundSqrtJeffreys(N0, N1);
    if (ub/lb<1+eps) return ub;
    else
    {
      vM Nmid=KobayashiMVNGeodesic(N0, N1, 0.5);
      return FisherRaoMVN(N0, Nmid, eps)+FisherRaoMVN(Nmid, N1, eps);
    }
  }
  
    public static double SqrMahalanobisDistance(Matrix m1, Matrix m2, Matrix Sigma)
 {Matrix res;
 Matrix m12=m1.minus(m2);
 res=m12.transpose().times(Sigma.inverse()).times(m12);
 return res.get(0,0);
 }
 
 
  // non-totally geodesic submanifold
  	public static double FisherRaoSameCovariance(Matrix m1, Matrix m2, Matrix Sigma)
 {
  
 return  Math.sqrt(2)*arccosh(1.0+0.25*SqrMahalanobisDistance(m1,m2,Sigma));
 		
 }
 
 
 // totally geodesic submanifold
   	public static double FisherRaoSameMean(Matrix m, Matrix S1, Matrix S2)
 {
  
 return  ScaledRiemannianSPDDistance(S1,S2);
 		
 }



  public static void main (String [] args)
  {
    System.out.println("Fisher-Rao MVN distance");
    Test();//good
    //Test2();
  }


 public static  Matrix randomSPDCholesky(int d)
  	  {
    int i, j;
    double [][] array=new double[d][d];
    Matrix L;

    for (i=0; i<d; i++)
      for (j=0; j<=i; j++)
      {
        array[i][j]=Math.random();
      }

    L=new Matrix(array);
    // Cholesky
    if (L.det()==0) return randomSPDCholesky(d);
    return L.times(L.transpose());
  }

  public static void Test()
  {
    Matrix m1, S1, m2, S2;
    vM N1, N2;
    
    System.out.println("Test Fisher-Rao distance between multivariate normal distributions");

    m1=new Matrix(2, 1);
    m1.set(0, 0, 0);
    m1.set(1, 0, 0);
    S1=new Matrix(2, 2);
    S1.set(0, 0, 1.0);
    S1.set(0, 1, 0);
    S1.set(1, 0, 0);
    S1.set(1, 1, 0.1);

    m2=new Matrix(2, 1);
    m2.set(0, 0, 1);
    m2.set(1, 0, 1);
    //m2=new Matrix(2,1);m2.set(0,0,0);m2.set(1,0,0);
    S2=new Matrix(2, 2);
    S2.set(0, 0, 0.1);
    S2.set(0, 1, 0);
    S2.set(1, 0, 0);
    S2.set(1, 1, 1);

    N1=new vM(m1, S1);
    N2=new vM(m2, S2);
    
    
    
double err;
    double eps=0.01;
    
    double fr,lb,ub;
    
    fr=FisherRaoMVN(N1, N2, eps);
 lb=LowerBoundCalvoOller(N1, N2);
    ub=UpperBoundSqrtJeffreys(N1, N2);
     System.out.println("Han & Park 2014 example (mid range).");
    System.out.println("Fisher-Rao distance (eps="+eps+"):"+fr);
  System.out.println("lower bound:"+lb+" upper bound:"+ub);
    // 3.1329
    
    
   
    m1=new Matrix(2,1);S1=Matrix.identity(2,2);
    m2=new Matrix(2,1);m2.set(0,0,10000);
    S2=Matrix.identity(2,2).times(10);
        N1=new vM(m1, S1);
    N2=new vM(m2, S2);
    
 fr=FisherRaoMVN(N1, N2, eps);
 lb=LowerBoundCalvoOller(N1, N2);
    ub=UpperBoundSqrtJeffreys(N1, N2);
     System.out.println("Han & Park 2014 example (long range).");
    System.out.println("Fisher-Rao distance (eps="+eps+"):"+fr);
  System.out.println("lower bound:"+lb+" upper bound:"+ub);

    // 23.5
    
    m1=new Matrix(2,1);m1.set(0,0,Math.random());m1.set(1,0,Math.random());
    m2=new Matrix(2,1);m2.set(0,0,Math.random());m2.set(1,0,Math.random());
    S1=S2=randomSPDCholesky(2);
            N1=new vM(m1, S1);
    N2=new vM(m2, S2);
 fr=FisherRaoMVN(N1, N2, eps);
double  exact;
exact=FisherRaoSameCovariance(m1,m2,S1);
  System.out.println("Same covariance matrix:");
  System.out.println("Fisher-Rao distance (eps="+eps+"):"+fr);
  System.out.println("exact="+exact);
  
  S1=randomSPDCholesky(2);
  S2=randomSPDCholesky(2);
  m1=m2;
   N1=new vM(m1, S1);
    N2=new vM(m2, S2);
 fr=FisherRaoMVN(N1, N2, eps);
   exact=FisherRaoSameMean(m1,S1,S2);
    System.out.println("Same mean:");
  System.out.println("Fisher-Rao distance (eps="+eps+"):"+fr);
  System.out.println("exact="+exact);
  
  double sigma1,sigma2;
  double mu1=Math.random();
  double mu2=Math.random();
  m1=new Matrix(1,1);m1.set(0,0,mu1);
  sigma1=1+Math.random();
  S1=new Matrix (1,1); S1.set(0,0,sqr(sigma1));
  
  m2=new Matrix(1,1);m2.set(0,0,mu2);
  sigma2=1+Math.random();
  S2=new Matrix (1,1); S2.set(0,0,sqr(sigma2));
  N1=new vM(m1, S1);
    N2=new vM(m2, S2);
 fr=FisherRaoMVN(N1, N2, eps);
 exact=RaoDistance1D(mu1,sigma1,mu2,sigma2);
 err=Math.abs(fr-exact)/exact;
 
     System.out.println("Rao distance in 1d case:");
  System.out.println("Fisher-Rao distance (eps="+eps+"):"+fr);
  System.out.println("exact="+exact);
  System.out.println("relative error="+err);
      }
      
     public static  void Test2()
      {int i,nb=200; int d=1;
      double fr;
      Matrix m1,S1,m2,S2;
      vM N1,N2;
      double eps=0.001;
      
      for(i=0;i<nb;i++)
      {
      	
       
  m1=new Matrix(d,1); 
  	  m2=new Matrix(d,1); 
  	  	for(int j=0;j<d;j++) {m1.set(j,0,Math.random());m2.set(j,0,Math.random());
  	  	}
  	S1=randomSPDCholesky(d);
  S2=randomSPDCholesky(d);
  
 
  N1=new vM(m1, S1);
    N2=new vM(m2, S2);
    
    double jey;
     fr=FisherRaoMVN(N1, N2, eps);
 jey=JeffreysMVN(N1,N2);
 System.out.println(jey+"\t"+fr);
 
      }
      	
      }
}

 
\end{lstlisting}

\end{document}